\documentclass[10pt,journal,compsoc]{IEEEtran}
\usepackage{graphicx}
\usepackage{amsmath}
\usepackage{amsthm}
\usepackage{amssymb}
\usepackage{booktabs}
\usepackage{algorithm}
\usepackage{xcolor}
\definecolor{mydarkred}{rgb}{0.6,0,0}
\definecolor{mydarkgreen}{rgb}{0,0.6,0}
\usepackage[breaklinks=true,
            colorlinks,
            linkcolor = mydarkred,
            urlcolor  = purple, 
            citecolor = mydarkgreen,
            bookmarks = false]{hyperref}
\usepackage{algorithmic}
\usepackage{multirow,amsfonts,makecell,array}
\usepackage{threeparttable}
\usepackage{xcolor}

\newcommand{\mb}[1]{\mathbb{#1}}
\newcommand{\myPara}[1]{\vspace{.05in}\noindent\textbf{#1}}
\newcommand{\best}[1]{{\textbf{#1}}}
\newcommand{\secbest}[1]{{\textcolor{blue}{#1}}}
\newcommand{\ie}{\textit{i}.\textit{e}.}
\newcommand{\eg}{\textit{e}.\textit{g}.}

\graphicspath{{./fig/}}
\usepackage{float}
\usepackage{subfig}
\usepackage{bbm}

\newtheorem{Definition}{Definition}
\newtheorem{Theorem}{Theorem}
\newtheorem{Lemma}{Lemma}
\newtheorem{Assumption}{Assumption}

\newenvironment{thmbis}[1]
{\renewcommand{\theTheorem}{\ref{#1}}%
	\addtocounter{Theorem}{-1}%
	\begin{Theorem}}
	{\end{Theorem}}
\ifCLASSOPTIONcompsoc
  \usepackage[nocompress]{cite}
\else
  \usepackage{cite}
\fi
\ifCLASSINFOpdf
\else
\fi
\begin{document}
%
\title{Multi-Label Noise Transition Matrix Estimation with Label Correlations: Theory and Algorithm}
%
%
%
%

\author{Shikun~Li,
   Xiaobo~Xia,
   Hansong~Zhang,
   Shiming Ge$^\dagger$,~\IEEEmembership{Senior Member, IEEE,}\\
  Tongliang Liu,~\IEEEmembership{Senior Member, IEEE}
  
\thanks{$^\dagger$\quad \ Corresponding author.}
\IEEEcompsocitemizethanks{
\IEEEcompsocthanksitem S. Li, H. Zhang, and S. Ge are with the Institute of Information Engineering, Chinese
Academy of Sciences, Beijing 100095, China, and also with the School of Cyber
Security at University of Chinese Academy of Sciences, Beijing 100049,
China (e-mail: \{lishikun,hansongzhang,geshiming\}@iie.ac.cn).
\IEEEcompsocthanksitem X. Xia and T. Liu are with the Sydney AI Center, School of Computer Science, Faculty of Engineering, The University of Sydney, Darlington, NSW2008, Australia (e-mail: xxia5420@uni.sydney.edu.au;  tongliang.liu@sydney.edu.au).}
}

\IEEEtitleabstractindextext{%
\begin{abstract}
Noisy multi-label learning has garnered increasing attention due to the challenges posed by collecting large-scale accurate labels, making noisy labels a more practical and cost-effective alternative. Motivated by noisy multi-class learning, where the noise transition matrix is employed to represent the probabilities that clean labels flip into noisy labels, the introduction of transition matrices can help model multi-label noise and enable the development of statistically consistent algorithms for noisy multi-label learning. However, estimating multi-label noise transition matrices remains a challenging task, as most existing estimators in noisy multi-class learning rely on anchor points and accurate fitting of noisy class posteriors, which is hard to satisfy in noisy multi-label learning. In this paper, we address this problem by first investigating the identifiability of class-dependent transition matrices in noisy multi-label learning. Building upon the insights gained from the identifiability results, we propose a novel estimator that leverages label correlations without the need for anchor points or precise fitting of noisy class posteriors. Specifically, we first estimate the occurrence probability of two noisy labels to capture noisy label correlations. Subsequently, we employ sample selection techniques to extract information implying clean label correlations, which are then used to estimate the occurrence probability of one noisy label when a certain clean label appears. By exploiting the mismatches in label correlations implied by these occurrence probabilities, we demonstrate that the transition matrix becomes identifiable and can be acquired by solving a straightforward bilinear decomposition problem. Theoretically, we establish an estimation error bound for our multi-label transition matrix estimator and derive a generalization error bound for our statistically consistent algorithm with true transition matrices. Empirically, we validate the effectiveness of our estimator in estimating multi-label noise transition matrices, leading to excellent classification performance.

\end{abstract}

}

\maketitle

\IEEEdisplaynontitleabstractindextext

%
\IEEEpeerreviewmaketitle

\IEEEraisesectionheading{\section{Introduction}\label{sec:introduction}}
	\IEEEPARstart{I}{n} real-world scenarios, an instance is naturally associated with multiple labels, and these labels have \textit{complex entangled correlations}~\cite{ChenWWG19}. 
	Recently, the problem of label-noise learning in multi-label classification has received more and more attention~\cite{Liu2021TPAMI,Pene2021MultiLabelGA,Xie2022TPAMI,xia2022sample,wu2021class2simi,wu2021lr,Song2022LearningFN}, since it is time-consuming and expensive to collect large-scale accurate labels and the noisy labels are much cheaper and easier to acquire.
	In the setting of \textit{noisy multi-label learning}, the multiple labels assigned to an instance may be corrupted simultaneously. That is to say, any label for each class can be flipped with its respective \textit{transition matrix} that denotes the transition relationship from clean labels to noisy labels. 
	
	The noise transition matrix has been utilized to build a series of \textit{statistically consistent} algorithms for \textit{noisy multi-class learning}~\cite{natarajan2013learning,XiaLW00NS19,xia2022extended,Cheng22nips}.  The main advantage of these consistent algorithms is that they can guarantee to vanish the differences between the classifiers learned from noisy data and the optimal ones from clean data by increasing the size of noisy examples~\cite{Liu2016TPAMI,PatriniRMNQ17,XiaLW00NS19,shu2020meta}.
	
	Fortunately, these statistically consistent algorithms for noisy multi-class learning can also be applied in such noisy multi-label learning with a little  modification~\cite{Xie2022TPAMI}.
	However, the effectiveness of these algorithms heavily relies on estimating the transition matrix. Although the estimation of the transition matrix has been investigated in noisy multi-class learning, the estimation of the transition matrix in noisy multi-label learning has not been studied and remains challenging. Specifically, a series of methods~\cite{Liu2016TPAMI,PatriniRMNQ17,YaoL0GD0S20,XiaLW00NS19,LiL00S21} have been proposed to estimate the transition matrix for noisy multi-class learning. Most of them assume the existence of anchor points~\cite{Liu2016TPAMI,PatriniRMNQ17,YaoL0GD0S20} that are defined as the training examples belonging to a particular clean class surely. Nevertheless, the assumption is strong and hard to check when we only have noisy data~\cite{XiaLW00NS19}. Also, the methods need to accurately fit the noisy or intermediate class posterior of anchor points, which is rather difficult in multi-label cases, due to severe positive-negative imbalance~\cite{RidnikBZNFPZ21}.
 
	In this paper, to address the problem of estimating the noise transition matrix in noisy multi-label learning, we consider utilizing label correlations among noisy multiple labels~\cite{xia2023multi}. Specifically, some label correlations that should \textit{not exist} in practice are included in noisy multi-label learning. For example, as illustrated in Fig.~\ref{solution1}, ``fish'' and ``water'' always co-occur, while ``bird'' and ``sky'' always co-occur.
	Nevertheless, due to label errors, there is a \textit{slight correlation} between ``fish'' and ``sky'', which is impractical. At a high level, we can utilize \textit{the mismatch of label correlations} to identify the transition matrix without neither anchor points nor accurate fitting of noisy class posterior.
    \begin{figure}[!t]
		\centering
   \includegraphics[width=1.0\linewidth, trim=5 5 5 5,clip]{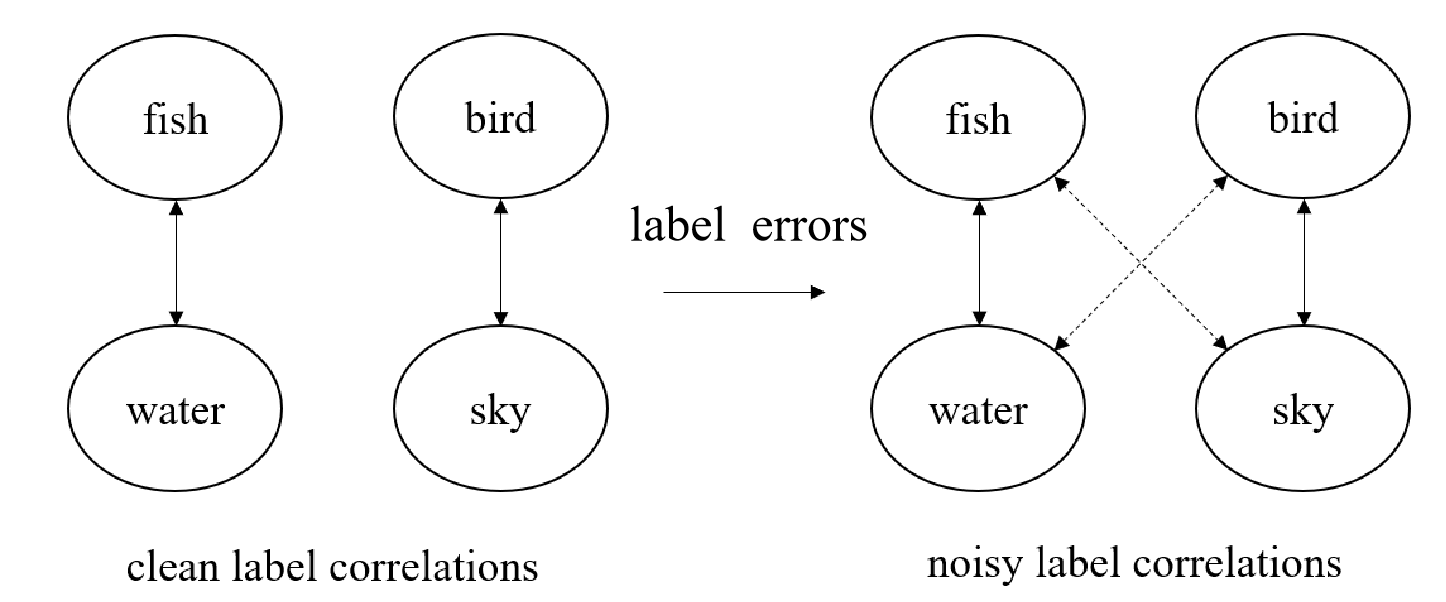}

		\caption{The illustration of the mismatch of label correlations among multiple class labels.}

		\label{solution1}
	\end{figure}
	
	In more detail, we first focus on the \textit{identifiability problem} of the class-dependent transition matrix in noisy multi-label learning. Accordingly, a new method that estimates the transition matrix by exploiting label correlations is proposed. That is, 
	motivated by the identifiability result that the label correlation of two noisy labels can not suffice to identify the transition matrix in noisy multi-label learning, we utilize \textit{sample selection} to extract useful information from noisy data, which implies clean label correlations to achieve the identifiability. Afterward, we not only estimate the \textit{occurrence probability} of two noisy labels in noisy data, but also of one noisy label when a certain clean label appears in selected data. By utilizing the mismatch of label correlations implied in these occurrence probabilities, we can prove the identifiability, and transform the problem of estimating the transition matrix using label correlations into a problem of \textit{bilinear decomposition}.
	Finally, with easy frequency counting, we can get a great estimation of the noise transition matrix.
    
	We conduct extensive experiments to justify our claims. Empirical results illustrate the effectiveness of the proposed estimator for estimating the transition matrix in noisy multi-label learning, and the consistent algorithms with our estimator can achieve better classification performance over various comparison methods.

    A preliminary version of this work was presented at the Thirty-seventh Conference on Neural Information Processing Systems (NeurIPS 2022) and selected to be ``Spotlight''~\cite{li2022estimating}. Compared to the preliminary version, we include new theoretical analyses and more experimental results. Specifically, in Section~\ref{addtheory}, we derive theoretical analyses that present an estimation error bound for our multi-label transition matrix estimator and a generalization error bound for the statistically consistent algorithm~\cite{Liu2016TPAMI,xia2020part} with true transition matrices. The proofs of the established error bounds are also provided in Appendices. Additional discussions and empirical results with instance-dependent label noise are included in Section~\ref{instance}. More experiments about ablation studies are also presented in Section~\ref{other}, which are consistent with our theoretical analyses. Our implementation is open-sourced at: \href{https://github.com/tmllab/Multi-Label-T}{https://github.com/tmllab/Multi-Label-T}.
\subsection{Contributions}
Before delving into details, we highlight the main contributions of this paper as follows:
\begin{itemize}
		\item We focus on an important but little explored problem of noisy multi-label learning, \ie, the estimation of multi-label noise transition matrices. In particular, we first prove some identifiability results of class-dependent transition matrices in such a setting.
		\item Inspired by the identifiability results, we propose a new estimator to learn the transition matrix by utilizing the wealth of label correlations without anchor points and accurate fitting of noisy class posterior.
        \item We provide an estimation error bound for our multi-label transition matrix estimator and a generalization error bound for the statistically consistent algorithm with true transition matrices, which theoretically justify the effectiveness of the proposed method.
	\item Empirical results across various noisy multi-label benchmarks with simulated class-dependent, instance-dependent, and real-world label noise, demonstrate the superiority of our transition matrix estimator.  Comprehensive ablation studies and discussions are also provided.
	\end{itemize}

\subsection{Organization}
The rest of the paper is organized as follows. In Section~\ref{setting}, we introduce the problem setting of label-noise learning in multi-label classification. In Section~\ref{method_intro}, we discuss the identifiability of the noise transition matrix under such a noisy multi-label setting, and introduce our estimation method. In Section~\ref{addtheory}, we provide theoretical analyses to justify the effectiveness of our proposed method. Experimental results are provided in Section~\ref{experiments}. 
Ablation studies about transition matrix estimation, multi-label classifier learning and the choice of hyperparameter are presented in Section~\ref{other}. 
Finally, we conclude the paper in Section~\ref{conclusion}. 
To improve readability, more discussions, proofs, and experimental results are provided in supplemental materials.




\section{Problem Setting}
	\label{setting}
	In this section, we introduce the problem setting of noisy multi-label learning. In what follows, scalars/vectors are in lowercase letters, and matrices/variables are in uppercase letters. For simplicity, let $[q]=\{1,\ldots,q\}$.
	
	\myPara{Preliminaries.}
	Let $D$ be the distribution of a pair of random variables $({X}, {Y})$, where ${X} \in \mathcal{X} \subseteq \mathbb{R}^{d}$ denotes the variable of instances, and  ${Y}=\{Y^{1},Y^{2},...,Y^{q}\} \in \{0,1\}^q $ denotes the variable of targets with $q$ possible class labels. As for ${Y}$, $Y^{j}=1$ indicates that the instance ${x}$ is associated with the class $j$; $Y^{j}=0$, otherwise. In multi-label learning, the goal is to learn a function about $D$ which maps each unseen instance ${x} \in \mathcal{X} $ to proper labels ${y}$. However, as discussed, ${Y}$ is hard to be annotated precisely. Before being observed, their true labels are independently flipped and what we can obtain are noisy training examples $\mathcal{D}_{t}=\left\{\left({x}_{i}, \bar{{{y}}}_i\right)\right\}_{i=1}^{n}$, where $\bar{{y}}_i=\{\bar{y}^{1}_{i},\bar{y}^{2}_{i},...,\bar{y}^{q}_{i}\}$ denotes noisy labels. Let $\bar{D}$ be the distribution of the noisy random variables $({X}, {\bar{Y}}) \in \mathcal{X} \times \{0,1\}^q $. In noisy multi-label learning, our goal is to infer proper labels for each unseen instance by \textit{only} using noisy training examples.
	
	\myPara{Noise transition matrix.} The random variables $\bar{Y}^{j}$ and $Y^{j}$ for the class $j$ are related through a noise transition matrix ${T}^{j} \in[0,1]^{2 \times 2}$, $j\in [q]$. 
	Generally, the transition matrix depends on instances, \ie, its element 
 ${T}^{j}_{i k}({X}={x})$ represents the probability $\mathbbm{P}(\bar{Y}^{j}=k \mid Y^{j}=i, {X}={x})$. Nevertheless, given only noisy examples, the instance-dependent transition matrix is \textit{non-identifiable} without any additional assumption~\cite{XiaLW00NS19,yao2021instance}. For example, both $\mathbbm{P}(\bar{Y}^j=k|{X}={x})=\sum_{i=0}^1 T^j_{ik}({X}={x})\mathbbm{P}(Y^j=i|{X}={x})$ and $\mathbbm{P}(\bar{Y}^j=k|{X}={x})=\sum_{i=0}^1 T^{\prime j}_{ik}({X}={x}) \mathbbm{P}^{\prime}(Y^j=i|{X}={x})$ are valid, for any $\mathbbm{P}^{\prime}(Y^j=0|{X}={x}) \in [0,1]$, 
 $\mathbbm{P}^{\prime}(Y^j=1|{X}={x})=1-\mathbbm{P}^{\prime}(Y^j=0|{X}={x})$,
 and
	$T_{ik}^{\prime j}({X}={x})=T^j_{ik}({X}={x}) \mathbbm{P}(Y^j=i | {X}={x}) / \mathbbm{P}^{\prime}(\bar{Y}^j=i | {X}={x})$. 
	Therefore, in this paper, we assume that the transition matrix is class-dependent and instance-independent~\cite{PatriniRMNQ17}, 
	\ie, $\mathbbm{P}(\bar{Y}^{j}=k \mid Y^{j}=i, {X}={x})=\mathbbm{P}(\bar{Y}^{j}=k \mid Y^{j}=i)$. The definition of the class-dependent label noise can be found in Appendix~\ref{discuss5}, where we further discuss its differences with the class-dependent label noise in the multi-class cases.
 
	\myPara{Statistically consistent algorithms.}
	The transition matrix bridges the class posterior probabilities for noisy and clean data, \ie, $\mathbbm{P}(\bar{Y}=k \mid {X}={x})=\sum_{i=0}^{1} T_{i k} \mathbbm{P}(Y=i \mid{X}={x})$. Therefore, it has been exploited to achieve many statistically consistent algorithms in noisy multi-class learning. Specifically, it has been utilized to build risk-consistent estimators via correcting loss functions \cite{Liu2016TPAMI,PatriniRMNQ17,XiaLW00NS19}, and to design classifier-consistent estimators via limiting hypotheses, \eg, \cite{PatriniRMNQ17,ChengZLGSL21,ZhuL021}. Since the multi-label task can be decomposed into multiple conditionally independent binary classification problems, we also can apply these consistent methods in noisy multi-label learning~\cite{Xie2022TPAMI}.
	In this paper, without loss of generality, by employing the importance reweighting (abbreviated as ``Reweight'') algorithm~\cite{Liu2016TPAMI,XiaLW00NS19}, we build the risk-consistent estimator $\bar{R}_{n, w}(\{{T}^{j}\}^q_{j=1}, {f})$ to learn statistically consistent classifiers with noise transition matrices:
\begin{align}
& \quad \bar{R}_{n, w}(\{{T}^{j}\}^q_{j=1}, {f})\\\nonumber
  &=\frac{1}{n} \sum_{i=1}^{n} \sum_{j=1}^{q} \frac{\mathbbm{\hat{P}}( {Y}^{j}=\bar{y}^{j}_{i} \mid {X}={x}_i)}{\mathbbm{\hat{P}}( \bar{Y}^{j}=\bar{y}^{j}_{i} \mid {X}={x}_i)} \ell\left(f_j\left({x}_{i}\right), \bar{y}^{j}_{i}\right),
		\label{reweiht}
\end{align}
where $\ell$ is the binary cross entropy function, and ${f}=(f_1,f_2,...,f_q)$ is the learnable $q$ classification functions; the subscript $w$ denotes that the loss function is weighted. 
 More details of the Reweight algorithm for noisy multi-label learning can be found in Appendix~\ref{Reweight}.
	
	\myPara{Transition matrix estimation.}  As inaccurate transition matrices will degenerate the performances of these consistent algorithms, a series of estimation methods~\cite{Liu2016TPAMI,XiaL0WGL0TS20,YaoL0GD0S20,ZhuSL21,LiL00S21} have been proposed for noisy multi-class learning to efficiently identify the transition matrix. However, most of them require the assumption of anchor points~\cite{PatriniRMNQ17,YaoL0GD0S20,LiL00S21}, which is strong and hard to check in multi-label cases when only noisy data are provided~\cite{XiaLW00NS19}. Besides, severe positive-negative imbalance in multi-label learning ~\cite{RidnikBZNFPZ21} will make it difficult to accurately approximate the noisy or intermediate class posterior of anchor points, which is crucial for these methods. This motivates us to seek for a better estimator that can do without anchor points and avoid estimating noisy posterior in noisy multi-label learning.
	
	\section{Estimating Transition Matrices with Label Correlations}
	\label{method_intro}
	In this section, we first study the identifiability problem~\cite{abs-2202-02016} of class-dependent transition matrices in the cases of multi-label learning (Section~\ref{Identifiability}). Furthermore, inspired by these results, we propose a new estimator to estimate the noise transition matrix by utilizing label correlations~(Section~\ref{ours}). 
	
	
	\subsection{Identifiability of Transition Matrix}
	\label{Identifiability}
	Recently, Liu et al.~\cite{abs-2202-02016} built identifiability of the noise transition matrix on Kruskal’s identifiability results. Inspired by them, with the complex correlations among class labels, we can get some identifiability results of the class-dependent transition matrix in noisy multi-label learning.
	
	Following~\cite{abs-2202-02016}, to define identifiability, we denote an observation space by $\Omega$. For a general parametric
	space $\Theta$, we denote the distribution (probability density function) induced by the parameter $\theta \in \Theta$ on the observation space $\Omega$ as $\mathbbm{P}_{\theta}$~\cite{Allman2009IdentifiabilityOP}. The identifiability for a general parametric space is defined as follows.
	\begin{Definition} [Identifiability~\cite{abs-2202-02016}] The parameter $\theta$ is identifiable if $\mathbbm{P}_\theta \neq \mathbbm{P}_{\theta^{\prime}}, \forall \theta \neq \theta^{\prime}$.
	\end{Definition}
	

	For the class-dependent transition matrix ${T}^j$ for class $j$, the identifiability can be defined as the following, when ${T}^j$ is a part of $\theta$.
	\begin{Definition} [Identifiability of ${T}^j$] ${T}^j$ is identifiable if  $\mathbbm{P}_{\theta} \neq \mathbbm{P}_{\theta^{\prime}}$ for
		$\theta \neq \theta^{\prime}$, up to the label permutation of class $j$.
		\label{define2_}
	\end{Definition}
	\myPara{Remark~1.} The label permutation of class $j$ means swapping 1 and 0 class values in class $j$, and the rows of ${T}^j$ will also swap. 
	Note that $\Omega$ does not necessarily include all observable variables. We use an example to better understand it. For example, let $\Omega:=\{\bar{Y}^j, {X}\}$, $\theta:=\{{T}^j, \mathbbm{P}({{Y}}^{j}|{X})\}$, and $\mathbbm{P}_\theta:=\mathbbm{P}(\bar{Y}^j\mid{X})$, the identifiability of ${T}^j$ can be achieved with the anchor point assumption~\cite{PatriniRMNQ17,YaoL0GD0S20,LiL00S21}.
	
	To use notations without confusion, for the identifiability of ${T}^j$ using label correlations, we let $\Omega:=\{\bar{Y}^j, {\bar{Y}}^{-j}\}$, where ${-j}$ means other classes having correlations with class $j$, $\theta:=\{{T}^j,\mathbbm{P}(Y^j),\mathbbm{P}(\bar{{Y}}^{-j}|{Y}^j)\}$, and $\mathbbm{P}_\theta:=\mathbbm{P}(\bar{Y}^j, {\bar{{Y}}}^{-j}) $. Also, $\Omega$ does not necessarily include all noisy class labels. As we need $\Omega$ to provide useful information to achieve identifiability, the exploration of an effective $\Omega$ and necessary conditions is one of the focuses of the paper. 
\subsubsection{Assumptions}	
In the following, we discuss some assumptions that are used to simplify theoretical analysis. 
	\begin{Assumption} 
		$\mathbbm{P}(\bar{Y}^{j}=0 \mid Y^{j}=1)+\mathbbm{P}(\bar{Y}^{j}=1 \mid Y^{j}=0)<1$, $j\in[q]$.
		\label{assum1}
	\end{Assumption}
	\myPara{Remark~2.} Assumption~\ref{assum1} means that the noisy label agrees with the clean label on average, which is a standard condition for analysis under the class-dependent transition matrix~\cite{natarajan2013learning,MenonROW15}.
	\begin{Assumption} 
		$\mathbbm{P}(Y^{i}=0 \mid Y^{j}=0) \neq \mathbbm{P}(Y^{i}=0 \mid Y^{j}=1)$,  $i,j\in[q]$ and $i \neq j$.
		\label{assum2}
	\end{Assumption}
\myPara{Remark~3.} Assumption~\ref{assum2} means that the multiple labels have correlations between each other, which is satisfied by most of $(i,j)$ pairs in the real-world dataset (more details are provided in Appendix~\ref{Validation}). 

\subsubsection{Theoretical results}

 When considering multi-label learning, the simplest case is having two class labels. In this case, the following theoretical results can be obtained. 
	
	\begin{Theorem} Two noisy labels $\{\bar{Y}^{j},\bar{Y}^{i}\}$ will not suffice to identify ${T}^{j}$.
		\label{theorem1}
	\end{Theorem}
	This result tells us that the label correlations of two noisy labels can not offer enough information to achieve the identifiability of ${T}^j$. We provide Theorem~\ref{theorem2} based on the Kruskal’s identifiability result~\cite{kruskal1977three,sidiropoulos2000uniqueness}.

	
	\begin{Theorem} If $\bar{Y}^{i}$ and $\bar{Y}^{k}$ are independent given $Y^{j}$, three noisy labels $\{\bar{Y}^{j},\bar{Y}^{i}, \bar{Y}^{k}\}$ are sufficient to identify ${T}^{j}$.
		\label{theorem2}
	\end{Theorem}
	The assumption that $\bar{Y}^{i}$ and $\bar{Y}^{k}$ are independent given $Y^{j}$ can be satisfied in certain cases, \eg, the occurrences of ``blue'' and ``dolphin'' may be independent given ``sea'' appearing or not. Nevertheless, due to the complex correlations among labels, sometimes, this condition is hard to hold in most cases. When the condition can not hold, these label correlations are no longer sufficient to determine ${T}^{j}$, as shown in Theorem~\ref{theorem3} below. 
	
	\begin{Theorem} If $\bar{Y}^{i}$ and $\bar{Y}^{k}$ are not independent given $Y^{j}$, three noisy labels $\{\bar{Y}^{j},\bar{Y}^{i}, \bar{Y}^{k}\}$will not suffice to identify ${T}^{j}$.
		\label{theorem3}
	\end{Theorem}
	The inspiration from Theorem~\ref{theorem3} is that, the increase of the number of noisy labels may make the identifiability decrease due to the entangled correlations (see Appendix~\ref{proof3} and~\ref{inspiration}).
	To handle this problem,
	we prove Theorem~\ref{theorem4} by assuming the transition relationship between the noisy label for class $i$ and the clean label for class $j$ is known.
	\begin{Theorem} If $\mathbbm{P}(\bar{Y}^{i}\mid {Y}^{j})$ is known, two noisy labels $\{\bar{Y}^{j},\bar{Y}^{i}\}$ are sufficient to identify ${T}^{j}$.
		\label{theorem4}
	\end{Theorem}
	Theorem~\ref{theorem4} theoretically guarantees that the identifiability of the class-dependent transition matrix can be achieved by utilizing the occurrence probabilities $\mathbbm{P}(\bar{Y}^{i}, \bar{Y}^{j})$ and $\mathbbm{P}(\bar{Y}^{i} \mid {Y}^{j})$. Note that $\mathbbm{P}(\bar{Y}^{i}, \bar{Y}^{j})$ can represent noisy label correlations, and $\mathbbm{P}(\bar{Y}^{i}\mid {Y}^{j})=\sum_{{Y}^{i}} \mathbbm{P}(\bar{Y}^{i}\mid {Y}^{i})\mathbbm{P}({Y}^{i}\mid {Y}^{j})$, which can imply clean label correlations. At a high level, \textit{the mismatch of label correlations} implied in the occurrence probabilities can achieve the identifiability.
	
	The detailed proofs of Theorems~\ref{theorem1}-\ref{theorem4} are provided in Appendix~\ref{proof1}-\ref{proof4} respectively. 
	
	\subsection{The Proposed Estimator}
	\label{ours}
	Our estimator is based on Theorem~\ref{theorem4}, which needs extra information to estimate $\mathbbm{P}(\bar{Y}^{i}\mid {Y}^{j})$.  
	Recently, the \textit{memorization effect}~\cite{arpit2017closer} of deep networks has received much attention in learning with noisy labels, which shows that deep networks will first memorize the training data with clean labels and then those with noisy labels. Prior works utilize this characteristic to develop \textit{sample selection} methods~\cite{Jiang2018icml,Han2018NIPS,LiSH20,KimKCCY21,song2019selfie,song2021robust,Li2022CVPR,xia2023combating},  where we select some examples with more likely clean labels for each class $j$ respectively in the early learning phase. The selected examples can serve as useful extra information that implies clean label correlations, through which we can achieve estimation via counting. However, when implementing sample-selection-based methods, a major concern is whether the sampling bias will lead to large estimation errors. 
	
	Generally speaking, according to the memorization effect, the sampling bias is about selecting easy examples for class $j$, which usually means these examples have easy-to-discriminate features. We can reasonably assume that given ${Y}^{j}$, the distribution of the features about class $j$ is biased, while the distribution of the features about another class $i$ is unbiased,
	\ie,
	\begin{equation}  
 \begin{aligned}
		\mathbbm{P}_{{D}^{j}_{s}}(\bar{Y}^{i}|{Y}^{j})
		&= \int \mathbbm{P}_{{D}^{j}_{s}}(\bar{Y}^{i}|{X}^{i}) \mathbbm{P}_{{D}^{j}_{s}}({X}^{i}|{Y}^{j}) d{x} 
		\\ &= \int \mathbbm{P}_{\bar{D}^{j}_{s}}(\bar{Y}^{i}|{X}^{i}) \mathbbm{P}_{\bar{D}^{j}_{s}}({X}^{i}|{Y}^{j}) d{x} 
		\\ &= \mathbbm{P}_{\bar{D}^{j}_{s}}(\bar{Y}^{i}|{Y}^{j}),
		\label{eq1}
  \end{aligned}
	\end{equation}
	where ${D}^{j}_{s}$ is the distribution of $({X}, {Y_s})$, ${Y_s}=\{\bar{Y}^{1},\bar{Y}^{2},...,Y^{j},...,\bar{Y}^{q}\}$, $\bar{D}^{j}_{s}$ is the biased distribution of $({X}, {Y_s})$, and ${X}^{i}$ is the part of ${X}$, which represents all information about class $i$ appearing or not.
	When the assumption is satisfied, the sample selection will not lead to large estimation errors on $\mathbbm{P}(\bar{Y}^{i}\mid {Y}^{j})$, and it can converge to zero exponentially fast by counting~\cite{Boucheron2013ConcentrationI}.

	In real-world scenarios, due to the complex label correlations, this assumption will not strictly hold. While, it may be roughly met when class labels $i$ and $j$ do not share the major discriminative features. Intuitively speaking, the classifying of simplest examples for one label is not easily affected by the presence or absence of other significantly different labels. Also, since most label pairs from typical real-world multi-label datasets are significantly different (see Appendix~\ref{real-world}), the assumption can be roughly held for those label pairs in typical cases.
	In Section~\ref{experiment1}, our empirical results justify this by showing a little gap between the estimation error of our estimator with a biased sample selection and an unbiased one.

	Based on the above discussions, we can approximate $\mathbbm{P}(\bar{Y}^{i}, \bar{Y}^{j})$ and $\mathbbm{P}(\bar{Y}^{i} \mid {Y}^{j})$ with \textit{frequency counting}, and utilize the \textit{mismatch} of label correlations implied in these occurrence probabilities to estimate the transition matrix. In this work, we propose to estimate transition matrices $\{{T}^{j}\}_{j=1}^q$ below. 

	First, we utilize sample selection to obtain useful information that implies clean label correlations. Specifically, we train a classifier with the standard multi-label classification loss on noisy training examples $\mathcal{D}_{t}$ for a few epochs, and then perform sample selection to get a selected clean set $\mathcal{D}^{j}_{s}$ for each class label $j$. Specially, we use a commonly used sample selection way~\cite{ArazoOAOM19,LiSH20} in learning with noisy labels, which extracts the subset of examples with small losses by modeling the distribution of losses for class $j$ with a Gaussian mixture model~(GMM).
	
	Second, we perform co-occurrence estimation by frequency counting, and then estimate the transition matrix by solving a simple bilinear decomposition problem. For class label $j$, we first choose another class label $i$ and estimate $\mathbbm{P}(\bar{Y}^{i}, \bar{Y}^{j})$ and $\mathbbm{P}(\bar{Y}^{i} \mid {Y}^{j})$  via frequency counting, \ie,
	\begin{equation}
		\mathbbm{\hat{P}}(\bar{Y}^{i}=v,\bar{Y}^{j}=k)=\frac{1}{n}\sum_{({x}, \bar{{y}}) \in \mathcal{D}_{t}}\mb{I}[\bar{y}^{i}=v,\bar{y}^{j}=k],
		\label{pro_two_labels}
	\end{equation}
	\begin{equation}
		\mathbbm{\hat{P}}(\bar{Y}^{i}=v \mid {Y}^{j}=k)=\frac{\sum_{({x}, \bar{{y}}) \in \mathcal{D}^{j}_{s}}\mb{I}[\bar{y}^{i}=v,\bar{y}^{j}=k]}{\sum_{({x}, \bar{{y}}) \in \mathcal{D}^{j}_{s}}\mb{I}[\bar{y}^{j}=k]},
		\label{pro_condition}
	\end{equation}
	where $\mb{I}[\cdot]$ is the indicator function which takes 1 if the identity index is true and 0 otherwise.

	Then, these co-occurrence probabilities, which imply the mismatch of label correlations, can lead to four equations involving ${T}^{j}$:
	\begin{equation}
		\begin{aligned}
			\mathbbm{P}\left(\bar{Y}^{j}=0,  \bar{Y}^{i}=0\right)&=\mathbbm{P}({Y}^{j}=0)T^{j}_{00}\mathbbm{P}(\bar{Y}^{i}=0|  {Y}^{j}=0)\\&+\mathbbm{P}({Y}^{j}=1)T^{j}_{10}\mathbbm{P}(\bar{Y}^{i}=0|  {Y}^{j}=1),\\
			\mathbbm{P}\left(\bar{Y}^{j}=0,  \bar{Y}^{i}=1\right)&=\mathbbm{P}({Y}^{j}=0)T^{j}_{00}\mathbbm{P}(\bar{Y}^{i}=1|  {Y}^{j}=0)\\&+\mathbbm{P}({Y}^{j}=1)T^{j}_{10}\mathbbm{P}(\bar{Y}^{i}=1|  {Y}^{j}=1),\\
			\mathbbm{P}\left(\bar{Y}^{j}=1,  \bar{Y}^{i}=0\right)&=\mathbbm{P}({Y}^{j}=0)T^{j}_{01}\mathbbm{P}(\bar{Y}^{i}=0|  {Y}^{j}=0)\\&+\mathbbm{P}({Y}^{j}=1)T^{j}_{11}\mathbbm{P}(\bar{Y}^{i}=0|  {Y}^{j}=1),\\
			\mathbbm{P}\left(\bar{Y}^{j}=1,  \bar{Y}^{i}=1\right)&=\mathbbm{P}({Y}^{j}=0)T^{j}_{01}\mathbbm{P}(\bar{Y}^{i}=1|  {Y}^{j}=0)\\&+\mathbbm{P}({Y}^{j}=1)T^{j}_{11}\mathbbm{P}(\bar{Y}^{i}=1|  {Y}^{j}=1).
		\end{aligned}
	\end{equation}
	For simplicity, we denote
 \begin{footnotesize}
	\begin{equation*}
		\begin{aligned}
			&{E}=\begin{pmatrix}\mathbbm{P}(\bar{Y}^{j}=0,  \bar{Y}^{i}=0) & \mathbbm{P}(\bar{Y}^{j}=0,  \bar{Y}^{i}=1)\\ \mathbbm{P}(\bar{Y}^{j}=1,  \bar{Y}^{i}=0) & \mathbbm{P}(\bar{Y}^{j}=1,  \bar{Y}^{i}=1) \end{pmatrix}=\begin{pmatrix} e_{00} & e_{01} \\ e_{10} & e_{11} \end{pmatrix},
			\\& {P} =\begin{pmatrix}\mathbbm{P}({Y}^{j}=0) & 0 \\ 0 & \mathbbm{P}({Y}^{j}=1) \end{pmatrix}=\begin{pmatrix}1-p & 0 \\ 0 & p \end{pmatrix},
			\\& {T}^{j}=\begin{pmatrix}\mathbbm{P}(\bar{Y}^{j}=0 \mid {Y}^{j}=0) & \mathbbm{P}(\bar{Y}^{j}=1 \mid {Y}^{j}=0)\\ \mathbbm{P}(\bar{Y}^{j}=0\mid {Y}^{j}=1) & \mathbbm{P}(\bar{Y}^{j}=1 \mid {Y}^{j}=1) \end{pmatrix}\\&\quad \ =\begin{pmatrix} 1-\rho_{-} & \rho_{-} \\ \rho_{+} & 1-\rho_{+} \end{pmatrix},
			\quad\text{and}\\&	{M}=\begin{pmatrix}\mathbbm{P}(\bar{Y}^{i}=0 \mid {Y}^{j}=0) & \mathbbm{P}(\bar{Y}^{i}=1 \mid {Y}^{j}=0)\\ \mathbbm{P}(\bar{Y}^{i}=0\mid {Y}^{j}=1) & \mathbbm{P}(\bar{Y}^{i}=1 \mid {Y}^{j}=1) \end{pmatrix}\\&\quad\  =\begin{pmatrix} 1-\rho^{\prime}_{-} & \rho^{\prime}_{-} \\ \rho^{\prime}_{+} & 1-\rho^{\prime}_{+} \end{pmatrix}.
		\end{aligned}
	\end{equation*}
 \end{footnotesize}
	Then the system of equations can be expressed as \\ ${E} =  ({T}^{j})^\top {P} {M}$, \ie,
 \begin{footnotesize}
	\begin{equation}
		\begin{pmatrix} e_{00} & e_{01} \\ e_{10} & e_{11} \end{pmatrix}=
		{\begin{pmatrix} 1-\rho_{-} & \rho_{-} \\ \rho_{+} & 1-\rho_{+} \end{pmatrix}}^\top
		\begin{pmatrix}1-p & 0 \\ 0 & p \end{pmatrix}
		{\begin{pmatrix} 1-\rho^{\prime}_{-} & \rho^{\prime}_{-} \\ \rho^{\prime}_{+} & 1-\rho^{\prime}_{+} \end{pmatrix}}.
		\label{system}
	\end{equation}
\end{footnotesize}	
	Denote the estimation of ${E}$, ${P}$, ${T}^{j}$, and ${M}$ as 
 \begin{footnotesize}
 $${\hat{{E}}}=\begin{pmatrix} \hat{e}_{00} & \hat{e}_{01} \\ \hat{e}_{10} & \hat{e}_{11} \end{pmatrix}, \hat{{P}}=\begin{pmatrix}1-\hat{p} & 0 \\ 0 & \hat{p} \end{pmatrix},
	\hat{{T}}^{j}=\begin{pmatrix} 1-\hat{\rho}_{-} & \hat{\rho}_{-} \\ \hat{\rho}_{+} & 1-\hat{\rho}_{+} \end{pmatrix},$$ and $
	\hat{{M}}=\begin{pmatrix} 1-\hat{\rho}^{\prime}_{-} & \hat{\rho}^{\prime}_{-} \\ \hat{\rho}^{\prime}_{+} & 1-\hat{\rho}^{\prime}_{+} \end{pmatrix}.$
\end{footnotesize}

	As ${\hat{{E}}}$ and $\hat{{M}}$ can be derived from Eq.~(\ref{pro_two_labels}) and Eq.~(\ref{pro_condition}), the problem is hence equivalent to a bilinear decomposition problem:
	\begin{equation}
		\hat{{E}} (\hat{{M}})^{-1} =  (\hat{{T}}^{j})^\top\hat{{P}}.
		\label{system2}
	\end{equation}
	By solving the above matrix equation, we can get 
	\begin{equation}\hat{p}= \frac{(1-\hat{\rho}^{\prime}_{-})-(\hat{e}_{00}+\hat{e}_{10})}{1-\hat{\rho}^{\prime}_{-}-\hat{\rho}^{\prime}_{+}},
 \label{eq6}
	\end{equation}
	and the estimation of the transition matrix 
	\begin{equation}
		\hat{{T}}^{j} = [\hat{{E}} (\hat{{M}})^{-1} (\hat{{P}})^{-1}]^\top.
		\label{eq7}
	\end{equation}
	\myPara{Implementation of our estimator.} The pseudo-code of our estimator is described in Algorithm~\ref{alg1}. A little difference from the above is that in order to make better use of correlations among labels, we perform $R$ times co-occurrence estimation and bilinear decomposition for different classes $i$ in the second stage to get $R$ estimations, $\hat{{T}}^{j}_r, r=1,2,...,R$. 
 \vspace{-2pt}
 Finally, we estimate the transition matrix ${{T}}^{j}$ by the below aggregation rule from :
	\begin{equation}
 \vspace{-5pt}
 \hat{{T}}^{j} = \arg \min_{\hat{{T}}^{j}_r} \sum_{i=1}^R \Vert \hat{{T}}^{j}_r - \hat{{T}}^{j}_i \Vert_1,
		\label{final}
  \vspace{-2pt}
	\end{equation}
	where $\Vert \cdot \Vert_1$ denotes $\ell_1$ norm.
 
	\begin{algorithm}[!t]
		\caption{Estimating Multi-Label Noise Transition Matrices with Label Correlations.}
		\label{alg1}
		\begin{algorithmic}[1]
			\REQUIRE Noisy training examples $\mathcal{D}_{t}$, the number of classes $q$, the early warmup training epoch $E_{warm}$, the threshold of sample selection $\tau$, and repeated estimation times $R$.
			
			\textbf{Stage1: Standard Training and Sample Selection}
			\STATE  Standard training with the standard multi-label loss for $E_{warm}$ epochs.
			\FOR {$j=1,2,...,q$}
			\STATE  Model losses with a trained classifier on $\mathcal{D}_{t}$ by a GMM.
			\STATE  Get the selected  set $\mathcal{D}^{j}_{s}$ for class label $j$ with the threshold $\tau$.
			\ENDFOR
			
			\textbf{Stage2: Co-occurrence Estimation and Bilinear Decomposition}
			\FOR {$j=1,2,...,q$}
			\FOR {$r=1,2,...,R$}
			\STATE Choose another class label $i$.
			\STATE Estimate  $\mathbbm{P}(\bar{Y}^{i}, \bar{Y}^{j})$ by $\mathbbm{\hat{P}}(\bar{Y}^{i}, \bar{Y}^{j})$ with Eq.~(\ref{pro_two_labels}) on $\mathcal{D}_{t}$, and $\mathbbm{P}(\bar{Y}^{i} \mid {Y}^{j})$ by $\mathbbm{\hat{P}}(\bar{Y}^{i} \mid {Y}^{j})$ with Eq.~(\ref{pro_condition}) on $\mathcal{D}^{j}_{s}$.
			\STATE Solve a bilinear decomposition problem (Eq.~(\ref{system2})) to get an estimation $\hat{{T}}^{j}_r$ by Eq.~(\ref{eq7}).
			\ENDFOR
			\STATE Estimate ${{T}}^{j}$ by $\hat{{T}}^{j}$ which has the minimum error with Eq.~(\ref{final}) from $R$ estimations.
			\ENDFOR
			\ENSURE The estimated transition matrices $\{\hat{{T}}^{j}\}_{j=1}^q$.	
		\end{algorithmic}
  \vspace{-3pt}
	\end{algorithm}
 
    \section{Theoretical Analyses}
    \label{addtheory}
    In this section, we provide an estimation error bound for our multi-label transition matrix estimator~(Section~\ref{est_error}) and a generalization error bound for the statistically consistent algorithm with true transition matrices~(Section~\ref{generalization}), which theoretically justify the effectiveness of the proposed method.

\vspace{-3pt}
\subsection{Estimation Error Bound}
\label{est_error}
In this subsection, we theoretically justify what factors will influence the effectiveness of our multi-label transition matrix estimator and how accurate it will be with a finite sample. As shown in Section~\ref{ours}, our estimator of ${T}^j$ need to 
perform $R$ times of estimations with different class $i$.
Then, for the $r$-th estimation, let  
	\begin{equation}
 \begin{aligned}
&\mathbbm{P}(\bar{Y}^{i}=v, \bar{Y}^{j}=k)=e_{kv}^r,\\
&\mathbbm{P}(\bar{Y}^{j}=1|{Y}^{j}=0)=\rho_{-}^{r}, \mathbbm{P}(\bar{Y}^{j}=0|{Y}^{j}=1)=\rho_{+}^{r},\\
&\mathbbm{P}_{{D}_s^j}(\bar{Y}^{i}=1|{Y}^{j}=0)=\rho_{-}^{\prime r}, \mathbbm{P}_{{D}_s^j}(\bar{Y}^{i}=0|{Y}^{j}=1)=\rho_{+}'^{r},\\
&\mathbbm{P}_{\bar{D}_s^j}(\bar{Y}^{i}=1| {Y}^{j}=0)-\mathbbm{P}_{{D}_s^j}(\bar{Y}^{i}=1 \mid {Y}^{j}=0)=\Delta_{0}^r,   \\
&\mathbbm{P}_{\bar{D}_s^j}(\bar{Y}^{i}=0| {Y}^{j}=1)-\mathbbm{P}_{{D}_s^j}(\bar{Y}^{i}=0 \mid {Y}^{j}=1)=\Delta_{1}^r.
\label{eq10}
\end{aligned}
	\end{equation}
Besides, let $\lambda_k$ be the ratio of 
selected examples with $\bar{y}^{j}=k$ in training examples. We assume $\lambda_1 < \lambda_0$.


To derive an estimation error bound of ${\hat{T}}^j$, we first introduce the following two assumptions.
	\begin{Assumption} 
		$|\Delta_0^r| \leq \Delta, |\Delta_1^r| \leq \Delta, r\in\{1,\ldots,R\}.$
		\label{assum3}
	\end{Assumption}
	Assumption~\ref{assum3} means that for convenience, we assume the influence of sampling bias is smaller than a \textit{certain constant} $\Delta$.
 	\begin{Assumption} 
		$\|{T}^j-{\hat{T}}^j\|_{1} \leq \|{T}^j-{\bar{T}}^j\|_{1}$, for $j\in\{1,\ldots,q\}$, where ${\hat{T}}^{j} =\arg \min_{\hat{{T}}^{j}_i} \sum_{r=1}^R \Vert \hat{{T}}^{j}_i - \hat{{T}}^{j}_r \Vert_1$ and $\bar{{T}}^{j} =\arg \min_{\hat{{T}}}  =\sum_{r=1}^R \Vert \hat{{T}} - \hat{{T}}^{j}_r \Vert_1 =\sum_{r=1}^R\frac{1}{R}\hat{{T}}^{j}_r$.
		\label{assum4}
	\end{Assumption}
	Assumption~\ref{assum4} means that the error of the final estimation $\hat{{T}}^{j}$ by Eq.~(\ref{final}) is smaller than that of the average of $R$ estimations. Note that this assumption has been verified by the experiments in Section~\ref{other1}. With the above assumptions, we are ready to present our theoretical results. 
 
 
\begin{Theorem}\label{thm:main2}
Let $C$ be the maximum absolute value of the first derivative of $\hat{\rho}^r_+$ and $\hat{\rho}^r_-$ w.r.t. the estimated co-occurrence probabilities. Under Assumptions~\ref{assum3} and~\ref{assum4},
with a size $n$ of training examples, for any $\delta>0$, with probability at least $1-\delta$,
\begin{align*}
  \|{T}^j-{\hat{T}}^j\|_{1} \leq  8C\Delta + \min \left( 4C\sqrt{\left(\frac{3}{4n}+\frac{1}{2\lambda_1 n}
\right)\frac{2}{{R\delta}}},\right.\\
\left.
4C\sqrt{(\frac{3}{4n}+\frac{1}{2\lambda_1 n})\frac{2\log (4 / \delta)}{R}}+\frac{8 \log (4/\delta)}{3 R} \right).  
\end{align*}
Further, if sample selection is unbiased ($\Delta\rightarrow0$), then $\hat{{T}}^j \stackrel{p}{\rightarrow} {T}^j$ as $n \rightarrow \infty$.
\end{Theorem}
A detailed proof is provided in Appendix~\ref{proof_6}. 

\myPara{Remark~4.} First, this bound theoretically shows that the effectiveness of our estimator is neither dependent on the existence of anchor points nor accurate fitting of noisy class posterior, which are major problems of existing estimators. Second, the derived estimation error bound shows that a sufficiently large number of training samples $n$ can ensure that the estimation error is mainly due to the error from sampling bias $\Delta$. When sample selection is unbiased ($\Delta\rightarrow0$), the proposed estimator ${\hat{T}}^j$ will converge to the true transition matrix ${T}^j$ in probability as $n$ increases, which we empirically verify in the ablation experiments (See Section~\ref{other1}). It guarantees the effectiveness of the proposed estimator in typical scenarios with large samples.  Third, it reveals that obtaining a final transition matrix via multiple estimations can promote its accuracy as the times of estimations $R$ increases. In addition, this bound tells us a limitation that when the number of one label appearing is too few, leading to very small $\lambda_1$, the estimation of the transition matrix for this class label will be difficult to be accurate. 

\subsection{Generalization Error Bound}
\label{generalization}
Here, we theoretically justify how the Reweight method~\cite{Liu2016TPAMI} generalizes for learning classifiers in noisy multi-label cases given true transition matrices $\{{T}^j\}_{j=1}^q$. 

Assume the neural network has $d$ layers, parameter matrices ${W}_1,\ldots,{W}_d$, and activation functions $\sigma_1,\ldots,\sigma_{d-1}$ for each layer. Let denote the mapping of the neural network by $h: {x}\mapsto {W}_d\sigma_{d-1}({W}_{d-1}\sigma_{d-2}(\ldots \sigma_1({W}_1{x})))\in\mathbb{R}^q$. The output of the sigmoid function is defined by $g_j({x})=1/(1+\exp{(-h_j({x}))}), j=1,\ldots,q$.
Let $f_j({x})=\mathbb{I}\left[g_j({x})>0.5\right], j=1,\ldots,q,$ be the classifier learned from the hypothesis space ${F}$ determined by the real-valued parameters of the neural network, \ie, $\hat{{f}}=\arg\min_{{f}\in {F}}\bar{R}_{n,w}({f})$.

To derive a generalization bound, following~\cite{XiaLW00NS19}, we assume that instances are upper bounded by $B$, \ie, $\|{x}\|\leq B$ for all ${x}\in\mathcal{{X}}$, and that the loss function $\mathcal{L}({f}({x}),\bar{{y}})$ is $L$-Lipschitz continuous w.r.t. ${f}({x})$ and upper bounded by $M$, \ie, for any ${f},{f}^{\prime} \in {F}$ and any $({x},\bar{{y}})$, $|\mathcal{L}({f}({x}),\bar{{y}})-\mathcal{L}({f}^{\prime}({x}),\bar{{y}})|\leq L|{f}({x})-{f}^{\prime}({x})|$, and for any $({x},\bar{{y}})$, $\mathcal{L}({f}({x}),\bar{{y}})\leq M$. 


\begin{Theorem}\label{thm:main}
Assume the Frobenius norm of the weight matrices ${W}_1,\ldots,{W}_d$ are at most $M_1,\ldots, M_d$. Let the activation functions be 1-Lipschitz, positive-homogeneous, and applied element-wise (such as the ReLU).
Let the loss function be the binary cross-entropy loss, \ie, $\mathcal{L}({f}({x}),\bar{{y}})=\sum_{j=1}^{q}-\bar{y}^j\log(g_j({x}))-(1-\bar{y}^j)\log(1-g_j({x}))$.
Let $\hat{{f}}$ be the learned classifier.
Assume the true transition matrices $\{{T}^j\}_{j=1}^q$ is given.  Let $U=\frac{1}{\min_j(1-\max(\mathbbm{P}(\bar{Y}^{j}=0 \mid Y^{j}=1),\mathbbm{P}(\bar{Y}^{j}=1 \mid Y^{j}=0)))}$.
Then, for any $\delta>0$, with probability at least $1-\delta$,
    \begin{align}
\mathbb{E}[\bar{R}_{n,w}(\{{T}^j\}_{j=1}^q,\hat{{f}})]- \bar{R}_{n,w}(\{{T}^j\}_{j=1}^q,\hat{{f}}) \nonumber \\ 
\leq \frac{2BqL(\sqrt{2d\log2}+1)\Pi_{i=1}^{d}M_i}{\sqrt{n}}+UM\sqrt{\frac{\log{(1/\delta)}}{2n}}.\nonumber
    \end{align}
\end{Theorem}

A detailed proof is provided in Appendix~\ref{proof_5}. 

 	\begin{table*}[h]
		\caption{Comparison for estimating transition matrices on Pascal-VOC2007 dataset.}
		\centering
		\scriptsize
		\setlength\tabcolsep{5.8pt}
		\begin{tabular}{l|cc|cc|cc|cc}
			\hline	
			Noise rates $(\rho_{-}, \rho_{+})$  & (0,0.2) & (0,0.6) & (0.2,0) & (0.6,0)  & (0.1,0.1) & (0.2,0.2) &(0.017,0.2)& (0.034,0.4) \\	
			\hline	
			T-estimator max  &  3.89$\pm$0.03 & 10.52$\pm$0.58  & 3.01$\pm$0.12  & 4.47$\pm$0.22    & 3.18$\pm$0.22  &  5.28$\pm$0.20 & 3.99$\pm$0.10 &  6.28$\pm$0.44   \\	
			T-estimator 97\%  & 4.95$\pm$0.17  & \secbest{4.42$\pm$0.18} & 1.77$\pm$0.03 & \secbest{2.13$\pm$0.12}  & 6.99$\pm$0.10 & 6.94$\pm$0.17  & 5.38$\pm$0.14  & 5.17$\pm$0.09 \\	
			Dual T-estimator max & \secbest{1.94$\pm$0.13} & 7.29$\pm$0.16 & \secbest{1.03$\pm$0.04} & 2.68$\pm$0.13 & \textbf{2.13$\pm$0.23}  & \secbest{4.02$\pm$0.18}  &  \textbf{1.71$\pm$0.08} &  \secbest{2.67 $\pm$0.27} \\
			Dual T-estimator 97\%   & 12.59$\pm$0.06  &  7.43$\pm$0.06 & 1.09$\pm$0.03 &  2.41$\pm$0.33  & 14.39$\pm$0.10 &  11.78$\pm$0.06 &  13.71$\pm$0.16 &  11.15$\pm$0.09  \\
			\hline	
			Our estimator  &  
   \textbf{1.63$\pm$0.16} & \textbf{1.54$\pm$0.30} & \textbf{0.36$\pm$0.08} & \textbf{1.17$\pm$0.37} & \secbest{2.94$\pm$0.24} & \textbf{3.10$\pm$0.28} & \secbest{2.09$\pm$0.06} & \textbf{1.85$\pm$0.13} \\
			Our estimator gold &  0.44$\pm$0.05 & 0.51$\pm$0.09 & 0.38$\pm$0.08 & 0.37$\pm$0.11 & 0.83$\pm$0.05 & 2.15$\pm$0.30 & 0.65$\pm$0.10 & 1.40$\pm$0.20  \\
			\hline	
		\end{tabular}
		\label{est_VOC2007}
	\end{table*}
	\begin{table*}[h]
		\caption{Comparison for estimating transition matrices on Pascal-VOC2012 dataset.}
		\centering
		\scriptsize
		\setlength\tabcolsep{5.8pt}
		\begin{tabular}{l|cc|cc|cc|cc}
			\hline	
			Noise rates $(\rho_{-}, \rho_{+})$  & (0,0.2) & (0,0.6) & (0.2,0) & (0.6,0)  & (0.1,0.1) & (0.2,0.2) &(0.017,0.2)& (0.034,0.4) \\
			\hline	
			T-estimator max  &  3.90$\pm$0.01 & 10.28$\pm$0.33 & 2.87$\pm$0.09 &  4.55$\pm$0.08 & 3.29$\pm$0.07  &  5.25$\pm$0.15 &  4.05$\pm$0.04 &  6.82$\pm$0.20  \\	
			T-estimator 97\%  &  5.42$\pm$0.09 & \secbest{3.98$\pm$0.09} & 1.53$\pm$0.06 &  \secbest{1.91$\pm$0.07} & 6.43$\pm$0.16  &  6.20$\pm$0.17 &  5.76$\pm$0.27 &  5.16$\pm$0.14  \\
			Dual T-estimator max &  \secbest{1.02$\pm$0.20} & 5.13$\pm$0.26 & \secbest{1.07$\pm$0.07} &  2.06$\pm$0.12 &  \secbest{1.94$\pm$0.05} &  \secbest{2.59$\pm$0.16} &  \secbest{1.17$\pm$0.13} & \secbest{1.93$\pm$0.08} \\
			Dual T-estimator 97\%   &  12.94$\pm$0.06 & 7.49$\pm$0.03 & 1.14$\pm$0.04 &  2.94$\pm$0.18 & 14.23$\pm$0.08  &  11.56$\pm$0.05 &  13.97$\pm$0.09 &  11.10$\pm$0.08  \\
			\hline	
			Our estimator  &  \textbf{0.83$\pm$0.10} & \textbf{1.94$\pm$0.15} & \textbf{0.26$\pm$0.03} &  \textbf{0.91$\pm$0.12} & \textbf{1.74$\pm$0.22}  &  \textbf{1.79$\pm$0.17} &  \textbf{0.94$\pm$0.07} &  \textbf{1.07$\pm$0.14}  \\
			Our estimator gold &  0.33$\pm$0.05 & 0.34$\pm$0.05 & 0.25$\pm$0.05 & 0.45$\pm$0.05 & 0.51$\pm$0.05 & 1.67$\pm$0.29 & 0.42$\pm$0.06 & 0.91$\pm$0.16 \\
			\hline		
		\end{tabular}
		\label{est_VOC2012}
	\end{table*}
	\begin{table*}[h]
		\caption{Comparison for estimating transition matrices on MS-COCO dataset.}
		\centering
		\scriptsize
		\setlength\tabcolsep{5.8pt}
		\begin{tabular}{l|cc|cc|cc|cc}
			\hline	
			Noise rates $(\rho_{-}, \rho_{+})$  & (0,0.2) & (0,0.6) & (0.2,0) & (0.6,0)  & (0.1,0.1) & (0.2,0.2) & (0.008,0.2)& (0.015,0.4)\\	
			\hline	
			T-estimator max  &  16.14$\pm$0.33 & 39.09$\pm$0.47 & 10.39$\pm$0.21 &  11.49$\pm$0.60 & 13.95$\pm$0.41  &  20.50$\pm$0.04 &  16.70$\pm$0.06 &  28.16$\pm$0.45  \\	
			T-estimator 97\%  &  50.49$\pm$0.01 & 25.70$\pm$0.08 & 4.04$\pm$0.08 &  \secbest{3.70$\pm$0.02} & 51.17$\pm$0.16  &  39.45$\pm$0.11 &  49.96$\pm$0.18 &  37.54$\pm$0.10  \\
			Dual T-estimator max &  \textbf{5.04$\pm$0.04} & \textbf{11.22$\pm$0.70} & 4.65$\pm$0.07 &  9.55$\pm$0.84 & \secbest{13.02$\pm$0.45}  &  \secbest{15.79$\pm$0.38} &  \textbf{7.04$\pm$0.31} &  \textbf{6.34$\pm$0.11}  \\
			Dual T-estimator 97\%   &  61.49$\pm$0.02 & 30.97$\pm$0.03 & \secbest{1.53$\pm$0.00} &  7.86$\pm$0.12 & 64.20$\pm$0.02 &  48.67$\pm$0.01 &  63.12$\pm$0.02 &  46.91$\pm$0.01  \\
			\hline	
			Our estimator  &  \secbest{7.42$\pm$0.38} & \secbest{11.23$\pm$0.11} & \textbf{0.50$\pm$0.03} &  \textbf{0.83$\pm$0.06} & \textbf{8.88$\pm$0.10}  &  \textbf{10.27$\pm$0.19} &  \secbest{7.51$\pm$0.43} &  \secbest{8.77$\pm$0.20}  \\
			Our estimator gold &  0.82$\pm$0.03 & 0.80$\pm$0.04 & 0.40$\pm$0.04 & 0.66$\pm$0.04 & 1.94$\pm$0.04 & 8.14$\pm$0.06 & 0.95$\pm$0.05 & 2.02$\pm$0.08 \\
			\hline		                   
		\end{tabular}                
		\label{est_MSCOCO}     
	\end{table*}
 
\myPara{Remark~5.} The factor $(\sqrt{2d\log2}+1)\Pi_{i=1}^{d}M_i$ is induced by the hypothesis complexity of the deep neural network \cite{golowich2018size} (see Theorem 1 therein). 
This derived generalization error bound shows that the Reweight algorithm in multi-label cases can achieve a small difference between training error ($\bar{R}_{n,w}(\{{T}^j\}_{j=1}^q,\hat{{f}})$) and test error ($\mathbb{E}[\bar{R}_{n,w}(\{{T}^j\}_{j=1}^q,\hat{{f}})]$) with the rate of convergence $\mathcal{O}\left({\sqrt{1/n}}\right)$. Also, note that deep learning is powerful in yielding a small training error. Therefore, if the training sample size $n$ is large, then the upper bound in Theorem \ref{thm:main} is small, which implies a small $\mathbb{E}[\bar{R}_{n,w}(\{{T}^j\}_{j=1}^q,\hat{{f}})]$ and justifies why such a method with true transition matrices will have small test errors in the experiment section. Meanwhile, in Section~\ref{cls_ccn}, we show that given true transition matrices $\{{T}^j\}_{j=1}^q$, the Reweight method in multi-label cases can overall outperform the state-of-the-art methods in classification performance, implying that the small generalization error is not obtained at the cost of enlarging the approximation error.

Overall, according to Theorem~\ref{thm:main2} and Theorem~\ref{thm:main}, it can be guaranteed that when sampling bias is small and given sufficiently large examples, the estimated multi-label transition matrices will be accurate, and then, with such accurate transition matrices, the classifier learned by the Reweight algorithm will achieve a good generalization performance.

\section{Experiments}
	\label{experiments}
\subsection{Experiments on Class-Dependent Noisy Multi-Label Datasets}
\label{ccn}
	\myPara{Dataset.}
	We verify the effectiveness of the proposed method on three synthetic noisy multi-label datasets, \ie, Pascal-VOC2007~\cite{pascal-voc-2007}, Pascal-VOC2012~\cite{pascal-voc-2012}, and MS-COCO~\cite{lin2014microsoft}.
	Pascal-VOC2007~\cite{pascal-voc-2007} and Pascal-VOC2012~\cite{pascal-voc-2012} datasets are two popular datasets for object recognition. They both contain images from the same 20 object classes, with an average of $n_a=1.5$ labels per image. Pascal-VOC2007 contains
	a training set of 5,011 images and a test set of 4,952 images.
	Pascal-VOC2012 consists of 11,540 images as the training set and 10,991 images as the test set~\cite{MCAR_TIP_2021}. As the labels of the test set in Pascal-VOC2012 are not publicly available, we use the test set in Pascal-VOC2007 for Pascal-VOC2012 evaluation. MS-COCO~\cite{lin2014microsoft} is a widely used multi-label dataset. It contains 82,081 images as the training set and 40,137 images as the test set and covers 80 object classes with an average of $n_a=2.9$ labels per image.
	For these datasets, we corrupted the training sets manually according to true transition matrices $\{{{T}}^{j}\}_{j=1}^q$. For convenience, we use the same true transition matrices for all classes, \ie, ${T}^{j}={T}= \begin{pmatrix} 1-\rho_{-} & \rho_{-} \\ \rho_{+} & 1-\rho_{+} \end{pmatrix}$, but do not divulge this information for algorithms. We generate
	four different types of synthetic datasets by using different transition matrices: (1) $\rho_{-}=0, \rho_{+}= \rho$, which annotates some positive examples as negative examples, also known as multi-label learning with missing labels~\cite{Wu2014MultilabelLW,Wu2018MultilabelLW} (termed MLML); (2) $\rho_{-}=\rho, \rho_{+}= 0$, which annotates some negative examples as positive examples, also known as partial multi-label learning~\cite{Xie2018PartialML,Xie2022PartialML} (termed PML); (3) $\rho_{-}=\rho, \rho_{+}= \rho$, where positive examples and negative examples are mislabeled with the same probability $\rho$ (termed uniform label flipping, ULF); (4)  $\rho_{-}=\frac{n_a}{q-n_a}\rho, \rho_{+}= \rho$, where positive examples and negative examples are mislabeled with the same number (termed asymmetric label flipping, ALF). In the experiments, we test the algorithms on various $\rho$ (two cases for each noise type). 
	For all datasets, we leave out 10\% of the noisy training examples as a noisy validation set. We use mAP on noisy validation set as the criterion for model selection. 
	
	\myPara{Implementation details.} For a fair comparison, we implement all methods with default parameters by PyTorch on NVIDIA RTX 3090. We use a ResNet-50 network~\cite{he2016cvpr} pre-trained on ImageNet~\cite{russakovsky2015imagenet} for all datasets, and the optimizer is Adam optimizer~\cite{KingmaB14} with $\beta=(0.9,0.999)$. The batch size is 128, the learning rate is 5e-5. The number of training epochs is 20 for Pascal-VOC2007/VOC2012, and 30 for MS-COCO. For the transition matrix estimation method, $E_{warm}$ is the same as the normal training epoch. For our estimator, we perform sample selection based on the average losses of 5 epochs before a certain warmup epoch~(10th epoch for Pascal-VOC2007/VOC2012, 15th epoch for MS-COCO), $R=q-1$ and $\tau=0.5$ in all experiments. All experiments are run at least three times with different random seeds, and the average and standard deviation values of results are reported. The best results are in \textbf{bold}, and the second-best results are in \secbest{blue}.
 
	\subsubsection{Comparison for estimating transition matrices}
	\label{experiment1}
	\myPara{Baselines.} For evaluating the effectiveness of estimating the transition matrix under multi-label cases, we compare the proposed method with the following methods: (1) T-estimator max~\cite{Liu2016TPAMI,PatriniRMNQ17}, which estimates the transition matrix via the noisy class posterior probabilities for anchor points that have the largest estimated noisy class posteriors. (2) T-estimator 97\%~\cite{Liu2016TPAMI,PatriniRMNQ17}, which selects the points with 97\% largest estimated noisy class posteriors to be anchor points. (3) Dual T-estimator max~\cite{YaoL0GD0S20}, which introduces an intermediate class to avoid directly estimating the noisy class posterior, and selects the points with the largest estimated intermediate class posteriors to be the anchor points. (4) Dual T-estimator 97\%~\cite{YaoL0GD0S20}, which selects the points with the 97\% largest estimated intermediate class posteriors to be the anchor points.
 
	\myPara{Metrics.} We use the sum of estimation error for the transition matrices as the estimation evaluation metric, \ie, $\sum_{j=1}^q\|{T}^{j}-\hat{{T}}^{j}\|_{1}$.
	
	\myPara{Results.}
	In Tab.~\ref{est_VOC2007}, \ref{est_VOC2012} and \ref{est_MSCOCO}, we can see that for all cases on three datasets, the proposed estimation method leads to the smallest or second-smallest estimator errors across various noise rates. Note that since the fitting of the noisy or intermediate class posterior is hard to be accurate in noisy multi-label learning, the T-estimator and Dual T-estimator need to carefully tune a hyperparameter for better estimation under different noise rates, and it's very sensitive in some cases, \eg, MS-COCO datasets with noise rates $(0.1, 0.1)$. In contrast, our method uses the same hyperparameters on one dataset to get good results for all cases, which reflects its robustness to various noise rates.
	Besides, to study the ablation of sampling bias, we also run our method with an unbiased sample selection, named ``our estimator gold''. We can see that sample bias is the main factor that contributes to the error for our estimator, but the little error gap between our estimator and our estimator gold shows it will not lead to large estimation errors. 
	
	\begin{table}[t]
		\caption{Comparison for classification performance on Pascal-VOC2007 dataset with class-dependent label noise.}
  \renewcommand{\arraystretch}{1.0}
		\centering
		\scriptsize
		\setlength\tabcolsep{5.8pt}
		\begin{tabular}{l|l|cccc|c}
			\hline	
			& Noise type  &  MLML & PML & ULF & ALF  & Avg. \\	
			\hline
			\multirow{12}*{\rotatebox{90}{mAP}}&  Standard & 80.71 & 75.68 & 80.97 & 82.45 & 79.95\\
  & GCE & 78.59 & 75.25 & 81.54 & 82.54 & 79.48\\
  & CDR & 81.03 & 75.81 & 81.12 & 82.76 & 80.18\\
  & AGCN & 79.37 & 73.78 & 77.44 & 79.38 & 77.49\\
  & CSRA & \best{82.29} & 75.15 & 80.90 & \best{83.23} & 80.39\\
  & WSIC & 80.16 & 74.56 & 80.67 & 82.46 & 79.46\\
  & Reweight-T max & 80.59 & 77.20 & 81.29 & 82.54 & 80.41\\
  & Reweight-T 97\% & 81.49 & 78.52 & 80.76 & 82.90 & 80.92\\
  & Reweight-DualT max & 81.06 & 78.75 & \secbest{81.71} & \secbest{83.19} & \secbest{81.18}\\
  & Reweight-DualT 97\% & 80.04 & \best{80.16} & 77.48 & 79.62 & 79.33\\
\cline{2-7} & Reweight-Ours & \secbest{81.58} & \secbest{79.27} & \best{82.24} & 82.53 & \best{81.41}\\
  & Reweight-TrueT & 81.04 & 79.94 & 82.26 & 82.96 & 81.55\\
			\hline	
			\multirow{12}*{\rotatebox{90}{OF1}}&  Standard & 53.63 & 46.97 & 77.55 & 67.83 & 61.5\\
  & GCE & 56.15 & 47.07 & \secbest{78.00} & 68.53 & 62.44\\
  & CDR & 55.08 & 47.02 & 77.67 & 67.94 & 61.93\\
  & AGCN & 52.95 & 46.15 & 75.67 & 65.86 & 60.16\\
  & CSRA & 55.30 & 46.83 & \best{78.64} & 70.05 & 62.71\\
  & WSIC & 45.75 & 46.95 & 75.72 & 62.78 & 57.80\\
  & Reweight-T max & 59.26 & 63.67 & 76.79 & 71.86 & 67.90\\
  & Reweight-T 97\% & \secbest{73.00} & \secbest{75.42} & 70.16 & \secbest{76.28} & \secbest{73.72}\\
  & Reweight-DualT max & \best{73.60} & 72.72 & 72.62 & 62.80 & 70.44\\
  & Reweight-DualT 97\% & 64.99 & \best{76.64} & 44.87 & 62.48 & 62.25\\
\cline{2-7} & Reweight-Ours & 72.15 & 74.14 & 77.58 & \best{76.80} & \best{75.17}\\
  & Reweight-TrueT & 72.44 & 77.22 & 77.76 & 77.35 & 76.19\\
			\hline	
			\multirow{12}*{\rotatebox{90}{CF1}}& Standard & 51.59 & 45.90 & 73.27 & 62.59 & 58.34\\
  & GCE & 53.09 & 46.01 & 73.99 & 65.13 & 59.56\\
  & CDR & 53.07 & 45.92 & 73.44 & 63.75 & 59.05\\
  & AGCN & 53.43 & 44.66 & 72.65 & 63.95 & 58.67\\
  & CSRA & 53.77 & 45.50 & 75.48 & 66.79 & 60.39\\
  & WSIC & 41.89 & 45.05 & 69.56 & 54.42 & 52.73\\
  & Reweight-T max & 56.71 & 63.82 & 75.26 & 68.54 & 66.08\\
  & Reweight-T 97\% & \best{74.03} & \best{72.44} & 71.53 & \best{76.09} & \best{73.52}\\
  & Reweight-DualT max & 60.29 & 70.28 & 73.38 & 71.99 & 68.99\\
  & Reweight-DualT 97\% & 67.89 & 73.20 & 56.67 & 65.79 & 65.89\\
\cline{2-7} & Reweight-Ours & \secbest{69.08} & \secbest{72.34} & \best{76.06} & \secbest{74.73} & \secbest{73.05}\\
  & Reweight-TrueT & 69.57 & 73.22 & 73.25 & 73.97 & 72.50\\
			\hline		
		\end{tabular}
		\label{cls_VOC2007}
	\end{table}
 	\begin{table}[h]
  \renewcommand{\arraystretch}{1.0}
		\caption{Comparison for classification performance on Pascal-VOC2012 dataset with class-dependent label noise.}
		\centering
		\scriptsize
		\setlength\tabcolsep{5.8pt}
		\begin{tabular}{l|l|cccc|c}
			\hline	
			& Noise type  &  MLML & PML & ULF & ALF  & Avg. \\	
			\hline
			\multirow{12}*{\rotatebox{90}{mAP}}  & Standard & 83.00 & 80.92 & 84.23 & 84.47 & 82.72\\
  & GCE & 82.37 & 80.49 & \secbest{84.74} & 84.40 & 82.53\\
  & CDR & 83.31 & 81.07 & 84.40 & 84.70 & 82.93\\
  & AGCN & 82.42 & 81.23 & 82.42 & 83.30 & 82.02\\
  & CSRA & \best{84.32} & 80.24 & 83.56 & \best{85.19} & 82.71\\
  & WSIC & \secbest{83.57} & 81.27 & 83.84 & 84.63 & 82.89\\
  & Reweight-T max & 82.23 & 82.20 & 84.43 & 83.95 & 82.95\\
  & Reweight-T 97\% & 83.51 & 82.97 & 84.24 & 84.64 & 83.57\\
  & Reweight-DualT max & 82.31 & 83.24 & 85.22 & \secbest{84.79} & \secbest{83.59}\\
  & Reweight-DualT 97\% & 81.95 & \best{84.11} & 80.50 & 82.14 & 82.19\\
\cline{2-7} & Reweight-Ours & 83.17 & \secbest{83.33} & \best{84.83} & 84.74 & \best{83.78}\\
  & Reweight-TrueT & 84.26 & 83.95 & 85.04 & 85.12 & 84.42\\
\hline	
			\multirow{12}*{\rotatebox{90}{OF1}}  & Standard & 52.88 & 47.70 & 79.89 & 69.75 & 60.16\\
  & GCE & 51.46 & 47.81 & \best{80.13} & 70.79 & 59.80\\
  & CDR & 53.74 & 47.78 & 79.96 & 69.88 & 60.49\\
  & AGCN & 53.59 & 47.96 & 78.15 & 68.71 & 59.90\\
  & CSRA & 55.49 & 47.80 & 79.97 & 70.41 & 61.09\\
  & WSIC & 53.23 & 47.83 & 79.17 & 68.04 & 60.08\\
  & Reweight-T max & 60.86 & 65.06 & \secbest{80.08} & 73.35 & 68.67\\
  & Reweight-T 97\% & \best{75.55} & 78.33 & 74.55 & 76.55 & 76.14\\
  & Reweight-DualT max & 72.06 & \best{78.09} & 79.88 & \secbest{78.02} & \secbest{76.68}\\
  & Reweight-DualT 97\% & 59.62 & 79.08 & 43.35 & 61.43 & 60.68\\
\cline{2-7} & Reweight-Ours & \secbest{74.81} & \secbest{78.58} & 79.45 & \best{79.03} & \best{77.61}\\
  & Reweight-TrueT & 77.72 & 79.88 & 80.15 & 79.51 & 79.25\\
\hline	
			\multirow{12}*{\rotatebox{90}{CF1}}  & Standard & 52.46 & 46.97 & 77.33 & 68.20 & 58.92\\
  & GCE & 50.75 & 47.19 & 77.49 & 69.02 & 58.48\\
  & CDR & 53.41 & 47.03 & 77.35 & 68.07 & 59.26\\
  & AGCN & 53.75 & 47.12 & 75.69 & 66.24 & 58.85\\
  & CSRA & 54.37 & 47.17 & 77.70 & 67.97 & 59.75\\
  & WSIC & 53.33 & 47.24 & 76.30 & 65.58 & 58.96\\
  & Reweight-T max & 57.60 & 64.70 & \best{78.03} & 71.18 & 66.78\\
  & Reweight-T 97\% & \best{76.30} & \best{76.56} & 74.97 & \secbest{77.16} & \best{75.94}\\
  & Reweight-DualT max & 68.67 & \secbest{76.34} & 77.83 & 72.19 & 74.28\\
  & Reweight-DualT 97\% & 64.93 & 75.99 & 58.69 & 66.18 & 66.54\\
\cline{2-7} & Reweight-Ours & \secbest{72.79} & 76.33 & \secbest{77.90} & \best{77.27} & \secbest{75.67}\\
  & Reweight-TrueT & 75.39 & 76.91 & 77.76 & 77.56 & 76.69\\
			\hline		
		\end{tabular}
		\label{cls_VOC2012}
	\end{table}
  	\begin{table}[h]
		\caption{Comparison for classification performance on MS-COCO dataset with class-dependent label noise.}
  \vspace{-6pt}
  \renewcommand{\arraystretch}{1.0}
		\centering
		\scriptsize
		\setlength\tabcolsep{5.8pt}
		\begin{tabular}{l|l|cccc|c}
			\hline	
			& Noise type  &  MLML & PML & ULF & ALF  & Avg. \\	
			\hline
			\multirow{12}*{\rotatebox{90}{mAP}}& Standard & 66.87 & 61.11 & 65.14 & 68.29 & 64.37\\
  & GCE & 66.24 & 60.67 & 65.92 & 68.36 & 64.28\\
  & CDR & 66.96 & 61.26 & 65.33 & 68.36 & 64.52\\
  & AGCN & \best{68.62} & 62.58 & \best{66.69} & \secbest{69.38} & \secbest{65.96}\\
  & CSRA & \secbest{68.23} & 59.71 & 65.32 & \best{69.39} & 64.42\\
  & WSIC & 66.01 & 59.92 & 64.87 & 67.62 & 63.60\\
  & Reweight-T max & 66.97 & 62.84 & 65.57 & 68.00 & 65.13\\
  & Reweight-T 97\% & 65.25 & \secbest{63.72} & 62.86 & 66.27 & 63.94\\
  & Reweight-DualT max & 63.98 & 63.50 & 65.09 & 67.26 & 64.19\\
  & Reweight-DualT 97\% & 61.91 & 59.43 & 52.89 & 60.74 & 58.08\\
\cline{2-7} & Reweight-Ours & 66.93 & \best{65.63} & \secbest{66.58} & 68.55 & \best{66.38}\\
  & Reweight-TrueT & 68.14 & 65.93 & 67.38 & 69.25 & 67.15\\
\hline		\multirow{12}*{\rotatebox{90}{OF1}} & Standard & 42.83 & 38.32 & 66.74 & 58.89 & 49.30\\
  & GCE & 43.14 & 38.43 & 67.21 & 59.31 & 49.59\\
  & CDR & 43.03 & 38.39 & 66.72 & 59.26 & 49.38\\
  & AGCN & 41.58 & 38.84 & \best{67.87} & 59.50 & 49.43\\
  & CSRA & 46.03 & 38.60 & \secbest{67.30} & 59.59 & 50.64\\
  & WSIC & 44.85 & 38.02 & 65.05 & 59.23 & 49.31\\
  & Reweight-T max & 53.30 & 64.55 & 60.86 & 60.04 & \secbest{59.57}\\
  & Reweight-T 97\% & 55.86 & \secbest{67.26} & 38.73 & 53.61 & 53.95\\
  & Reweight-DualT max & \secbest{60.12} & 63.53 & 54.71 & \secbest{63.69} & 59.45\\
  & Reweight-DualT 97\% & 29.35 & 59.56 & 17.17 & 37.54 & 35.36\\
\cline{2-7} & Reweight-Ours & \best{65.92} & \best{68.15} & 61.53 & \best{67.67} & \best{65.20}\\
  & Reweight-TrueT & 68.29 & 68.46 & 68.77 & 69.38 & 68.51\\
\hline		\multirow{12}*{\rotatebox{90}{CF1}} & Standard & 41.50 & 35.87 & 59.58 & 52.70 & 45.65\\
  & GCE & 40.91 & 36.16 & 60.16 & 52.47 & 45.74\\
  & CDR & 41.34 & 36.17 & 59.52 & 52.63 & 45.68\\
  & AGCN & 40.55 & 36.67 & 61.53 & 52.12 & 46.25\\
  & CSRA & 43.32 & 36.43 & 59.94 & 52.78 & 46.56\\
  & WSIC & 44.14 & 35.37 & 56.8 & 53.02 & 45.44\\
  & Reweight-T max & 47.19 & \secbest{60.75} & 61.64 & 54.66 & 56.53\\
  & Reweight-T 97\% & 54.23 & 60.34 & 44.34 & 53.76 & 52.97\\
  & Reweight-DualT max & \secbest{58.48} & 59.07 & \secbest{63.81} & \secbest{63.03} & \secbest{60.45}\\
  & Reweight-DualT 97\% & 31.24 & 49.50 & 21.38 & 39.19 & 34.04\\
\cline{2-7} & Reweight-Ours & \best{62.32} & \best{62.03} & \best{63.89} & \best{64.68} & \best{62.75}\\
  & Reweight-TrueT & 62.71 & 62.11 & 62.68 & 64.06 & 62.50\\
			\hline		
		\end{tabular}
		\label{cls_coco}
		\vspace{-7pt} 
	\end{table}
   	\begin{table}[h]
		\caption{Comparison for classification performance on multi-label datasets with instance-dependent label noise.}
  \vspace{-6pt}
  \renewcommand{\arraystretch}{1.0}
		\centering
		\scriptsize
	\setlength\tabcolsep{2.5pt}
		\begin{tabular}{l|l|ccc}
			\hline	
			& Dataset  &  Pascal-VOC2007 & Pascal-VOC2012 & MS-COCO \\	
			\hline
			\multirow{11}*{\rotatebox{90}{mAP}}& Standard & \secbest{81.86} & 83.49 & 65.72\\
  & GCE & 81.80 & \secbest{83.75} & 66.02\\
  & CDR & 81.81 & 83.63 & 65.79\\
  & AGCN & 79.54 & 82.04 & \best{66.89}\\
  & CSRA & \best{82.59} & \best{83.81} & \secbest{66.45}\\
  & WSIC & 81.47 & \best{83.81} & 65.13\\
  & Reweight-T max & 81.32 & 83.08 & 65.77\\
  & Reweight-T 97\% & 81.00 & 83.32 & 63.30\\
  & Reweight-DualT max & 81.58 & 83.67 & 63.63\\
  & Reweight-DualT 97\% & 79.04 & 79.87 & 58.38\\
\cline{2-5} & Reweight-Ours & 81.45 & 83.73 & 66.25\\
\hline		\multirow{11}*{\rotatebox{90}{OF1}}   & Standard & 74.11 & 75.64 & 62.15\\
  & GCE & 73.96 & 75.47 & 62.15\\
  & CDR & 74.26 & 75.49 & 62.21\\
  & AGCN & 72.56 & 74.46 & 61.68\\
  & CSRA & 75.32 & 75.78 & 62.21\\
  & WSIC & 71.88 & 74.25 & 62.92\\
  & Reweight-T max & 75.21 & 76.90 & \secbest{63.94}\\
  & Reweight-T 97\% & 75.94 & 76.25 & 53.39\\
  & Reweight-DualT max & \secbest{76.56} & \best{78.48} & 62.40\\
  & Reweight-DualT 97\% & 64.17 & 61.35 & 33.75\\
\cline{2-5} & Reweight-Ours & \best{77.11} & \secbest{78.22} & \best{66.61}\\
\hline	
\multirow{11}*{\rotatebox{90}{CF1}} & Standard & 69.62 & 72.63 & 55.53\\
  & GCE & 69.81 & 72.41 & 55.67\\
  & CDR & 69.90 & 72.67 & 55.67\\
  & AGCN & 69.26 & 71.64 & 56.05\\
  & CSRA & 72.35 & 73.45 & 56.20\\
  & WSIC & 65.45 & 71.35 & 57.22\\
  & Reweight-T max & 71.92 & 74.19 & 57.17\\
  & Reweight-T 97\% & \secbest{74.27} & \secbest{75.70} & 51.82\\
  & Reweight-DualT max & 73.33 & 74.14 & \secbest{61.32}\\
  & Reweight-DualT 97\% & 67.49 & 66.41 & 37.33\\
\cline{2-5} & Reweight-Ours & \best{75.00} & \best{76.97} & \best{63.49}\\
			\hline		
		\end{tabular}
  \vspace{-11pt}
		\label{ins}
	\end{table}
 
	\subsubsection{Comparison for classification performance}
 \label{cls_ccn}
	\myPara{Baselines.} We exploit 10 baselines: (1) Standard, which trains with a standard multi-label classification loss. (2) GCE~\cite{zhang2018generalized}, which proposes a Generalized Cross-Entropy loss for robustness. (3) CDR~\cite{xia2021robust}, which performs different update rules for two types of parameters to achieve robust learning. (4) AGCN~\cite{ye2020add}, which adopts a dynamic GCN to model the relation of content-aware class representations. (5) CSRA~\cite{zhu2021residual}, which generates class-specific features for every category by proposing a spatial attention score. (6) WSIC~\cite{hu2019weakly} proposes to use a small set with clean labels to learn a residual net for regularization in noisy multi-label learning, and we only provide noisy datasets to it for a fair comparison. (7) Reweight-T max, which learns with a reweighting algorithm using transition matrices estimated by T-estimator max~\cite{PatriniRMNQ17}. (8) Reweight-T 97\%, which learns with a reweighting algorithm using transition matrices estimated by T-estimator 97\%~\cite{PatriniRMNQ17}. (9) Reweight-DualT max, which learns with a reweighting algorithm using transition matrices estimated by  Dual T-estimator max~\cite{YaoL0GD0S20}. (10) Reweight-DualT 97\%, which learns with a reweighting algorithm using transition matrices estimated by  Dual T-estimator 97\%~\cite{YaoL0GD0S20}. Note that Standard, AGCN, and CSRA are designed for clean multi-label data, and GCE and CDR are designed for noisy multi-class learning. 
	
	\myPara{Metrics.} Following conventional setting~\cite{pascal-voc-2007,ChenWWG19,RidnikBZNFPZ21,Ridnik2021MLDecoderSA}, we compute the mean average precision (mAP),  overall F1-measure (OF1) and per-class F1-measure (CF1) as classification evaluation metrics. For each image, we assign a positive label if its prediction probability is greater than 0.5.
	
	\myPara{Results.}
 As shown in Tab.~\ref{cls_VOC2007}, \ref{cls_VOC2012} and \ref{cls_coco}, we report the results of four noise types (MLML, PML, ULF, and ALF) and the overall average results on three multi-label datasets, where some observations are concluded.~\footnote{Note that to present the results more concise here, we only report the average number of two cases for each noise type. The complete numerical results can be found in Appendix~\ref{complete}} First, we can find those statistically consistent methods achieve the best or second-best results on all three metrics in the vast majority of cases, while other methods can only achieve good results in some cases. For example, on the Pascal-VOC2012 dataset with MLML label noise, CSRA achieves the best result in the mAP metric with the help of its well-designed network, but its performance is far lower than those consistent methods on the OF1 and CF1 metrics, which shows the learned model can not approximate well the true class posterior $\mathbbm{P}({Y}|{X})$. Note that since network structure and loss correction are compatible, the risk-consistent methods can also help AGCN and CSRA perform better (shown in Section~\ref{other2}). Second, 
 to show the performance of the consistent algorithm with true transition matrices, we also evaluate the performance of the Reweight method with true transition matrix (named "Reweight-TrueT"). Its excellent performance not only verifies the small generalization error with a sufficiently large number of training samples as analyzed in Section~\ref{generalization}, but also implies that it is not obtained at the cost of enlarging the approximation error.
Third, among those consistent methods with the estimated transition matrices, Reweight algorithm with our estimator (named "Reweight-Ours") obtains the best or second-best results on the three metrics, which is due to the smaller error of our estimation. Especially on the challenging and large-scale MS-COCO dataset, Reweight-Ours outperforms all state-of-the-art methods on the CF1 metric.

	\begin{table*}[h]
  \label{ablation1}
		\caption{Ablation study of the influence of repeated multiple estimations on Pascal-VOC2007 dataset with class-dependent label noise. The best results are in \textbf{bold}.}
		\centering
		\scriptsize
		\setlength\tabcolsep{5.8pt}
		\begin{tabular}{l|cc|cc|cc|cc}
			\hline	
			Noise rates $(\rho_{-}, \rho_{+})$  & (0,0.2) & (0,0.6) & (0.2,0) & (0.6,0)  & (0.1,0.1) & (0.2,0.2) &(0.017,0.2)& (0.034,0.4) \\	
			\hline	
			Our estimator ($R=1$)  &  2.05$\pm$0.33 & 2.04$\pm$0.44 & 0.76$\pm$0.10 & 3.06$\pm$0.74 & 4.20$\pm$0.40 & 4.65$\pm$0.56 & 2.91$\pm$0.26 & 3.49$\pm$0.33  \\	
			Our estimator (Avg.)  & 1.81$\pm$0.16 & \best{1.31$\pm$0.26} & 0.52$\pm$0.02 & 3.16$\pm$0.42 & 3.51$\pm$0.24 & \best{2.84$\pm$0.27} & 2.21$\pm$0.13 & \best{1.39$\pm$0.18} \\	
			Our estimator & \best{1.63$\pm$0.16} & 1.54$\pm$0.30 & \best{0.36$\pm$0.08} & \best{1.17$\pm$0.37} & \best{2.94$\pm$0.24} & 3.10$\pm$0.28& \best{2.09$\pm$0.06} & 1.85$\pm$0.13
 \\
			\hline	
		\end{tabular}
		\label{ablation1}
	\end{table*}
\begin{table*}[h]
		\caption{Comparison for estimation error with different numbers of training samples on Pascal-VOC dataset with class-dependent label noise.}
		\centering
		\scriptsize
		\setlength\tabcolsep{5.8pt}
		\begin{tabular}{c|cc|cc|cc|cc}
			\hline	
			Noise rates $(\rho_{-}, \rho_{+})$  & (0,0.2) & (0,0.6) & (0.2,0) & (0.6,0)  & (0.1,0.1) & (0.2,0.2) &(0.017,0.2)& (0.034,0.4) \\	
			\hline	
$n=$1000  &  0.22$\pm$0.22 & 0.31$\pm$0.37 & 0.65$\pm$0.09 & 0.58$\pm$0.10 & 1.45$\pm$0.23 & 3.07$\pm$0.36 & 1.12$\pm$0.20 & 2.96$\pm$0.33  \\
$n=$5000  &  0.13$\pm$0.12 & 0.21$\pm$0.07 & 0.38$\pm$0.14 & 0.49$\pm$0.16 & 0.76$\pm$0.10 & 2.32$\pm$0.12 & 0.48$\pm$0.14 & 1.35$\pm$0.21  \\
$n=$10000  &  0.16$\pm$0.09 & 0.18$\pm$0.10 & 0.25$\pm$0.02 & 0.47$\pm$0.10 & 0.64$\pm$0.04 & 1.96$\pm$0.21 & 0.40$\pm$0.09 & 0.94$\pm$0.09  \\
$n=$20000  &  0.07$\pm$0.02 & 0.10$\pm$0.06 & 0.19$\pm$0.01 & 0.38$\pm$0.09 & 0.36$\pm$0.07 & 1.30$\pm$0.09 & 0.23$\pm$0.02 & 0.58$\pm$0.03 \\
			\hline	
		\end{tabular}
		\label{ablation2}
	\end{table*}
	\begin{table*}[h]
		\caption{Comparison for estimation error of $T^{49}$ using different levels of label correlations on the MS-COCO dataset. The best results are in \textbf{bold}.}
		\centering
		\scriptsize		\setlength\tabcolsep{5.8pt}
		\begin{tabular}{l|cc|cc|cc|cc}
			\hline	
			Noise rates $(\rho_{-}, \rho_{+})$  &  (0,0.2) & (0,0.6) & (0.2,0) & (0.6,0)  & (0.1,0.1) & (0.2,0.2) &(0.017,0.2)& (0.034,0.4) \\		
			\hline	
			Using the correlations with label A  &
			\textbf{0.02$\pm$0.02} & \textbf{0.02$\pm$0.02} & \textbf{0.09$\pm$0.04} & \textbf{0.07$\pm$0.06} & \textbf{0.10$\pm$0.05} & \textbf{0.10$\pm$0.08} & \textbf{0.09$\pm$0.04} & \textbf{0.08$\pm$0.03}\\		
			Using the correlations with label B &0.05$\pm$0.06 & 0.05$\pm$0.07 & 0.13$\pm$0.11 & 0.15$\pm$0.14 & 0.19$\pm$0.12 & 0.18$\pm$0.11 & 0.17$\pm$0.10 & 0.17$\pm$0.11\\	
			\hline	
		\end{tabular}
		\label{ablation_n1}
	\end{table*}
 	\begin{table}[h]
		\caption{Comparison for classification performance with different loss correction ways on Pascal-VOC2007 dataset with class-dependent label noise.}
		\centering
		\scriptsize
		\setlength\tabcolsep{5.8pt}
		\begin{tabular}{l|l|cccc|c}
			\hline	
			& Noise type  &  MLML & PML & ULF & ALF  & Avg. \\	
			\hline
			\multirow{4}*{\rotatebox{90}{mAP}}&  Reweight-Ours & 81.58 & 79.27 & 82.24 & 82.53 & 81.41\\
&Backward-Ours & 75.32 & 72.48 & 74.72 & 75.25 & 74.44\\
&Forward-Ours & \best{82.60} & 77.43 & 79.98 & 82.39 & 80.60\\
&Revision-Ours & 82.13 & \best{79.41} & \best{82.55} & \best{83.56} & \best{81.91}\\
			\hline	
			\multirow{4}*{\rotatebox{90}{OF1}}&  Reweight-Ours & 72.15 & 74.14 & 77.58 & 76.80 & 75.17\\
&Backward-Ours & 45.67 & 46.56 & 68.98 & 54.03 & 53.81\\
&Forward-Ours & \best{75.24} & 69.90 & 76.85 & \best{77.55} & 74.89\\
&Revision-Ours & 71.52 & \best{74.97} & \best{77.95} & 77.53 & \best{75.49}\\
			\hline	
			\multirow{4}*{\rotatebox{90}{CF1}}& Reweight-Ours & \best{69.08} & 72.34 & 76.06 & 74.73 & \best{73.05}\\
&Backward-Ours & 42.83 & 45.02 & 56.23 & 44.67 & 47.19\\
&Forward-Ours & 71.25 & 69.70 & 74.09 & 74.75 & 72.45\\
&Revision-Ours & 67.69 & \best{72.58} & \best{76.25} & \best{74.99} & 72.88\\
			\hline		
		\end{tabular}
		\label{ablation3}
	\end{table}
 	\begin{table}[h]
		\caption{Comparison for classification performance with different base multi-label learning algorithms on Pascal-VOC2007 dataset with class-dependent label noise.}
		\centering
		\scriptsize
		\setlength\tabcolsep{3.5pt}
		\begin{tabular}{l|l|cccc|c}
			\hline	
			& Noise type  &  MLML & PML & ULF & ALF  & Avg. \\	
			\hline
			\multirow{6}*{\rotatebox{90}{mAP}}&Standard & 80.71 & 75.68 & 80.97 & 82.45 & 79.95\\
&+R-Ours & 81.58{\tiny (\textcolor{red}{+0.87})} & 79.27{\tiny (\textcolor{red}{+3.59})} & 82.24{\tiny (\textcolor{red}{+1.27})} & 82.53{\tiny (\textcolor{red}{+0.08})} & 81.41{\tiny (\textcolor{red}{+1.46})}\\ \cline{2-7}
&AGCN & 79.37 & 73.78 & 77.44 & 79.38 & 77.49\\
&+R-Ours & 82.71{\tiny (\textcolor{red}{+3.34})} & 74.74{\tiny (\textcolor{red}{+0.96})} & 80.99{\tiny (\textcolor{red}{+3.55})} & 82.57{\tiny (\textcolor{red}{+3.19})} & 80.25{\tiny (\textcolor{red}{+2.76})}\\ \cline{2-7}
&CSRA & 82.29 & 75.15 & 80.9 & 83.23 & 80.39\\
&+R-Ours & 81.74{\tiny (\textcolor{brown}{-0.55})} & 80.04{\tiny (\textcolor{red}{+4.89})} & 82.85{\tiny (\textcolor{red}{+1.95})} & 84.19{\tiny (\textcolor{red}{+0.96})} & 82.21{\tiny (\textcolor{red}{+1.82})}\\
			\hline	
			\multirow{6}*{\rotatebox{90}{OF1}}&Standard & 53.63 & 46.97 & 77.55 & 67.83 & 61.5\\
&+R-Ours & 72.15{\tiny (\textcolor{red}{+18.52})} & 74.14{\tiny (\textcolor{red}{+27.17})} & 77.58{\tiny (\textcolor{red}{+0.03})} & 76.80{\tiny (\textcolor{red}{+8.97})} & 75.17{\tiny (\textcolor{red}{+13.67})}\\\cline{2-7}
&AGCN & 52.95 & 46.15 & 75.67 & 65.86 & 60.16\\
&+R-Ours & 75.82{\tiny (\textcolor{red}{+22.87})} & 62.98{\tiny (\textcolor{red}{+16.83})} & 76.10{\tiny (\textcolor{red}{+0.43})} & 77.64{\tiny (\textcolor{red}{+11.78})} & 73.14{\tiny (\textcolor{red}{+12.98})}\\\cline{2-7}
&CSRA & 55.30 & 46.83 & 78.64 & 70.05 & 62.71\\
&+R-Ours & 71.05{\tiny (\textcolor{red}{+15.75})} & 69.35{\tiny (\textcolor{red}{+22.52})} & 78.24{\tiny (\textcolor{brown}{-0.40})} & 79.22{\tiny (\textcolor{red}{+9.17})} & 74.47{\tiny (\textcolor{red}{+11.76})}\\
			\hline	
			\multirow{6}*{\rotatebox{90}{CF1}}&Standard & 51.59 & 45.9 & 73.27 & 62.59 & 58.34\\
&+R-Ours & 69.08{\tiny (\textcolor{red}{+17.49})} & 72.34{\tiny (\textcolor{red}{+26.44})} & 76.06{\tiny (\textcolor{red}{+2.79})} & 74.73{\tiny (\textcolor{red}{+12.14})} & 73.05{\tiny (\textcolor{red}{+14.71})}\\\cline{2-7}
&AGCN & 53.43 & 44.66 & 72.65 & 63.95 & 58.67\\
&+R-Ours & 73.66{\tiny (\textcolor{red}{+20.23})} & 64.54{\tiny (\textcolor{red}{+19.88})} & 74.93{\tiny (\textcolor{red}{+2.28})} & 76.21{\tiny (\textcolor{red}{+12.26})} & 72.34{\tiny (\textcolor{red}{+13.67})}\\\cline{2-7}
&CSRA & 53.77 & 45.5 & 75.48 & 66.79 & 60.39\\
&+R-Ours & 68.97{\tiny (\textcolor{red}{+15.2})} & 70.46{\tiny (\textcolor{red}{+24.96})} & 77.13{\tiny (\textcolor{red}{+1.65})} & 76.95{\tiny (\textcolor{red}{+10.16})} & 73.38{\tiny (\textcolor{red}{+12.99})}\\
			\hline		
		\end{tabular}
		\label{ablation4}
	\end{table}
 
\subsection{Experiments on Instance-Dependent Noisy Multi-Label Datasets}
\label{instance}
\myPara{Datasets.} To illustrate the applicability in realistic scenarios, we perform the experiments with two types of instance-dependent label noise on Pascal-VOC2007, Pascal-VOC2012 and MS-COCO datasets: (1) Pair-wise label noise~\cite{Han2018NIPS}, where one class label is mistaken as another class label with a certain probability; (2) PMD label noise~\cite{ZhangZW0021}, where data points near the decision boundary are harder to distinguish and more likely to be mislabeled. Both of them are dependent on class labels and instances, which are much more realistic than class-dependent label noise.  In the experiments, we test the algorithms on cases for each noise type. Same to Section~\ref{ccn}, for all datasets, we also leave out 10\% of the noisy training examples as a noisy validation set, and use mAP on noisy validation set as the criterion for model selection. Note that as class labels in those multi-label datasets are unbalanced, we only mislabel one class into another class when the class change does not occur.

\myPara{Implementation details.} We use the same baselines and implementation way as the experiments in Section~\ref{ccn}.

\myPara{Results.} The comparison for classification performance with instance-dependent label noise can be seen in Tab.~\ref{ins}, where we report the overall average classification performance in three noisy multi-label datasets.
The results show the proposed method with our estimator (Reweight-Ours) can achieve the best or second-best performance among the state-of-the-art methods in the OF1 and CF1 metrics. Especially, on the MS-COCO dataset, it outperforms all baselines by a large margin (at least +2.67\%) in the OF1 metric. Besides, in terms of the mAP metric, Reweight-Ours consistently outperforms the Reweight algorithm with other estimators across three datasets, showing the superiority of our multi-label transition matrix estimator.

\section{Ablation Study}
\label{other}
\subsection{Ablation Study about Our Multi-Label Transition Matrix Estimation}
\label{other1}
\myPara{Influence of repeated multiple estimations.}
To exploit the correlations among multiple labels, our estimator performs $R$ estimations and obtains the final estimation by Eq.~(\ref{final}). To study the ablation of repeated estimations, we test its estimation error when performing one estimation, \ie, our estimator ($R=1$) and our estimator (Avg.). As shown in Tab.~\ref{ablation1}, our estimator can consistently get a significantly better estimation than our estimator ($R=1$) across various cases, revealing the improvement brought by repeated estimations. Further, we evaluate the estimation error with the average aggregation rule. The results in most cases verified the assumption that the error of the obtained estimation by Eq.~(\ref{final}) is smaller than that of the average estimation from multiple estimations.
 
\myPara{Influence of the number of training examples.} To study the ablation of the number of training samples for our estimator, we consider testing its estimation errors with different training samples. Note that to eliminate the influence of sampling bias here, we actually test the estimation errors of "our estimator gold" where the sample selection is unbiased. As shown in Tab.~\ref{ablation2}, the estimated multi-label transition matrices will become more accurate as the number of training examples $n$ increases, which verifies our theoretical analyses in Section~\ref{est_error}.

\myPara{Influence of label correlations}
	According to the results in Appendix~\ref{Validation}, we divide all labels on the MS-COCO dataset into two categories: one is very likely to have a correlation with class label 50 (the label belonging to this category is termed "label A") and the other is less likely to has a correlation with class label 50 (the label belonging to this category is termed as "label B"). Tab.~\ref{ablation_n1} shows the estimation error of the transition matrix for class label 50 using the correlations with label A (or B) in various cases.  The results represent using the correlations with label A, the estimation error of $T^{49}$ will be smaller, showing stronger label correlations can lead to better estimation in our approach.

\subsection{Ablation Study about Our Multi-Label Classifier Learning}
\label{other2}
\myPara{Influence of loss correction ways.}
Many statistically consistent algorithms~\cite{Liu2016TPAMI,PatriniRMNQ17,XiaLW00NS19} consist of a two-step training procedure. The first step estimates the transition matrix, and the second step builds statistically consistent algorithms via modifying loss functions. Our proposed estimator can be seamlessly embedded into their frameworks. Here, we compare classification performance of applying our estimation method to four different loss correction ways: Reweight~\cite{Liu2016TPAMI} (named Reweight-Ours), Backward~\cite{natarajan2013learning} (named Backward-Ours), Forward~\cite{PatriniRMNQ17} (named Forward-Ours) and T-Revision~\cite{XiaLW00NS19} (named Revision-Ours) on Pascal-VOC2007 dataset. As shown in Tab.~\ref{ablation3}, T-Revision with our estimator achieves the best performance in most cases, which may be because it can further tune the transition matrices via classification learning.

\myPara{Influence of base multi-label learning algorithms.}
Recently, many advanced multi-label learning algorithms~\cite{zhu2021residual,ye2020add,MCAR_TIP_2021}, which designed exquisite networks for multi-label learning, have been proposed, and they perform well on the clean dataset. As they also use binary cross entropy loss function, we can also apply the Reweight method with our estimator to their frameworks. Here, we compare the classification performance of applying the Reweight method with our estimator to three different base learning algorithms:
	Standard, AGCN~\cite{ye2020add} and CSRA~\cite{zhu2021residual} on Pascal-VOC2007 dataset. Besides, to show the importance of consistent algorithms, we also present the performance of base learning algorithms here. As shown in Tab.~\ref{ablation4}, the consistent algorithms with our estimator can work well with more advanced networks.

 \subsection{Ablation Study about the Hyperparameter $\tau$}
	\label{discuss3}
	The sample selection we adopted~\cite{Sanchez2019UnsupervisedLN,LiSH20} is to estimate this clean probability of examples by modeling loss values with a GMM model using the Expectation-Maximization algorithm.  If the clean sample can be distinguished according to loss values and its estimated probability is accurate, the best threshold will be about 0.5. Hence, $\tau=0.5$ is a typical value in related works~\cite{Sanchez2019UnsupervisedLN,LiSH20}, and we follow this practice in our experiments. Here, we conduct the ablation study about the sample selection threshold on  Pascal-VOC2007 dataset with class-dependent label noise. As shown in Fig.~\ref{learning_process}, $\tau=0.5$ is a good choice both according to mAP scores on the noisy validation set and according to mAP scores on the clean test dataset.
	
	\begin{figure}[h]
		\centering
		\subfloat[]{\includegraphics[width=1.0\linewidth,trim=5 10 5 10,clip]{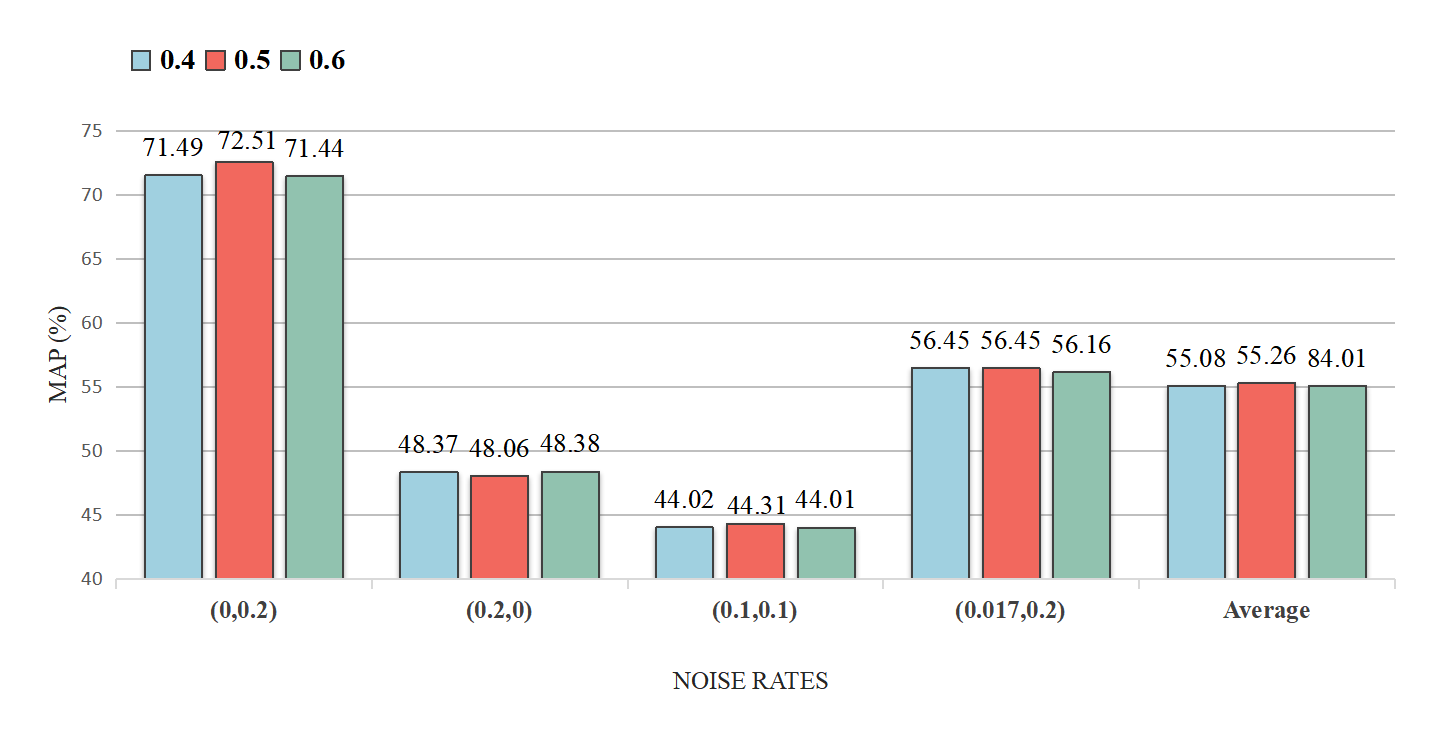}}
		\hfil
		\subfloat[]{\includegraphics[width=1.0\linewidth,trim=5 10 5 10,clip]{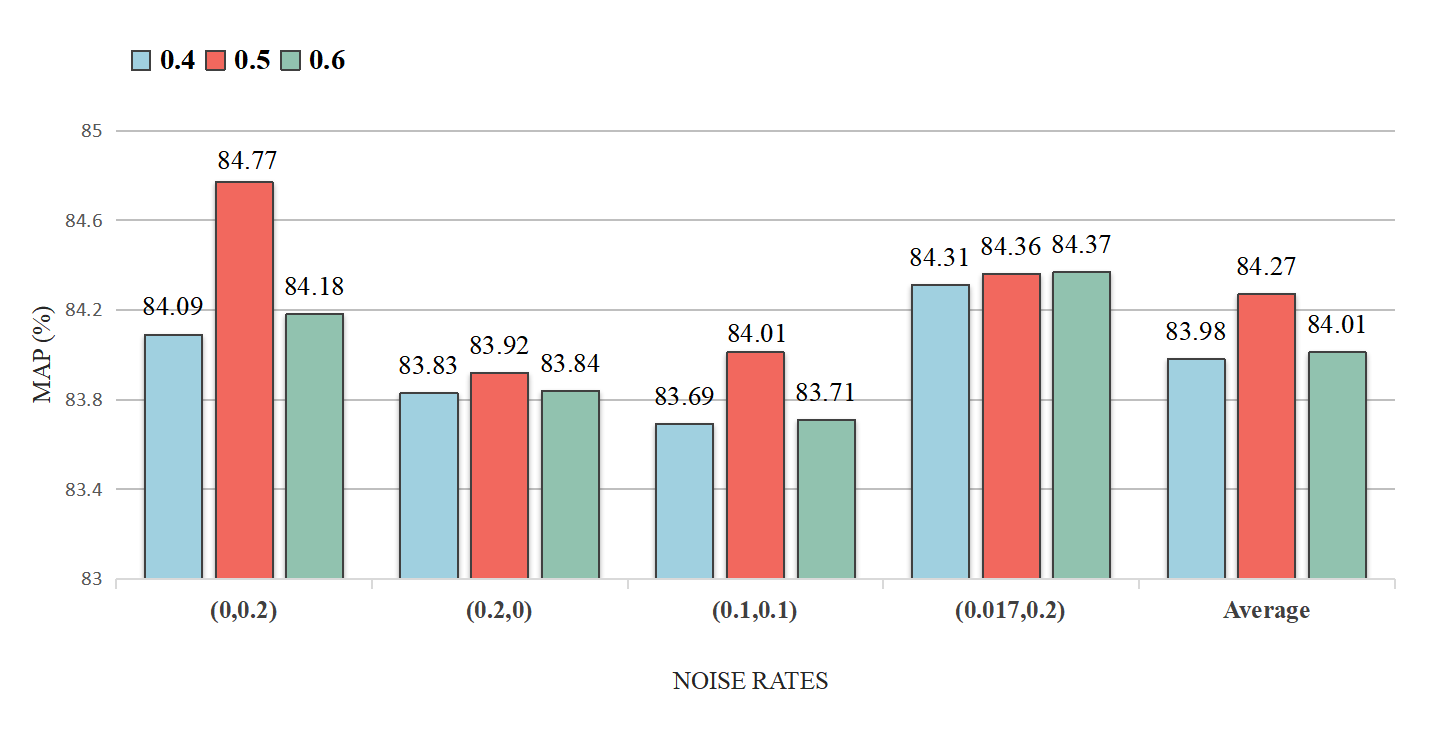}}
		\caption{(a) mAP scores on the noisy validation set of the Reweight method with our estimator with $\tau=0.4,0.5,0.6$; (b) mAP scores on the clean test dataset of the Reweight method with our estimator with $\tau=0.4,0.5,0.6$. }
		\label{learning_process}
  \vspace{-5pt}
	\end{figure}



 
	\section{Conclusion}
	\label{conclusion}
	In this paper, we study the estimation problem of the transition matrices in the noisy multi-label setting. We prove some identifiability results of class-dependent transition matrices in such a setting, inspired by which we propose a new estimator to estimate the transition matrix. The proposed estimator utilizes the information of label correlations, and demands neither anchor points nor accurate fitting of noisy class posterior. The estimation error bound and generalization error bound theoretically justify the effectiveness of the proposed method. Experiments on popular multi-label datasets illustrate the superiority of the proposed estimator to accurately estimate transition matrices, and the consistent algorithms with this estimator achieve better classification performance. 

\ifCLASSOPTIONcaptionsoff
  \newpage
\fi



\bibliographystyle{IEEEtran}
\bibliography{IEEEabrv,multi-label}
%
\vspace{-30pt}
\begin{IEEEbiography}[{\includegraphics[width=1in,height=1.25in,clip,keepaspectratio]{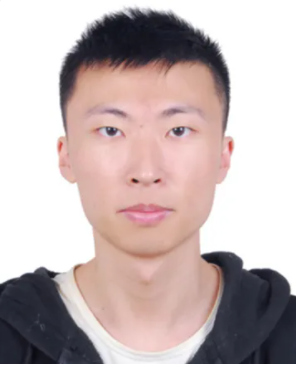}}]
		{Shikun Li} received the B.S. degree from the School of Information and Communication Engineering, Beijing University of Posts and Telecommunications (BUPT), Beijing, China. He is currently pursuing the Ph.D. degree with the Institute of Information Engineering, Chinese Academy of Sciences, Beijing, and the School of Cyber Security, University of Chinese Academy of Sciences, Beijing. His research interests include machine learning, data analysis and computer vision.
	\end{IEEEbiography}
\begin{IEEEbiography}[{\includegraphics[width=1in,height=1.25in,clip,keepaspectratio]{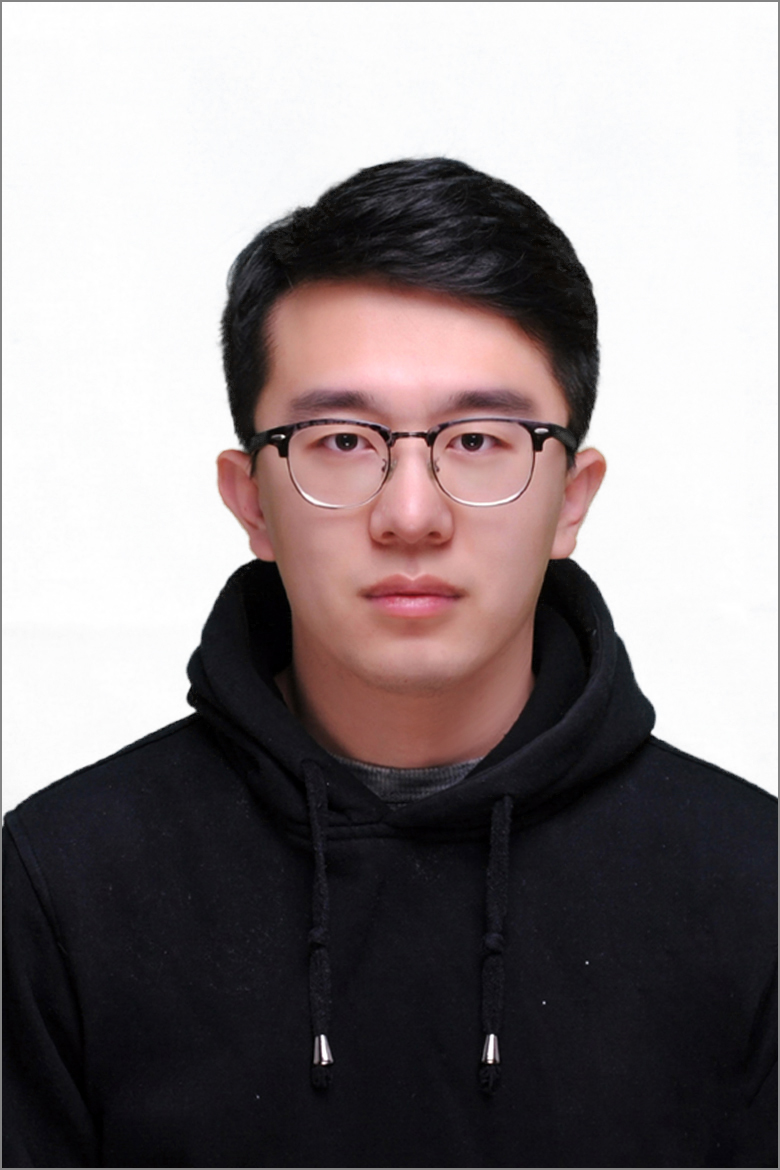}}]{Xiaobo Xia} received the B.E. degree in
telecommunications engineering from Xidian University, in 2020. He is currently pursuing a Ph.D. degree in computer science from the University of Sydney. He has published
more than 20 papers at top-tiered conferences
or journals such as IEEE T-PAMI, ICML, ICLR, NeurIPS, CVPR, ICCV, and KDD. He also serves as the reviewer for top-tier conferences such as ICML, NeurIPS, ICLR, CVPR, ICCV, and ECCV. His research interest lies in machine learning, with a particular emphasis on weakly-supervised learning. He was a recipient of the Google Ph.D. Fellowship~(machine learning) in 2022. 
\end{IEEEbiography}

\begin{IEEEbiography}[{\includegraphics[width=1in,height=1.25in,clip,keepaspectratio]{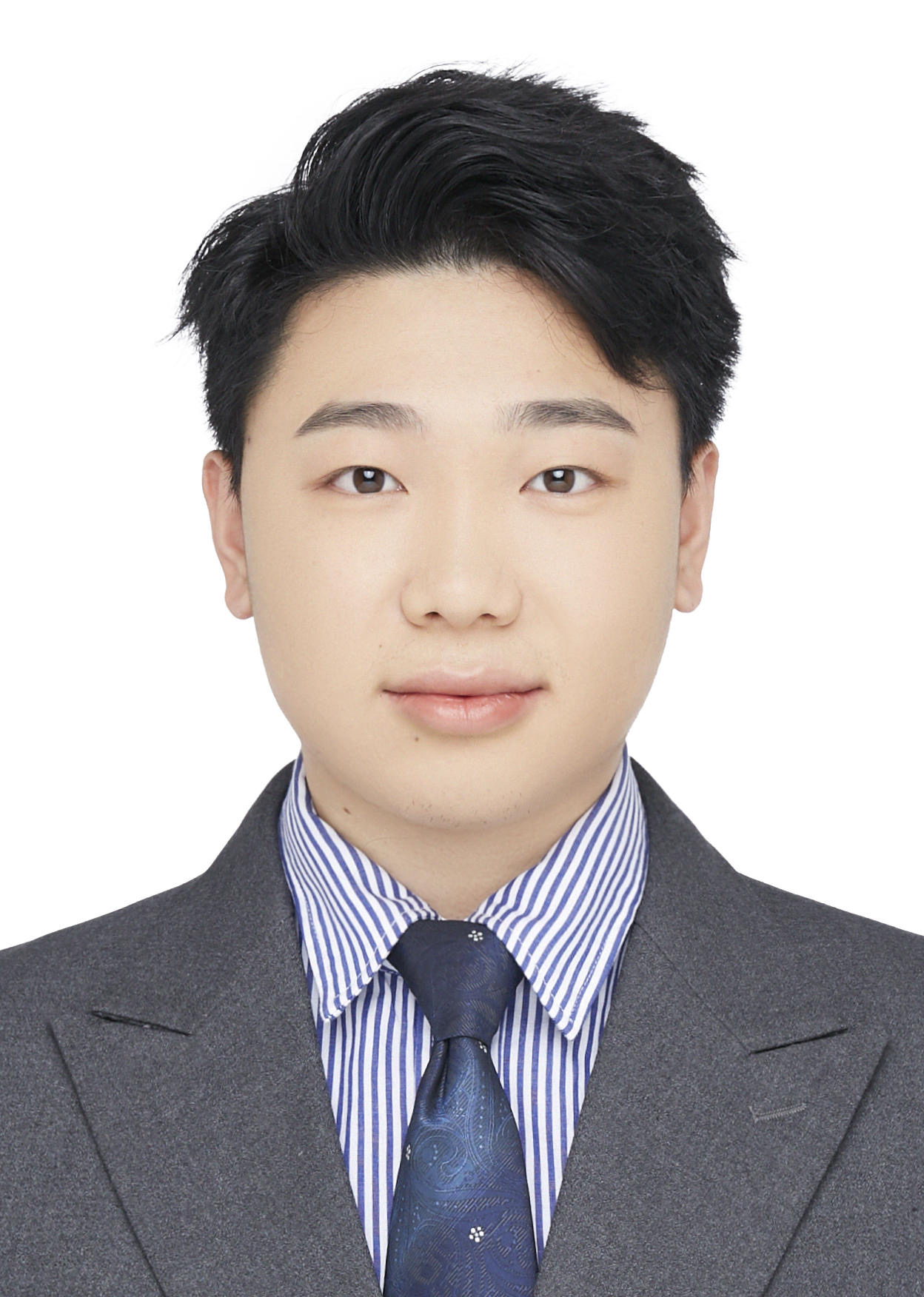}}]
		{Hansong Zhang} received the B.S. degree from the School of Mathematics and physics in North China Electric Power University, Beijing, China. He is currently a Ph.D. candidate in the Institute of Information Engineering, Chinese Academy of Sciences, Beijing, and the School of Cyber Security, University of Chinese Academy of Sciences, Beijing. His research interests include machine learning, data-centric AI, and computer vision.
	\end{IEEEbiography}
\begin{IEEEbiography}[{\includegraphics[width=1in,height=1.25in,clip,keepaspectratio]{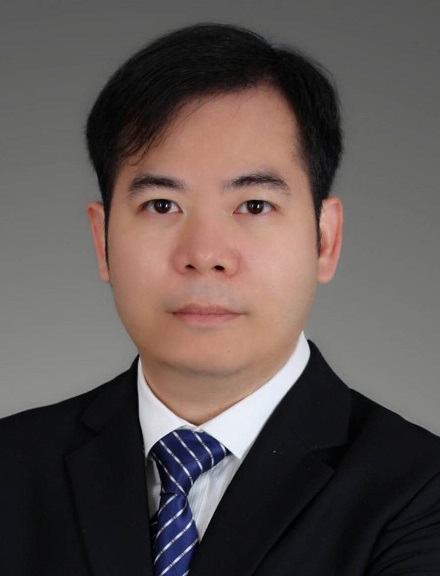}}]{Shiming Ge}
		(M'13-SM'15) is a Professor with the Institute of Information Engineering, Chinese Academy of Sciences. He is also the member of Youth Innovation Promotion Association, Chinese Academy of Sciences. Prior to that, he was a senior researcher and project manager in Shanda Innovations, a researcher in Samsung Electronics and Nokia Research Center. He received the B.S. and Ph.D degrees both in Electronic Engineering from the University of Science and Technology of China (USTC) in 2003 and 2008, respectively. His research mainly focuses on computer vision, data analysis, machine learning and AI security, especially efficient learning solutions towards scalable applications. He is a senior member of IEEE, CSIG and CCF.
	\end{IEEEbiography}
\begin{IEEEbiography}[{\includegraphics[width=1in,height=1.25in,clip,keepaspectratio]{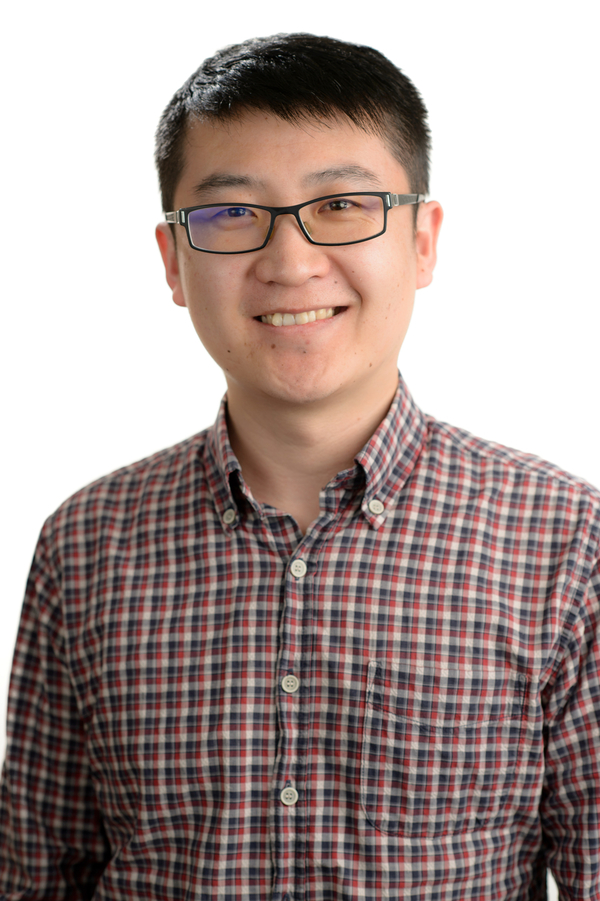}}]{Tongliang Liu} (Senior Member, IEEE) is the Director of Sydney AI Centre at the University of Sydney.
He is broadly interested in the fields of trustworthy
machine learning and its interdisciplinary applications, with a particular emphasis on learning with
noisy labels, adversarial learning, transfer learning,
unsupervised learning, and statistical deep learning
theory. He has authored and co-authored more than
200 research articles including ICML, NeurIPS,
ICLR, CVPR, AAAI, IJCAI, JMLR, and TPAMI.
His monograph on machine learning with noisy
labels will be published by MIT Press. He is/was a (senior-) meta reviewer
for many conferences, such as ICML, NeurIPS, ICLR, UAI, AAAI, IJCAI,
and KDD, and was a notable AC for ICLR.
Dr. Liu is an Associate Editor of TMLR and is on the Editorial Boards of
JMLR and MLJ. He is a recipient of the AI’s 10 to Watch Award from IEEE in
2023, the Future Fellowship Award from Australian Research Council (ARC)
in 2022, and the Discovery Early Career Researcher Award (DECRA) from
ARC in 2018.
\end{IEEEbiography}
 



	\newpage
    \onecolumn
	\begin{appendices}
	
	\section{Definition of Class-Dependent Label Noise}
	\label{discuss5}
	In this paper, the class-dependent multi-label noise for class label $Y^j$ means that the flip probabilities of $\bar{Y}^j$ are only dependent on the value of class label $Y^j$, \ie, $Y^j=0$ or $Y^j=1$. The corresponding class-dependent transition matrix represents $\mathbbm{P}(\bar{Y}^j|Y^j)$. The differences between this definition and the class-dependent label noise in the single-label cases are as follows.
	
	First, the class-dependent label noise in the single-label cases represents the flip probability from class $i$ to class $j$ ($i$ and $j$ are two different classes), while in this paper, the class-dependent label noise for class label $Y^j$ is only dependent on $Y^j=0$ or $Y^j=1$, which is independent on another class label $Y^i$, \ie, $\mathbbm\mathbbm{P}\left(\bar{Y}^{j}| Y^{j}, Y^{i}\right) = \mathbbm\mathbbm{P}\left(\bar{Y}^{j}| Y^{j} \right)$.
	
	Second, the class-dependent label noise in the single-label cases can be modeled by a $C \times C$ transition matrix, bridging the transition from clean single label to noisy single label,   while in this paper, the class-dependent label noise for class label $Y^j$ can be modeled by a $2 \times 2$ transition matrix, bridging the transition from clean label $Y^j$ to noisy label $\bar{Y}^{j}$.
	
	Third, the definition of class-dependent multi-label noise in this paper can be extended to instance-dependent multi-label noise, where the flip probabilities of $\bar{Y}^j$ are dependent on class label $Y^j$ and instance feature $X$. Such instance-dependent label model can sufficiently model various multi-label noise cases such as missing multi-labels~\cite{Wu2014MultilabelLW,Wu2018MultilabelLW}, partial multi-labels~\cite{Xie2018PartialML,Xie2022PartialML}, pair-wise label noise~\cite{Han2018NIPS,Li2020aaai,Li2022TMM}, and PMD label noise~\cite{ZhangZW0021}. While the instance-dependent label noise in the single-label cases can not simultaneously model such complex multi-label noise.
	\section{Reweight Algorithm for Noisy Multi-Label Learning}
	\label{Reweight}	
 In this appendix, we present applying a risk-consistent algorithm, \ie, Reweight~\cite{Liu2016TPAMI,XiaLW00NS19}, in noisy multi-label learning.
	First of all, we decompose the task into $q$ independent binary classification problems given ${X}$, which is a widely used assumption for deep multi-label learning~\cite{ChenWWG19,RidnikBZNFPZ21,Xie2022TPAMI} and the surrogate loss is as:
	\begin{equation}
		\mathcal{L}({{f}}({X}), {Y})=\sum_{j=1}^{q} \ell(f_j({X}),Y^{j})
		\label{standard}
	\end{equation}
	where ${f}=(f_1,f_2,...,f_q)$ is the learnable $q$ classification functions, and $\ell$ is the base loss function.
	In the deep learning commonunity, ${f}$ is usually modeled by a deep nerual network with the outputs of $q$ sigmoid functions, and $\ell$ is usually the binary cross entropy function.
	
	Similar to the single-label case~\cite{Liu2016TPAMI,XiaLW00NS19}, Reweight method employs the importance reweighting technique to rewrite the expected risk w.r.t. clean data:
	\begin{equation}
		\begin{aligned}
			&R({f})=\mathbb{E}_{({X}, {Y}) \sim D}[\mathcal{L}({{f}}({X}), {Y})]=
			\sum_{j=1}^{q} \mathbb{E}_{({X}, {Y}) }[\ell(f_j({X}),Y^{j})]\\
			&=\sum_{j=1}^{q} \int_{{x}} \sum_{i} \mathbbm{P}_{D}({X}={x}, Y^{j}=i) \ell(f_j({x}), i) d {x} \\
			&=\sum_{j=1}^{q} \int_{{x}} \sum_{i} \mathbbm{P}_{\bar{D}}({X}={x}, \bar{Y}^{j}=i) \frac{\mathbbm{P}_{D}({X}={x}, Y^{j}=i)}{\mathbbm{P}_{\bar{D}}({X}={x}, \bar{Y}^{j}=i)} \ell(f_j({x}), i) d {x} \\
			&=\sum_{j=1}^{q} \int_{{x}} \sum_{i} \mathbbm{P}_{\bar{D}}({X}={x}, \bar{Y}^{j}=i) \frac{\mathbbm{P}_{D}(Y^{j}=i \mid {X}={x})}{\mathbbm{P}_{\bar{D}}(\bar{Y}^{j}=i \mid {X}={x})} \ell(f_j({x}), i) d {x} \\
			&=\sum_{j=1}^{q} \mathbb{E}_{({X}, \bar{{Y}})}[\bar{\ell}_j(f_j({X}),\bar{Y}^{j})] =\mathbb{E}_{({X}, \bar{{Y}}) \sim \bar{D}}[\bar{\mathcal{L}}({{f}}({X}), \bar{{Y}})],
		\end{aligned}
	\end{equation}
	where $D$ denotes the distribution for clean data, $\bar{D}$ for noisy data, $\bar{\ell}_j(f_j({x}), i)= \frac{\mathbbm{P}_{D}(Y^{j}=i \mid {X}={x})}{\mathbbm{P}_{\bar{D}}(\bar{Y}^{j}=i \mid {X}={x})} \ell(f_j({x}), i)$, $\bar{\mathcal{L}}({{f}}({x}), \bar{{y}})= \sum_{j=1}^{q} \bar{\ell}_j(f_j({x}), \bar{y}^{j})$, and the third last equation holds because label noise is assumed to be independent of instances. In the paper, we have omitted the subscript for $P$ when no confusion is caused.
	
	We use the output of the sigmoid function $g_j({x})$ to approximate $\mathbbm{P}(Y^{j}=1|{X}={x})$, \ie, $\mathbbm{P}(Y^{j}=1|{X}={x}) \approx \mathbbm{\hat{P}}(Y^{j}=1|{X}={x}) = g_j({x})$ and $\mathbbm{P}(Y^{j}=0|{X}={x}) \approx \mathbbm{\hat{P}}(Y^{j}=0|{X}={x}) = 1-g_j({x})$. Then, $\mathbbm{\hat{P}}(\bar{Y}^{j}=k \mid {X}={x}) = \sum_{i=0}^{1} {T}^{j}_{ik} \mathbbm{\hat{P}}(Y^{j}=i \mid{X}={x})$ is an approximation for $\mathbbm{P}(\bar{Y}^{j}=k \mid {X}={x})$. By employing the Reweight algorithm, we build the risk-consistent estimator as:
	\begin{equation}
		\bar{R}_{n, w}(\{{T}^{j}\}^q_{j=1}, {f})=\frac{1}{n} \sum_{i=1}^{n} \sum_{j=1}^{q} \frac{\mathbbm{\hat{P}}( {Y}^{j}=\bar{y}^{j}_{i} \mid {X}={x}_i)}{\mathbbm{\hat{P}}( \bar{Y}^{j}=\bar{y}^{j}_{i} \mid {X}={x}_i)} \ell\left(f_j\left({x}_{i}\right), \bar{y}^{j}_{i}\right),
		\label{reweiht}
	\end{equation}
	where $f_j({x})=\mb{I}[g_j({x})>0.5]$, and $\mb{I}[.]$ is the indicator function which takes 1 if the identity index is true and 0 otherwise; the subscript $w$ denotes that the loss function is weighted. 
	\section{Validation of Assumption~\ref{assum2}}
	\label{Validation}
	\begin{table*}[h]
		\label{frequencies}
		\caption{The frequencies of $Y^{i}=0$ given $Y^{50}=0/1$ on MS-COCO training dataset. The frequencies whose difference between given $Y^{50}=0$ and $Y^{50}=1$ is greater than 0.015 are in \textbf{bold}.}
		\centering
		\scriptsize
		\setlength\tabcolsep{3.8pt}
		\begin{tabular}{cccccccccc}
			\hline	
			$i=1$  &  $i=2$  &$i=3$  & $i=4$  & $i=5$   & $i=6$  & $i=7$   &$i=8$  & $i=9$  & $i=10$ \\		
			\hline	
			\textbf{0.960/0.983}&0.979/0.992 &\textbf{0.991/0.920}&0.974/0.985 &\textbf{0.999/0.961}&\textbf{0.999/0.959}&\textbf{0.984/0.999}&\textbf{0.955/0.980}&\textbf{0.974/0.936}&\textbf{0.989/0.958}\\	
			\hline	
			$i=11$  &  $i=12$  &$i=13$  & $i=14$  & $i=15$   & $i=16$  & $i=17$   &$i=18$  & $i=19$  & $i=20$ \\		
			\hline	
			\textbf{0.955/0.987}&0.980/0.970 &\textbf{0.943/0.964}&0.925/0.929 &\textbf{0.920/0.954}&\textbf{0.969/0.996}&\textbf{0.983/0.952}&0.975/0.975 &\textbf{0.931/0.866}&\textbf{0.975/0.995}\\	
			\hline
			$i=21$ & $i=22$  &$i=23$  &$i=24$  &$i=25$  &$i=26$  &$i=27$  &$i=28$  &$i=29$  &$i=30$\\		
			\hline	
			\textbf{0.937/0.989}&\textbf{0.982/0.941}&\textbf{0.912/0.874}&\textbf{0.950/0.971}&0.958/0.964 &\textbf{0.974/0.990}&0.918/0.923 &\textbf{0.886/0.908}&0.957/0.968 &0.987/0.987 \\	
			\hline
			$i=31$ & $i=32$  &$i=33$  &$i=34$  &$i=35$  &$i=36$  &$i=37$  &$i=38$  &$i=39$  &$i=40$\\		
			\hline	
			0.976/0.986 &0.981/0.989 &0.966/0.972 &\textbf{0.994/0.972}&\textbf{0.959/0.993}&0.998/0.999 &\textbf{0.989/0.902}&0.982/0.969 &0.990/0.990 &\textbf{0.972/0.991}\\	
			\hline
			$i=41$ & $i=42$  &$i=43$  &$i=44$  &$i=45$  &$i=46$  &$i=47$  &$i=48$  &$i=49$  &$i=50$\\		
			\hline	
			\textbf{0.997/0.967}&0.962/0.962 &0.966/0.973 &0.980/0.992 &\textbf{0.987/0.957}&\textbf{0.974/0.993}&\textbf{0.977/0.992}&\textbf{0.965/0.984}&0.993/0.995 &\textbf{1.000/0.000}\\	
			\hline
			$i=51$ & $i=52$  &$i=53$  &$i=54$  &$i=55$  &$i=56$  &$i=57$  &$i=58$  &$i=59$  &$i=60$\\		
			\hline	
			0.970/0.975 &\textbf{0.954/0.969}&0.972/0.986 &0.981/0.967 &0.974/0.985 &0.989/0.994 &\textbf{0.977/0.994}&\textbf{0.930/0.984}&\textbf{0.999/0.945}&\textbf{0.999/0.952}\\	
			\hline
			$i=61$ & $i=62$  &$i=63$  &$i=64$  &$i=65$  &$i=66$  &$i=67$  &$i=68$  &$i=69$  &$i=70$\\		
			\hline	
			\textbf{0.999/0.975}&0.965/0.973 &\textbf{0.997/0.936}&0.980/0.990 &0.984/0.977 &\textbf{0.998/0.950}&0.976/0.987 &\textbf{0.999/0.948}&\textbf{0.997/0.944}&0.997/0.999 \\	
			\hline
			$i=71$ & $i=72$  &$i=73$  &$i=74$  &$i=75$  &$i=76$  &$i=77$  &$i=78$  &$i=79$  &$i=80$\\		
			\hline	
			\textbf{0.945/0.994}&0.991/0.992 &0.969/0.961 &0.964/0.975 &\textbf{0.959/0.938}&\textbf{0.950/0.970}&\textbf{0.991/0.947}&\textbf{0.948/0.987}&0.984/0.974 &\textbf{0.966/0.998}\\	
			\hline
		\end{tabular}
		\label{validation_coco}
	\end{table*}
	
	In order to verify the assumption~\ref{assum2}, we count the frequencies of $Y^{i}=0$ given $Y^{50}=0/1$ on MS-COCO training dataset, \ie, $\mathbbm{\hat{P}}(Y^{i}=0|Y^{50}=0)/ \mathbbm{\hat{P}}(Y^{i}=0|Y^{50}=1)$. According to Hoeffding's inequality~\cite{Boucheron2013ConcentrationI}, when the frequencies whose difference between given $Y^{50}=0$ and $Y^{50}=1$ is greater than 0.015, we have at least 95\% confidence to make sure $\mathbbm{P}(Y^{i}=0|Y^{50}=0) \neq\mathbbm{P}(Y^{i}=0|Y^{50}=1)$. As shown in Tab.~\ref{validation_coco}, class 50 has a correlation with another 46 classes with a high probability, which means they are very likely to hold Assumption~\ref{assum2}, accounting for the majority of all 79 classes.
	
	In the implementation of our estimator, we always assume class $j$ and class $i$ have a strong correlation at first, and use them to estimate. While, when a reasonable solution cannot be obtained, we will abandon class $i$ and choose another class.
	
	\section{Discussion about Significantly Different Label Pairs in Real-world Datasets}
	\label{real-world}
	Although the real-world scenarios have complex label correlations, we claim significantly different class label pairs usually account for the majority of all label pairs in the typical multi-label datasets, \eg, MS-COCO~\cite{lin2014microsoft} and OpenImages~\cite{OpenImages2} datasets. It is because, among a large number of classes, most label pairs belong to significantly different superclasses. For example, in MS-COCO datasets, there are 80 classes, which belong to 10 significantly different superclasses (outdoor, food, indoor, appliance, sports, person, animal, vehicle, furniture, accessory, electronic, kitchen). In the OpenImages dataset, there are 19,957 class labels, and also have significantly different superclasses, such as Toy, Building, Medical equipment, Clothing, Insect, and so on. A part of these superclasses can be seen in~\cite{OpenImages_hyper}, which clearly shows most of the label pairs are significantly different and do not share the major discriminant features.

 	\section{Discussion about Relaxation of Instance-independent Assumption}
	\label{noise_assumption}
	In this work, we assume that label noise is class-dependent but instance-independent. While, in real-world scenarios, label noise is instance-dependent. Actually, this instance-independent assumption can be roughly related to the assumption that the label noise of one class label is dependent on the label correlations with a few classes, and independent on the label correlations with most classes, which means most label pairs $(i,j)$s meet $\mathbbm\mathbbm{P}\left(\bar{Y}^j \mid Y^j, Y^i\right)=\mathbbm\mathbbm{P}\left(\bar{Y}^j \mid Y^j\right), \mathbbm\mathbbm{P}\left(\bar{Y}^i \mid Y^j, Y^i\right)=\mathbbm\mathbbm{P}\left(\bar{Y}^i \mid Y^i\right)$. With such labels, the system of equations involving ${T}^{j}$ in Section~\ref{ours} holds, and our approach also works well. 
	This relaxed assumption can be nearly satisfied in many real-world scenes because generally speaking, the multi-label label noises for class $j$ are usually dependent on confusing features for itself, and the majority of classes will not share the same confusing features with class $j$. This claim agrees with the discussion about significantly different label pairs in Appendix~\ref{real-world}, and some research works about real-world label noise~\cite{Wei2022LearningWN,SongK019} also show that noisy labels usually flip to some similar class labels in the real-world scene. For example, In CIFAR-100N, which is a re-annotated version of the CIFAR-100 with real-world human annotations, most classes are more likely to be mislabeled into less than four fine classes~\cite{Wei2022LearningWN}. In ANIMAL-10N, the label noise mainly happens between five pairs of confusing animals~\cite{SongK019}. Besides, the experiments in Section~\ref{instance} also verify the effectiveness of our approach in two typical instance-dependent multi-label noise cases.
 
	\section{Proof of Theorem~\ref{theorem1}}
	\label{proof1}
	\begin{Lemma} $\bar{Y}^{i}$ and $\bar{Y}^{j}$  are independent given $Y^{j}$.
		\label{lemma1}
	\end{Lemma}	
	\begin{proof}
		Since we assume that the transition matrix is class-dependent and instance-independent, $\bar{Y}^{j}$ and any variable are independent given $Y^{j}$. Therefore, this lemma holds.
	\end{proof}
	
	\begin{Lemma} The product of two row-stochastic matrices is still a row-stochastic one.
		\label{lemma1.5}
	\end{Lemma}	
	\begin{proof}
		Let ${P}$ and ${Q}$ be row-stochastic matrices of the following form:	${P}=\begin{pmatrix} 1-p_{-} & p_{-} \\ p_{+} & 1-p_{+} \end{pmatrix}$ and	${Q}=\begin{pmatrix} 1-q_{-} & q_{-} \\ q_{+} & 1-q_{+} \end{pmatrix}$. Then the product of ${P}$ and ${Q}$ is ${P Q}=\begin{pmatrix} 1-q_{-}+p_{-}(-1+q_{-}+q_{+})  & q_{-}-p_{-}(-1+q_{-}+q_{+}) \\ q_{+}-p_{+}(-1+q_{-}+q_{+}) & 1-q_{+}+p_{+}(-1+q_{-}+q_{+}) \end{pmatrix}$. It can be readily verified that the sum of each row of ${P Q}$ is equal to $1$, meaning that ${P Q}$ is a row-stochastic matrix.
	\end{proof}
	
	\begin{thmbis}{theorem1} Two noisy labels $\{\bar{Y}^{j},\bar{Y}^{i}\}$ will not suffice to identify ${T}^{j}$.
	\end{thmbis}
	\begin{proof}
		First, the information from $\bar{Y}^{j},\bar{Y}^{i}$ can be fully captured by the following four quantities: $\mathbbm{P}\left(\bar{Y}^{j}=0,  \bar{Y}^{i}=0\right)$, $\mathbbm{P}\left(\bar{Y}^{j}=1,  \bar{Y}^{i}=0\right)$, $\mathbbm{P}\left(\bar{Y}^{j}=0,  \bar{Y}^{i}=1\right)$, and $\mathbbm{P}\left(\bar{Y}^{j}=1,  \bar{Y}^{i}=1\right)$.
		According to Lemma~\ref{lemma1}, these four quantities can lead to four equations that depend on ${T}^{j}$:
		\begin{normalsize} 
			\begin{equation*}
				\begin{aligned}
					&\mathbbm{P}\left(\bar{Y}^{j}=0,  \bar{Y}^{i}=0\right)=\mathbbm{P}({Y}^{j}=0)T^{j}_{00}\mathbbm{P}(\bar{Y}^{i}=0|  {Y}^{j}=0)+\mathbbm{P}({Y}^{j}=1)T^{j}_{10}\mathbbm{P}(\bar{Y}^{i}=0|  {Y}^{j}=1)\\
					&\mathbbm{P}\left(\bar{Y}^{j}=0,  \bar{Y}^{i}=1\right)=\mathbbm{P}({Y}^{j}=0)T^{j}_{00}\mathbbm{P}(\bar{Y}^{i}=1|  {Y}^{j}=0)+\mathbbm{P}({Y}^{j}=1)T^{j}_{10}\mathbbm{P}(\bar{Y}^{i}=1|  {Y}^{j}=1)\\
					&\mathbbm{P}\left(\bar{Y}^{j}=1,  \bar{Y}^{i}=0\right)=\mathbbm{P}({Y}^{j}=0)T^{j}_{01}\mathbbm{P}(\bar{Y}^{i}=0|  {Y}^{j}=0)+\mathbbm{P}({Y}^{j}=1)T^{j}_{11}\mathbbm{P}(\bar{Y}^{i}=0|  {Y}^{j}=1)\\
					&\mathbbm{P}\left(\bar{Y}^{j}=1,  \bar{Y}^{i}=1\right)=\mathbbm{P}({Y}^{j}=0)T^{j}_{01}\mathbbm{P}(\bar{Y}^{i}=1|  {Y}^{j}=0)+\mathbbm{P}({Y}^{j}=1)T^{j}_{11}\mathbbm{P}(\bar{Y}^{i}=1|  {Y}^{j}=1).
				\end{aligned}
			\end{equation*}
		\end{normalsize}
		For simplicity, we denote
		\begin{equation*}
			\begin{aligned}
				&{E}=\begin{pmatrix}\mathbbm{P}(\bar{Y}^{j}=0,  \bar{Y}^{i}=0) & \mathbbm{P}(\bar{Y}^{j}=0,  \bar{Y}^{i}=1)\\ \mathbbm{P}(\bar{Y}^{j}=1,  \bar{Y}^{i}=0) & \mathbbm{P}(\bar{Y}^{j}=1,  \bar{Y}^{i}=1) \end{pmatrix}=\begin{pmatrix} e_{00} & e_{01} \\ e_{10} & e_{11} \end{pmatrix},
				\\& {P} =\begin{pmatrix}\mathbbm{P}({Y}^{j}=0) & 0 \\ 0 & \mathbbm{P}({Y}^{j}=1) \end{pmatrix}=\begin{pmatrix}1-p & 0 \\ 0 & p \end{pmatrix},
				\\& {T}^{j}=\begin{pmatrix}\mathbbm{P}(\bar{Y}^{j}=0 \mid {Y}^{j}=0) & \mathbbm{P}(\bar{Y}^{j}=1 \mid {Y}^{j}=0)\\ \mathbbm{P}(\bar{Y}^{j}=0\mid {Y}^{j}=1) & \mathbbm{P}(\bar{Y}^{j}=1 \mid {Y}^{j}=1) \end{pmatrix}=\begin{pmatrix} 1-\rho_{-} & \rho_{-} \\ \rho_{+} & 1-\rho_{+} \end{pmatrix}, \text{and}
				\\&	{M}=\begin{pmatrix}\mathbbm{P}(\bar{Y}^{i}=0 \mid {Y}^{j}=0) & \mathbbm{P}(\bar{Y}^{i}=1 \mid {Y}^{j}=0)\\ \mathbbm{P}(\bar{Y}^{i}=0\mid {Y}^{j}=1) & \mathbbm{P}(\bar{Y}^{i}=1 \mid {Y}^{j}=1) \end{pmatrix}=\begin{pmatrix} 1-\rho^{\prime}_{-} & \rho^{\prime}_{-} \\ \rho^{\prime}_{+} & 1-\rho^{\prime}_{+} \end{pmatrix}.
			\end{aligned}
		\end{equation*}
		Then, the system of equations can be expressed as ${E} =  ({T}^{j})^\top {P} {M}$, \ie,
		\begin{equation}
			\begin{pmatrix} e_{00} & e_{01} \\ e_{10} & e_{11} \end{pmatrix}=
			{\begin{pmatrix} 1-\rho_{-} & \rho_{-} \\ \rho_{+} & 1-\rho_{+} \end{pmatrix}}^\top
			\begin{pmatrix}1-p & 0 \\ 0 & p \end{pmatrix}
			{\begin{pmatrix} 1-\rho^{\prime}_{-} & \rho^{\prime}_{-} \\ \rho^{\prime}_{+} & 1-\rho^{\prime}_{+} \end{pmatrix}}.
			\label{system_proof}
		\end{equation}
		
		Assuming $\{{T}^{0},{P}^0, {M}^0\}$ satisfies 
		\begin{equation}
			{E} =  ({T}^{0})^\top {P}^0 {M}^0,
			\label{solution}
		\end{equation}
		Next, we will prove that by selecting proper parameters, a different solution of Eq.~(\ref{system_proof}) can be derived, which ruins the identifiability of ${T}^{j}$.
		Let ${A}=\begin{pmatrix} 1-a_{-} & a_{-} \\ a_{+} & 1-a_{+} \end{pmatrix}$ and ${B}=\begin{pmatrix} 1-b_{-} & b_{-} \\ b_{+} & 1-b_{+} \end{pmatrix}$ be invertible, row-stochastic matrices. Based on the invertibility of ${A}$ and ${B}$, Eq.~(\ref{solution}) can be rewritten as ${E} =  ({A}{T}^{0})^\top ({A}^\top)^{-1} {P}^0 {B}^{-1}{B} {M}^0$. According to Lemma~\ref{lemma1.5}, ${T}^1={A}{T}^0$ and ${M}^1={B}{M}^0$ are row-stochastic matrices, which is consistent with the form of ${T}^j$ and ${M}$.
		
		Last, denoting ${P}^{0}=\begin{pmatrix}1-p_0 & 0 \\ 0 & p_0 \end{pmatrix}$ and ${P}^1=\begin{pmatrix} p_{00} & p_{01} \\ p_{10} & p_{11} \end{pmatrix}$, by letting ${P}^1=({A}^\top)^{-1} {P}^{0}{B}^{-1}=	\left[{\begin{pmatrix} 1-a_{-} & a_{-} \\ a_{+} & 1-a_{+} \end{pmatrix}}^\top\right]^{-1} \begin{pmatrix}1-p_0 & 0 \\ 0 & p_0 \end{pmatrix} {\begin{pmatrix} 1-b_{-} & b_{-} \\ b_{+} & 1-b_{+} \end{pmatrix}}={\begin{pmatrix} p_{00} & p_{01} \\ p_{10} & p_{11} \end{pmatrix}}$ be in the form of ${P}$, \ie, solving the following equations:
		\begin{equation*}
			\begin{cases}
				p_{00}+p_{11}=1,\\
				p_{01}=0,\\
				p_{10}=0,
			\end{cases}
		\end{equation*}
		we can get $a_{+}=\frac{b_{-}(1-p_0)}{b_{-}-p_0}$ and $b_{+}=\frac{a_{-}(1-p_0)}{a_{-}-p_0}$. It means that when we have a solution $\{{T}^0,{P}^0, {M}^0\}$ of Eq.~(\ref{system_proof}), we can get another different solution $\{{T}^1,{P}^1, {M}^1\}$ by setting appropriate values of $b_{-}$ and $a_{-}$. Hence, ${T}^{j}$ is unidentifiable in this situation. 
		
	\end{proof}
	
	\section{Proof of Theorem~\ref{theorem2}}
	\label{proof2}
	To make the proof clear, following~\cite{abs-2202-02016}, we reproduce the Kruskal’s identifiability result here. The setup of Kruskal’s identifiability result is as follows: suppose
	that there is an unobserved variable $Z$ that takes values in
	$\{0,1,...,K-1\}$. $Z$ has a non-degenerate prior $\mathbbm{P}(Z = i) >
	0$. Instead of observing $Z$, we observe a set of  conditionally independent variables
	$\left\{{O}^{(t)}\right\}_{t=1}^{N}$. Each ${O}^{(t)}$ has a finite state space with cardinality $\kappa_{t}$. Let ${M}^{(t)}$ be a matrix of size $K \times \kappa_{t}$, which $j$-th row is simply
	$\left[\mathbbm{P}\left({O}^{(t)}=1 \mid Z=j\right), \ldots, \mathbbm{P}\left({O}^{(t)}=\kappa_{t} \mid Z=j\right)\right]$.
	The previous works~\cite{kruskal1977three,sidiropoulos2000uniqueness} have proved the following Theorem~\ref{Kruskal_result}.
	
	\begin{Definition} [Kruskal rank~\cite{abs-2202-02016}] For a
		matrix ${M}$, the Kruskal rank of ${M}$ is the largest number $I$ such that every set of $I$ rows of ${M}$ are independent. The symbol is $\operatorname{Kr}({M})=I$.
		\label{define2}
	\end{Definition}
	\begin{Theorem}[Kruskal’s identifiability result~\cite{kruskal1977three,sidiropoulos2000uniqueness}] The model parameters are uniquely identifiable, up to label permutation, if $$\sum_{t=1}^{N} \operatorname{Kr}\left({M}^{(t)}\right) \geq 2 K+N-1.$$
		\label{Kruskal_result}
	\end{Theorem}
	
	We can prove Theorem~\ref{theorem2} with the following lemmas:
	\begin{Lemma} $\bar{Y}^{i}$ and ${Y}^{j}$  are independent given $Y^{i}$.
		\begin{proof}
			Since we assume that the transition matrix is class-dependent and instance-independent, $\bar{Y}^{i}$ and any variable are independent given $Y^{i}$. Therefore, this lemma holds.
			\label{lemma2_1}
		\end{proof}
	\end{Lemma}
	\begin{Lemma} If $Y^{j} \in\{0,1\}$ corresponds to $Z$, $\{\bar{Y}^{j},\bar{Y}^{i}, \bar{Y}^{k}\}$ correspond to the observations $\left\{{O}^{(t)}\right\}_{t=1}^{3}$, then $\operatorname{Kr}\left({M}^{(t)}\right)=2,t\in[3]$.
		\label{lemma2}
	\end{Lemma}
	\begin{proof}
		As ${M}^{(1)}={T}^{j}$ is a row-stochastic matrix,  according to Assumption~\ref{assum1}, every set of 2 rows of it are independent, thus its Kruskal rank is 2. Therefore, according to Definition~\ref{define2}, $\operatorname{Kr}\left({T}^{j}\right)=2$.
		
		The $(p, q)$ entry of ${M}^{(2)}$ is 
		\begin{equation}
			\begin{aligned}
				&{M}^{(2)}_{p q} = \mathbbm{P}(\bar{Y}^{i}=q \mid Y^{j}=p)=\sum_{c=0}^1 \mathbbm{P}(\bar{Y}^{i}=q, Y^{i}=c \mid Y^{j}=p)\\
				&= \sum_{c=0}^1 \mathbbm{P}(\bar{Y}^{i}=q \mid Y^{i}=c, Y^{j}=p) \mathbbm{P}(Y^{i}=c \mid Y^{j}=p) \\
				&= \sum_{c=0}^1 \mathbbm{P}(\bar{Y}^{i}=q \mid Y^{i}=c) \mathbbm{P}(Y^{i}=c \mid Y^{j}=p),
			\end{aligned}
			\label{m2-derivation}
		\end{equation}
		where the last equation holds because of Lemma.~\ref{lemma2_1}.
		
		Denote ${M}^{\prime(2)}= \begin{pmatrix} \mathbbm{P}(Y^{i}=0 \mid Y^{j}=0) & \mathbbm{P}(Y^{i}=1 \mid Y^{j}=0) \\ \mathbbm{P}(Y^{i}=0 \mid Y^{j}=1) & \mathbbm{P}(Y^{i}=1 \mid Y^{j}=1) \end{pmatrix}$. Then Eq.~(\ref{m2-derivation}) can be rewritten as the following matrix form:
		\begin{equation}
			{M}^{(2)} = {M}^{\prime(2)}{T}^{i},
		\end{equation}
		where ${T}^{i}$ denotes the transition matrix of class $i$. According to Assumption~\ref{assum1}, $\operatorname{Kr}\left({T}^{i}\right)=2$.
		As ${M}^{\prime(2)}$ is a row-stochastic matrix,  according to Assumption~\ref{assum2}, every set of 2 rows of it are independent. Hence, according to Definition~\ref{define2}, $\operatorname{Kr}\left({M}^{\prime(2)}\right)=2$.
		
		Since ${T}^{i}$ and ${M}^{\prime(2)}$ are full-rank matrices, based on Eq.~(\ref{m2-derivation}), we can get the Kruskal rank of ${M}^{(2)}$ as $\operatorname{Kr}\left({M}^{(2)}\right)=2$. Similarly, $\operatorname{Kr}\left({M}^{(3)}\right)=2$.
	\end{proof}
	\begin{thmbis}{theorem2} If $\bar{Y}^{i}$ and $\bar{Y}^{k}$ are independent given $Y^{j}$, three noisy labels $\{\bar{Y}^{j},\bar{Y}^{i}, \bar{Y}^{k}\}$ are sufficient to identify ${T}^{j}$.
	\end{thmbis}
	\begin{proof}
		According to Lemma~\ref{lemma1} and that $\bar{Y}^{i}$ and $\bar{Y}^{k}$ are independent given $Y^{j}$, we can relate our multi-label noise setting to the setup of Kruscal’s identifiability scenario: $Y^{j} \in\{0,1\}$ corresponds to the unobserved hidden variable $Z$; $\mathbbm{P}(Y^{j} = i)$ corresponds to the prior of this hidden variable; 
		Noisy labels $\{\bar{Y}^{j},\bar{Y}^{i}, \bar{Y}^{k}\}$ correspond to the
		observations $\left\{{O}^{(t)}\right\}_{t=1}^{3}$. $\kappa_{t}$ is then simply the cardinality of the noisy label space, \ie, $\kappa_{t}=K=2$; Each ${O}^{(t)}$ has a corresponding observation matrix ${M}^{(t)}$, and  ${M}^{(t)}_{vk}=\mathbbm{P}\left({O}^{(t)}=k \mid Y^{j}=v\right)$.
		Now we can get the following result about identifiability of ${T}$.

		According to Lemma~\ref{lemma2}, the Kruskal ranks satisfy 
		$$\sum_{t=1}^{3} \operatorname{Kr}\left({M}^{(t)}\right) = 3K = 2 K +2 \geq 2 K+N-1, \text{ when } N=3.$$
		Calling Theorem~\ref{Kruskal_result} proves the uniqueness of ${M}^{(t)}$. As ${M}^{(1)}={T}^{j}$, then ${T}^{j}$ is identifiable.
	\end{proof}
	
	\section{Proof of Theorem~\ref{theorem3}}
	\label{proof3}
	
	\begin{Lemma} $\bar{Y}^{i}, \bar{Y}^{j}$ and $\bar{Y}^{k}$ are independent given $Y^{j}$ and $Y^{i}$.
		\label{lemma3}
	\end{Lemma}	
	\begin{proof}
		Since we assume that the transition matrix is class-dependent and instance-independent, $\bar{Y}^{j}$ and any variable are independent given $Y^{j}$, and $\bar{Y}^{i}$ and any variable are independent given $Y^{i}$. Therefore, this lemma holds.
	\end{proof}
	\begin{Lemma} If ${M}_{4\times 2}=\mathbbm{P}(\bar{Y}^i \mid Y^j,Y^i)$, the first two rows and last two rows of ${M}$ are identical respectively, \ie, $M_{0p} = M_{1p}$ and $M_{2p} = M_{3p}, p=0,1.$
		\label{lemma4}
	\end{Lemma}
	\begin{proof} 
		\begin{equation*}
			\begin{aligned}
				{M} 
				& = \begin{pmatrix} \mathbbm{P}(\bar{Y}^{i}=0 \mid Y^{j}=0, Y^{i}=0) & \mathbbm{P}(\bar{Y}^{i}=1 \mid Y^{j}=0,Y^{i}=0) \\ \mathbbm{P}(\bar{Y}^{i}=0 \mid Y^{j}=1,Y^{i}=0 ) & \mathbbm{P}(\bar{Y}^{i}=1 \mid Y^{j}=1,Y^{i}=0) \\ \mathbbm{P}(\bar{Y}^{i}=0 \mid Y^{j}=0, Y^{i}=1) & \mathbbm{P}(\bar{Y}^{i}=1 \mid Y^{j}=0,Y^{i}=1) \\ \mathbbm{P}(\bar{Y}^{i}=0 \mid Y^{j}=1,Y^{i}=1 ) & \mathbbm{P}(\bar{Y}^{i}=1 \mid Y^{j}=1,Y^{i}=1) \end{pmatrix}\\
				&=\begin{pmatrix} \mathbbm{P}(\bar{Y}^{i}=0 \mid Y^{i}=0) & \mathbbm{P}(\bar{Y}^{i}=1 \mid Y^{i}=0) \\ \mathbbm{P}(\bar{Y}^{i}=0 \mid Y^{i}=0) & \mathbbm{P}(\bar{Y}^{i}=1 \mid Y^{i}=0) \\ \mathbbm{P}(\bar{Y}^{i}=0 \mid Y^{i}=1) & \mathbbm{P}(\bar{Y}^{i}=1 \mid Y^{j}=1) \\ \mathbbm{P}(\bar{Y}^{i}=0 \mid Y^{j}=1) & \mathbbm{P}(\bar{Y}^{i}=1 \mid Y^{j}=1) \end{pmatrix},
			\end{aligned}
		\end{equation*}
		where the second equation holds because $\bar{Y}^{i}$ is only dependent on ${Y}^{i}$. Accordingly, $M_{0p} = M_{1p}$ and $M_{2p} = M_{3p}, p=0,1$.
	\end{proof}
	\begin{thmbis}{theorem3} If $\bar{Y}^{i}$ and $\bar{Y}^{k}$ are not independent given $Y^{j}$, three noisy labels $\{\bar{Y}^{j},\bar{Y}^{i}, \bar{Y}^{k}\}$ will not suffice to identify ${T}^{j}$.
	\end{thmbis}
	\begin{proof}
		The likelihood of noisy labels can be formulated as a third-order tensor ${L}$, where the $(p,q,m)$ entry is 
		\begin{equation}
			{L}_{pqm} = \mathbbm{P}(\bar{Y}^i=p, \bar{Y}^j=q, \bar{Y}^k=m).
			\label{the3proof-0}
		\end{equation}
		According to Lemma~\ref{lemma3}, Eq.~(\ref{the3proof-0}) can be expanded by the total probability rule:
		\begin{equation}
			\begin{aligned}
				&{L}_{pqm} = \sum_{c_0,c_1=0}^{c_0,c_1=1} \mathbbm{P}(\bar{Y}^i=p, \bar{Y}^j=q, \bar{Y}^k=m \mid Y^j=c_0, Y^i=c_1) \mathbbm{P}(Y^j=c_0, Y^i=c_1)\\
				&=\sum_{c_0,c_1=0}^{c_0,c_1=1} \mathbbm{P}(\bar{Y}^i=p \mid Y^j,Y^i)\mathbbm{P}(\bar{Y}^j=q \mid Y^j,Y^i)\mathbbm{P}(\bar{Y}^k=m \mid Y^j,Y^i)\mathbbm{P}(Y^j,Y^i).
			\end{aligned}
			\label{the3proof-1}
		\end{equation}
		where we omit the value of $Y^i,Y^j$ in the last equation for simplicity. Let ${M}^{(1)}_{4\times 2}=\mathbbm{P}(\bar{Y}^j\mid Y^i,Y^j)$, ${M}^{(2)}_{4\times 2}=\mathbbm{P}(\bar{Y}^i\mid Y^i,Y^j)$, ${M}^{(3)}_{4\times 2}=\mathbbm{P}(\bar{Y}^k\mid Y^i,sY^j)$ and ${\Lambda}_{4\times1}=\mathbbm{P}(Y^i,Y^j)$, and the Eq.~(\ref{the3proof-1}) can be expressed as ${L}_{pqm} = \sum_{v=0}^{3} M^{(2)}_{vp} M^{(1)}_{vq} M^{(3)}_{vm} \Lambda_v$. 
		
		Let $\{ {A}^0, {B}^0$, ${C}^0, {\Lambda}^0 \}$ be a solution of $\{ {M}^{(1)}, {M}^{(2)}, {M}^{(3)}, {\Lambda}\}$, which means it fulfils the likelihood equations (Eq.~(\ref{the3proof-1})), \ie,
		\begin{equation}
			{L}_{pqm} = \sum_{v=0}^{3} B^0_{vp} A^0_{vq} C^0_{vm} \Lambda^0_v .
			\label{the3proof-2}
		\end{equation}
		Note that as stated in Lemma.~\ref{lemma4}, the first two rows and last two rows of ${B}^0$ are identical respectively. Next, we will show that a different solution can be constructed by simply switching the corresponding rows in ${A}^0$, ${C}^0$ and ${\Lambda}^0$, which is consistent with the result in \cite{ten2002uniqueness} when $\operatorname{Kr}({M}^{(2)})=1$. By letting 
		$${A}^1=\begin{pmatrix} A^0_{10} & A^0_{11}\\ A^0_{00} &  A^0_{01} \\ A^0_{30} & A^0_{31}\\ A^0_{20} &  A^0_{21} \end{pmatrix}, \quad {C}^1=\begin{pmatrix} C^0_{10} & C^0_{11}\\ C^0_{00} &  C^0_{01} \\ C^0_{30} & C^0_{31}\\ C^0_{20} &  C^0_{21} \end{pmatrix}, \quad \text{and} \quad {\Lambda}^1=\begin{pmatrix} \Lambda^0_1\\\Lambda^0_0 \\ \Lambda^0_3\\ \Lambda^0_2 \end{pmatrix}, $$ 
		then the Eq.~(\ref{the3proof-2}) is equivalent to
		\begin{equation}
			\begin{aligned}
				{L}_{pqm} &= B^0_{0p} A^0_{0q} C^0_{0m} \Lambda^0_0 + B^0_{1p} A^0_{1q} C^0_{1m} \Lambda^0_1 +  B^0_{2p} A^0_{2q} C^0_{2m} \Lambda^0_2 + B^0_{3p} A^0_{3q} C^0_{3m} \Lambda^0_3\\
				&= \textcolor{blue}{B^0_{1p}} A^0_{0q} C^0_{0m} \Lambda^0_0 + \textcolor{blue}{B^0_{0p}} A^0_{1q} C^0_{1m} \Lambda^0_1 + \textcolor{blue}{B^0_{3p}} A^0_{2q} C^0_{2m} \Lambda^0_2 + \textcolor{blue}{B^0_{2p}} A^0_{3q} C^0_{3m} \Lambda^0_3 \\
				&= B^0_{1p} \textcolor{blue}{A^1_{1q} C^1_{1m} \Lambda^1_1} + B^0_{0p} \textcolor{blue}{A^1_{0q} C^1_{0m} \Lambda^1_0} +B^0_{3p} \textcolor{blue}{A^1_{3q} C^1_{3m} \Lambda^1_3}  + B^0_{2p} \textcolor{blue}{A^1_{2q} C^1_{2m} \Lambda^1_2}\\
				&=\sum_{v=0}^{3} B^0_{vp} \textcolor{blue}{A^1_{vq} C^1_{vm} \Lambda^1_v} 
			\end{aligned}
			\label{the3proof-3}
		\end{equation}
		where the second equation holds because the first two rows and last two rows of ${B}^0$ are identical respectively, \ie, $B^0_{0p} = B^0_{1p}$ and $B^0_{2p} = B^0_{3p}$. Note that ${A}^1, {C}^1$ and ${\Lambda}^1$ are consistent with the form of ${M}^{(1)}, {M}^{(3)}$ and ${\Lambda}$ respectively. Then, according to Eq.~(\ref{the3proof-3}), it can be readily observed that the $\{ {A}^1, {B}^0$, ${C}^1, {\Lambda}^1 \}$ is a new solution of $\{ {M}^{(1)}, {M}^{(2)}, {M}^{(3)}, {\Lambda}\}$, hence the uniqueness is not guaranteed under this circumstance.  As ${M}^{(1)}=\begin{pmatrix} {T}^{j} \\ {T}^{j} \end{pmatrix}$ is not unique, then ${T}^{j}$ is unidentifiable.
	\end{proof}
	
	Note that according to Definition~\ref{define2_}, in the situation of Theorem~\ref{theorem3}, the model parameter $\theta:=\{{T}^j,\mathbbm{P}(Y^j),\mathbbm{P}(\bar{Y}^{i},\bar{Y}^{k}|{Y}^j)\}$. For convenience, in the proof of Theorem~\ref{theorem3}, we use $\theta^1 := \{{T}^j,\mathbbm{P}(Y^j,Y^i),\mathbbm{P}(\bar{Y}^{i}|Y^j,Y^i),\mathbbm{P}(\bar{Y}^{k}|Y^j,Y^i)\}$ as the model parameter. It is easy to prove that the above different solutions of $\theta^1$ can lead to two different solutions of $\theta$.
	
	\section{Proof of Theorem~\ref{theorem4}}
	\label{proof4}
	\begin{thmbis}{theorem4} If $\mathbbm{P}(\bar{Y}^{i}\mid {Y}^{j})$  is known, two noisy labels $\{\bar{Y}^{j},\bar{Y}^{i}\}$ are sufficient to identify ${T}^{j}$.
	\end{thmbis}
	\begin{proof}
		The proof of Theorem~\ref{theorem4} is much similar to that of Theorem~\ref{theorem1}, we can get a system of equations expressed as ${E} =  ({T}^{j})^\top {P} {M}$. The difference lies in that in Theorem~\ref{theorem4}, the matrix ${M}$ which is parameterized by $\mathbbm{P}(\bar{Y}^{i}\mid {Y}^{j})$ is given. Since ${M}$ is invertible (Similar to Lemma.~\ref{lemma2}), the problem can be converted to a simple bilinear decomposition problem: $${E} ({M})^{-1} = ({T}^{j})^\top {P},$$ \ie,
		\begin{equation}
			\begin{pmatrix} e_{00} & e_{01} \\ e_{10} & e_{11} \end{pmatrix} 	{\begin{pmatrix} 1-\rho^{\prime}_{-} & \rho^{\prime}_{-} \\ \rho^{\prime}_{+} & 1-\rho^{\prime}_{+} \end{pmatrix}}^{-1}=
			{\begin{pmatrix} 1-\rho_{-} & \rho_{-} \\ \rho_{+} & 1-\rho_{+} \end{pmatrix}}^\top
			\begin{pmatrix}1-p & 0 \\ 0 & p \end{pmatrix}.
			\label{theor4-1}
		\end{equation}
		Solving the above bilinear decomposition problem, the unique solution can be obtained as:
		\begin{equation}
			p= \frac{(1-\rho^{\prime}_{-})-(e_{00}+e_{10})}{1-\rho^{\prime}_{-}-\rho^{\prime}_{+}}.
			\label{theor4-2}
		\end{equation}
		
		Substituting Eq.~(\ref{theor4-2}) into Eq.~(\ref{theor4-1}), then right multiplying ${P}^{-1}$ on both side of the equation, the matrix ${T}^{j}$ can be derived as:
		$${T}^{j}=[{E} ({M})^{-1}({P})^{-1}]^{\top},$$
		which indicates that ${T}^{j}$ is identifiable given label correlation $\mathbbm{P}(\bar{Y}^{i}\mid {Y}^{j})$.
	\end{proof}
\section{Summary of the Inspirations from the proof of Theorem~\ref{theorem1}-\ref{theorem4}}
	\label{inspiration}
	From the proof of Theorem~\ref{theorem1}, we can know that the label correlations of two noisy labels $\{\bar{Y}^j, \bar{Y}^{i} \}$ can not offer enough information to achieve the identifiability of ${T}^j$. 
	
	From the proof of Theorem~\ref{theorem2}, three noisy labels $\{\bar{Y}^j, \bar{Y}^{i}, \bar{Y}^{k}\}$ can provide more information than two noisy labels $\{\bar{Y}^j, \bar{Y}^{i}, \bar{Y}^{k}\}$ to achieve the identifiability when $\bar{Y}^{i}$ and $\bar{Y}^{k}$ are independent given $Y^{j}$. Note that when satisfying the condition, the parameter $\mathbbm{P}(\bar{Y}^{i},\bar{Y}^{k}|{Y}^j)$  is reduced to $\mathbbm{P}(\bar{Y}^{i}|{Y}^j)$ and $\mathbbm{P}(\bar{Y}^{k}|{Y}^j)$.
	
	From the proof of Theorem~\ref{theorem3}, due to the entangled correlations, when $\bar{Y}^{i}$ and $\bar{Y}^{k}$ are not independent given $Y^{j}$, modelling three noisy labels $\{\bar{Y}^j, \bar{Y}^{i}, \bar{Y}^{k}\}$ will increase too many model parameters, making the identifiability decrease from the situation in Theorem~\ref{theorem2}.  
	
	From the proof of Theorem~\ref{theorem4}, by reducing the unknown model parameters via some extra information, the label correlations of two noisy labels $\{\bar{Y}^j, \bar{Y}^{i} \}$ can achieve the identifiability.
	
	The summary of $\Omega$, $\theta$, $\mathbbm{P}_\theta$, conditions, and the identifiability in Theorem~\ref{theorem1}-\ref{theorem4} is shown in Tab.~\ref{diff}.
	
	\begin{table*}[h]
		\caption{Summary of $\Omega$, $\theta$, $\mathbbm{P}_\theta$, condition, and the identifiability in Theorem~\ref{theorem1}-\ref{theorem4} .}
		\centering
	\setlength\tabcolsep{4.8pt}
		\begin{tabular}{l|c|c|c|c|c}
			\hline	
			& $\Omega$ & $\theta$ & $\mathbbm{P}_\theta$ & Condition & Identifiability of ${T}^j$ \\		
			\hline	
			Theorem~\ref{theorem1}  &
			$ \bar{Y}^j, \bar{Y}^{i}  $ & $ {T}^j,\mathbbm{P}(Y^j),\mathbbm{P}(\bar{Y}^{i}|{Y}^j) $ & $\mathbbm{P}(\bar{Y}^j, \bar{Y}^{i} )$ & -- &  \\		
			Theorem~\ref{theorem2} & $ \bar{Y}^j, \bar{Y}^{i}, \bar{Y}^{k} $ & $ {T}^j,\mathbbm{P}(Y^j),\mathbbm{P}(\bar{Y}^{i}|{Y}^j),\mathbbm{P}(\bar{Y}^{k}|{Y}^j) $	 & $\mathbbm{P}(\bar{Y}^j, \bar{Y}^{i}, \bar{Y}^{k})$ &  $\bar{Y}^{i}$ and $\bar{Y}^{k}$ are independent given $Y^{j}$ & \checkmark\\	
			Theorem~\ref{theorem3}  & $ \bar{Y}^j, \bar{Y}^{i}, \bar{Y}^{k} $ & 	$ {T}^j,\mathbbm{P}(Y^j),\mathbbm{P}(\bar{Y}^{i},\bar{Y}^{k}|{Y}^j)$ & $\mathbbm{P}(\bar{Y}^j, \bar{Y}^{i}, \bar{Y}^{k})$ & -- & \\		
			Theorem~\ref{theorem4} & $ \bar{Y}^j, \bar{Y}^{i}  $& $ {T}^j,\mathbbm{P}(Y^j) $	 & $\mathbbm{P}(\bar{Y}^j, \bar{Y}^{i} )$ & $\mathbbm{P}(\bar{Y}^{i}\mid {Y}^{j})$  is known  & \checkmark\\	
			\hline	
		\end{tabular}
		\label{diff}
	\end{table*}

\section{Proof of Theorem~\ref{thm:main2}}
\label{proof_6}
We have 
\begin{equation}
\begin{aligned}
  \|{T}^j-{\hat{T}}^j\|_{1} & \leq \|{T}^j-{\bar{T}}^j\|_{1} \\
  &\leq \|{T}^j-\mathbb{E}[{\bar{T}}^j]+\mathbb{E}[{\bar{T}}^j]-{\bar{T}}^j\|_{1}   \\
  &\leq \|{T}^j-\mathbb{E}[{\bar{T}}^j]\|_{1} + \|\mathbb{E}[{\bar{T}}^j]-{\bar{T}}^j\|_{1}.
\end{aligned}
\end{equation}
As ${\bar{T}}^j =\arg \min_{\bar{{T}}^{j}} \sum_{r=1}^R \Vert \bar{{T}}^{j} - \hat{{T}}^{j}_r \Vert_1 = \sum_{r=1}^R \frac{\hat{{T}}^{j}_r}{R}$, we first consider the $r$-{th} estimation $$\hat{{T}}^j_r=\begin{pmatrix} 1-\hat{\rho}^r_{-} & \hat{\rho}_{-}^r \\ \hat{\rho}_{+}^r& 1-\hat{\rho}_{+}^r\end{pmatrix}, r\in \{1,\ldots,R\}.$$

\begin{Lemma}\label{lemma:4}
When the training sample size $n$ is large,
\begin{align}
\|{T}^j-\mathbb{E}[{\hat{T}}^j_r]\|_{1} \leq 8C\Delta,\\
\operatorname{Var}[\hat{\rho}^r_+] \leq C^2(\frac{3}{4n}+\frac{1}{2\lambda_1 n}), \\
\operatorname{Var}[\hat{\rho}^r_-] \leq C^2(\frac{3}{4n}+\frac{1}{2\lambda_1 n}),
\end{align}
where $C$ is the maximum absolute value of the first derivative of $\hat{\rho}^r_+$ and $\hat{\rho}^r_-$ w.r.t. the estimated co-occurrence probabilities.
\end{Lemma}
The proof of Lemma \ref{lemma:4} is provided in Appendix \ref{sec:app4}.

After that, according to Lemma \ref{lemma:4}, we can bound $\|{T}^j-\mathbb{E}[{\bar{T}}^j]\|_{1}$ by
\begin{equation}
\begin{aligned}
\|{T}^j-\mathbb{E}[{\bar{T}}^j]\|_{1}&=\|{T}^j-\mathbb{E}\left[\sum_{r=1}^R \frac{\hat{{T}}^{j}_r}{R}\right]  \|_{1} =\|{T}^j-\sum_{r=1}^R \frac{1}{R}\mathbb{E}\left[ \hat{{T}}^{j}_r\right]  \|_{1} \\
&= \|\sum_{r=1}^R \frac{1}{R}{T}^j-\sum_{r=1}^R \frac{1}{R}\mathbb{E}\left[ \hat{{T}}^{j}_r\right]  \|_{1}\leq \sum_{r=1}^R \frac{1}{R} \|{T}^j-\mathbb{E}\left[ \hat{{T}}^{j}_r\right]\|_{1} \leq \sum_{r=1}^R \frac{1}{R} 8C\Delta = 8C\Delta.
\label{final1}
\end{aligned}    
\end{equation}
Also, let
\begin{equation}
\begin{aligned}
\bar{{T}}_r^j&=\begin{pmatrix} 1-\bar{\rho}_{-} & \bar{\rho}_{-} \\ \bar{\rho}_{+} & 1-\bar{\rho}_{+} \end{pmatrix} = \begin{pmatrix} 1-\sum_{r=1}^R\frac{1}{R}\hat{\rho}^r_{-} & \sum_{r=1}^R\frac{1}{R}\hat{\rho}^r_{-} \\ \sum_{r=1}^R\frac{1}{R}\hat{\rho}^r_{+} & 1-\sum_{r=1}^R\frac{1}{R}\hat{\rho}^r_{+} \end{pmatrix},
\end{aligned}    
\end{equation}
we can get
\begin{equation}
\begin{aligned}
\operatorname{Var}[\bar{\rho}_+]&=\operatorname{Var}[\sum_{r=1}^R\frac{1}{R}\hat{\rho}^r_{+}]=\sum_{r=1}^R\frac{1}{R^2}\operatorname{Var}[\hat{\rho}^r_{+}]\\
&\leq \sum_{r=1}^R \frac{C^2}{R^2}(\frac{3}{4n}+\frac{1}{2\lambda_1 n})= \frac{C^2}{R}(\frac{3}{4n}+\frac{1}{2\lambda_1 n}).
\end{aligned}
\label{eq44}
\end{equation}
Similarly, we could prove
\begin{equation}
\operatorname{Var}[\bar{\rho}_-] \leq \frac{C^2}{R}(\frac{3}{4n}+\frac{1}{2\lambda_1 n}).
\end{equation}

Since $\|\mathbb{E}[{\bar{T}}^j]-{\bar{T}}^j\|_{1} = 2\left| \mathbb{E}[\bar{\rho}_{-}]-\bar{\rho}_{-}\right|+2\left| \mathbb{E}[\bar{\rho}_{+}]-\bar{\rho}_{+}\right|$, with the variances of $\bar{\rho}_+$ and $\bar{\rho}_-$ known, we consider to bound $\|\mathbb{E}[{\bar{T}}^j]-{\bar{T}}^j\|_{1}$ by employing certain concentration inequalities~\cite{boucheron2013concentration}.

First, by employing Chebyshev’s inequality~\cite{de1867valeurs}, we can derive below bound:
\begin{Lemma}\label{lemma:5}
for any $\delta>0$, with probability at least $1-\delta$, we have
\begin{equation}
\begin{aligned}
\|\mathbb{E}[{\bar{T}}^j]-{\bar{T}}^j\|_{1} 
\leq 4C\sqrt{(\frac{3}{4n}+\frac{1}{2\lambda_1 n})\frac{2}{R\delta}}.
\label{final2}
\end{aligned}    
\end{equation}
\end{Lemma}
The proof of Lemma \ref{lemma:5} is provided in Appendix \ref{sec:app5}.

Second, by employing Bernstein’s inequality~\cite{bernstein1946theory}, we can derive below bound:
\begin{Lemma}\label{lemma:6}
for any $\delta>0$, with probability at least $1-\delta$, we have
\begin{equation}
\begin{aligned}
\|\mathbb{E}[{\bar{T}}^j]-{\bar{T}}^j\|_{1} 
\leq  4C\sqrt{\frac{2\log (4 / \delta)}{R}(\frac{3}{4n}+\frac{1}{2\lambda_1 n})}+\frac{8 \log (4/\delta)}{3 R}.
\label{final2}
\end{aligned}    
\end{equation}
\end{Lemma}
Proof of Lemma \ref{lemma:6} is provided in Appendix \ref{sec:app6}.

Finally, according to Eq.~(\ref{final1}), Lemmas~\ref{lemma:5}, and~\ref{lemma:6}, we can conclude that for any $\delta>0$, with probability at least $1-\delta$,
\begin{align*}
  \|{T}^j-{\hat{T}}^j\|_{1} \leq  8C\Delta + \min \left( 4C\sqrt{\left(\frac{3}{4n}+\frac{1}{2\lambda_1 n}
\right)\frac{2}{{R\delta}}},
4C\sqrt{(\frac{3}{4n}+\frac{1}{2\lambda_1 n})\frac{2\log (4 / \delta)}{R}}+\frac{8 \log (4/\delta)}{3 R} \right).  
\end{align*}

\subsection{Proof of Lemma~\ref{lemma:4}}
\label{sec:app4}
First of all, as mentioned in Section~\ref{ours} and Eq.~(\ref{eq10}), for the $r$-{th} estimation, our estimator need to first estimate $\{{e}^r_{00},{e}^r_{01}, {e}^r_{10},{\rho}'^{r}_+, {\rho}'^{r}_-\}$ by $\{\hat{e}^r_{00},\hat{e}^r_{01}, \hat{e}^r_{10},\hat{\rho}'^{r}_+, \hat{\rho}'^{r}_-\}$. For simplicity, we omitted the superscript $r$ below when no confusion is caused.

According to Eq.~(\ref{eq6}) and Eq.~(\ref{eq7}), it is easy to get that
$$ \hat{\rho}_+=\frac{\hat{e}_{00}-(\hat{e}_{00}+\hat{e}_{01})(1-\hat{\rho}'_{-})}{\hat{e}_{10}+\hat{e}_{00}+\hat{\rho}'_{-}-1},\quad \text{and}$$ $$\hat{\rho}_{-}=\frac{\hat{e}_{10}-(1-\hat{e}_{01}-\hat{e}_{00})\hat{\rho}'_{+}}{\hat{e}_{10}+\hat{e}_{00}-\hat{\rho}'_{+}},$$ 
which means $ \hat{\rho}_+$ and $\hat{\rho}_-$ can be seen as the function of the estimated co-occurrence probabilities ${\theta}=\{\hat{e}_{00},\hat{e}_{01}, \hat{e}_{10},\hat{\rho}'_+, \hat{\rho}'_-\}$, and 
\begin{equation}
\begin{aligned}
\hat{{T}}_r^j({\theta})=\begin{pmatrix} 1-\hat{\rho}_{-}({\theta}) & \hat{\rho}_{-}({\theta}) \\ \hat{\rho}_{+}({\theta}) & 1-\hat{\rho}_{+}({\theta}) \end{pmatrix}.
\end{aligned}
\end{equation}
Then, we begin to obtain the expectation and variance of $\hat{\rho}_+$ and $\hat{\rho}_{-}$.

Let ${\theta}^*=\{{e}_{00}, {e}_{01}, {e}_{10}, {\rho}'_{+}+\Delta_1, {\rho}'_{-}+\Delta_0\}$ be the expectation of ${\theta}$ as derived in Section~\ref{est_error} (See Eq.~(\ref{eq10})). According to Taylor's theorem~\cite{spivak2019calculus}, $\hat{\rho}_+\left({\theta}\right)$ can be expressed as
\begin{equation}
\begin{aligned}
\hat{\rho}_+\left({\theta}\right)=\hat{\rho}_+({\theta}^*)+\sum_{i=1}^5\left(\theta_i-\theta_i^*\right) \frac{\partial \hat{\rho}_+}{\partial\theta_i}\left({\theta}^*\right)+ o({\theta}).
\end{aligned}
\end{equation}
Since the estimated co-occurrence probabilities ${\theta}$ will converge to their expectation ${\theta}^*$ exponentially fast by counting~\cite{boucheron2013concentration}, when the training sample size is large, we have 
\begin{equation}
\begin{aligned}
\hat{\rho}_+\left({\theta}\right) \approx \hat{\rho}_+({\theta}^*)+\sum_{i=1}^5\left(\theta_i-\theta_i^*\right) \frac{\partial \hat{\rho}_+}{\partial\theta_i}\left({\theta}^*\right).
\end{aligned}
\end{equation}
Hence, we can get the expectation and the variance of $ \hat{\rho}_+$ as
\begin{align}
\mathbb{E}_{{\theta}}[\hat{\rho}_+({\theta})]  &\approx \hat{\rho}_+({\theta}^*)+\sum_{i=1}^5 \mathbb{E}_{{\theta}}[\left(\theta_i-\theta_i^*\right)] \frac{\partial \hat{\rho}_+}{\partial\theta_i}\left({\theta}^*\right)  \nonumber \\
&=\hat{\rho}_+({\theta}^*), \label{eq42}\quad \text{and} \\
\operatorname{Var}_{{\theta}} [\hat{\rho}_+({\theta})] & \approx \mathbb{E}_{{\theta}}\left[\left(\hat{\rho}_+({\theta})-\hat{\rho}_+({\theta}^*)\right)^2\right] \nonumber \\
& \approx \mathbb{E}_{{\theta}}\left[\left(\sum_{i=1}^5\left(\theta_i-\theta_i^*\right) \frac{\partial \hat{\rho}_+}{\partial\theta_i}\left({\theta}^*\right)\right)^2\right] \nonumber \\
& =\sum_{i=1}^5\left(\frac{\partial \hat{\rho}_+}{\partial\theta_i}\left({\theta}^*\right)\right)^2 \operatorname{Var}_{{\theta}}[\theta_i]
\nonumber +2 \sum_{i>j} \frac{\partial \hat{\rho}_+}{\partial\theta_i}\left({\theta}^*\right) \frac{\partial \hat{\rho}_+}{\partial\theta_j}\left({\theta}^*\right)
\operatorname{Cov}_{{\theta}}\left[\theta_i, \theta_j\right] \nonumber \\
& =\sum_{i=1}^5\left(\frac{\partial \hat{\rho}_+}{\partial\theta_i}\left({\theta}^*\right)\right)^2 \operatorname{Var}_{{\theta}}[\theta_i], \label{eq43}
\end{align}
where the last equation holds because of the independence of the estimated co-occurrence probabilities.
Similarly, as for $\hat{\rho}_-$, we have
\begin{align}
\mathbb{E}_{{\theta}}[\hat{\rho}_-({\theta})]  &\approx \hat{\rho}_-({\theta}^*), \quad \text{and} \\
\operatorname{Var}_{{\theta}} [\hat{\rho}_-({\theta})] & \approx \sum_{i=1}^5\left(\frac{\partial \hat{\rho}_-}{\partial\theta_i}\left({\theta}^*\right)\right)^2 \operatorname{Var}_{{\theta}}[\theta_i].
\end{align}
Now, we can obtain 
\begin{equation}
\begin{aligned}
    \mathbb{E}[\hat{{T}}^j_r({\theta})]&=\begin{pmatrix} 1-\mathbb{E}[\hat{\rho}_{-}({\theta})] & \mathbb{E}[\hat{\rho}_{-}({\theta})] \\ \mathbb{E}[\hat{\rho}_{+}({\theta})] & 1-\mathbb{E}[\hat{\rho}_{+}({\theta})]\end{pmatrix} \\
    & =\begin{pmatrix} 1-\hat{\rho}_{-}({\theta}^*) & \hat{\rho}_{-}({\theta}^*) \\ \hat{\rho}_{+}({\theta}^*) & 1-\hat{\rho}_{+}({\theta}^*) \end{pmatrix} =  \hat{{T}}^j_r({\theta}^*)
\end{aligned}    
\end{equation}

Let ${\theta}^\divideontimes=\{{e}_{00}, {e}_{01}, {e}_{10}, {\rho}'_{+}, {\rho}'_{-}\}$ be the true probabilities of ${\theta}$.
Since, according to Eq.~(\ref{system}), $\hat{{T}}^j({\theta}^\divideontimes)={{T}}^j$, then 
\begin{equation}
\begin{aligned}
\|{T}^j-\mathbb{E}[{\hat{T}}^j_r]\|_{1} &= \|\hat{{T}}^j({\theta}^\divideontimes)-\hat{{T}}^j({\theta}^*)\|_{1}\\
&=2|\hat{\rho}_{-}({\theta}^\divideontimes)-\hat{\rho}_{-}({\theta}^*)|+2|\hat{\rho}_{+}({\theta}^\divideontimes)-\hat{\rho}_{+}({\theta}^*)| \\
& \approx 2|\sum_{i=4}^5\left(\theta_i-\theta_i^*\right) \frac{\partial \hat{\rho}_-}{\partial\theta_i}\left({\theta}^*\right)| + 2|\sum_{i=4}^5\left(\theta_i-\theta_i^*\right) \frac{\partial \hat{\rho}_+}{\partial\theta_i}\left({\theta}^*\right)| \\
& \leq 2(\Delta_0 C+\Delta_1 C) + 2(\Delta_0 C+\Delta_1 C) \leq 8C\Delta,
\nonumber
\end{aligned}    
\end{equation}
where $C$ is the maximum absolute value of the first derivative of $\hat{\rho}^r_+$ and $\hat{\rho}^r_-$ w.r.t. the estimated co-occurrence probabilities ${\theta}$. According to Eq.~(\ref{eq10}), it is easy to get that
\begin{equation}
\begin{aligned}
&\operatorname{Var}[\hat{e}_{kv}]= \frac{{e}_{kv}(1-{e}_{kv})}{n}, \qquad \text{  for } k,v\in \{0,1\},\\
&\operatorname{Var}[\hat{\rho}'_+]= \frac{({\rho}'_++\Delta_1)(1-{\rho}'_+-\Delta_1)}{4\lambda_1 n},\\
&\operatorname{Var}[\hat{\rho}'_-]= \frac{({\rho}'_++\Delta_1)(1-{\rho}'_+-\Delta_1)}{4\lambda_0 n}.
\end{aligned}    
\end{equation}
Therefore,
\begin{equation*}
\begin{aligned}
&\operatorname{Var}[\theta_i]= \leq \frac{1}{4n}, \qquad \text{  for } i=1,2,3,\\
&\operatorname{Var}[\theta_i]= \leq \frac{1}{4\lambda_1 n}, \quad \text{for } i=4,5.
\end{aligned}    
\end{equation*}

\subsection{Proof of Lemma~\ref{lemma:5}}
\label{sec:app5}
According to Eq.~(\ref{eq44}), by employing Chebyshev’s inequality, for any $\delta>0$, with probability at least $1-\delta$, we have
$$
\left| \mathbb{E}[\bar{\rho}_{+}]-\bar{\rho}_{+}\right| \leq \sqrt{\frac{\operatorname{Var}[\bar{\rho}^r_+]}{\delta}}\leq C\sqrt{(\frac{3}{4n}+\frac{1}{2\lambda_1 n})\frac{1}{R\delta}}.$$
Similarly, we could prove that, for any $\delta>0$, with probability at least $1-\delta$, 
$$
\left| \mathbb{E}[\bar{\rho}_{-}]-\bar{\rho}_{-}\right| \leq \sqrt{\frac{\operatorname{Var}[\bar{\rho}^r_-]}{\delta}}\leq C\sqrt{(\frac{3}{4n}+\frac{1}{2\lambda_1 n})\frac{1}{R\delta}}.$$
Then, for any $\delta>0$, with probability at least $1-\delta$, we have
\begin{equation}
\begin{aligned}
\|\mathbb{E}[{\bar{T}}^j]-{\bar{T}}^j\|_{1} 
& = 2\left| \mathbb{E}[\bar{\rho}_{-}]-\bar{\rho}_{-}\right|+2\left| \mathbb{E}[\bar{\rho}_{+}]-\bar{\rho}_{+}\right||\\
& \leq 4C\sqrt{(\frac{3}{4n}+\frac{1}{2\lambda_1 n})\frac{2}{R\delta}}.
\label{final2}
\end{aligned}    
\end{equation}
\subsection{Proof of Lemma~\ref{lemma:6}}
\label{sec:app6}
We consider to bound $\|\mathbb{E}[{\bar{T}}^j]-{\bar{T}}^j\|_{1}$ by employing the below theorem implied by Bernstein’s inequality~\cite{boucheron2013concentration}.
\begin{Theorem}
Let $X_1, \ldots, X_n$ be independent random variables with finite variances, and assume $\max _{1 \leq i \leq n}\left|X_i\right| \leq B$ almost surely for some constant  $B>0$. Let $V=\sum_{i=1}^n \mathbb{E}[X_i^2]$. Then, for any $\delta>0$, with probability at least $1-\delta$,
$$\left|\frac{1}{n} \sum_{i=1}^n\left(\mathbb{E}[X_i]-X_i\right)\right| \leq \frac{\sqrt{2 V \log (2 / \delta)}}{n}+\frac{2 B \log (2/\delta)}{3 n}.$$
\label{Bernstein}
\end{Theorem}
Since we can get
\begin{equation}
\begin{aligned}
\|\mathbb{E}[{\bar{T}}^j]-{\bar{T}}^j\|_{1} &= 2\left| \mathbb{E}[\bar{\rho}_{-}]-\bar{\rho}_{-}\right|+2\left| \mathbb{E}[\bar{\rho}_{+}]-\bar{\rho}_{+}\right|\\
& = 2\left| \mathbb{E}[\sum_{r=1}^R\frac{1}{R}\hat{\rho}^r_{-}]-\sum_{r=1}^R\frac{1}{R}\hat{\rho}^r_{-}\right|+2\left| \mathbb{E}[\sum_{r=1}^R\frac{1}{R}\hat{\rho}^r_{+}]-\sum_{r=1}^R\frac{1}{R}\hat{\rho}^r_{+}\right| \\
& = 2\left| \sum_{r=1}^R\frac{1}{R}(\mathbb{E}[\hat{\rho}^r_{-}]-\hat{\rho}^r_{-})\right|
+2 \left| \sum_{r=1}^R\frac{1}{R}(\mathbb{E}[\hat{\rho}^r_{+}]-\hat{\rho}^r_{+})\right|,
\end{aligned}    
\end{equation}
where we consider to bound $\sum_{r=1}^R \mathbb{E}[(\hat{\rho}^r_{+})^2]$ and $\sum_{r=1}^R \mathbb{E}[(\hat{\rho}^r_{-})^2]$.

According to Eq.~(\ref{eq44}), we have
\begin{equation}
\begin{aligned}
\sum_{r=1}^R \mathbb{E}[(\hat{\rho}^r_{+})^2]&=\sum_{r=1}^R \left(\operatorname{Var}[\hat{\rho}^r_+] -\mathbb{E}[(\hat{\rho}^r_{+})]^2\right)\leq \sum_{r=1}^R (\operatorname{Var}[\hat{\rho}^r_+]) \\
&\leq \sum_{r=1}^R C^2(\frac{3}{4n}+\frac{1}{2\lambda_1 n}) = RC^2(\frac{3}{4n}+\frac{1}{2\lambda_1 n}).
\end{aligned}    
\end{equation}

After that, since $\max_r|\hat{\rho}^r_{+}| \leq 1$, by employing Theorem~\ref{Bernstein}, for any $\delta>0$, with probability at least $1-\delta$, we have
$$
\left|\frac{1}{R} \sum_{r=1}^R\left(\mathbb{E}[\hat{\rho}^r_{+}]-\hat{\rho}^r_{+}\right)\right| \leq C\sqrt{\frac{2\log (2 / \delta)}{R}(\frac{3}{4n}+\frac{1}{2\lambda_1 n})}+\frac{2 \log (2/\delta)}{3 R}.$$
Similarly, we could prove that 
$$
\left|\frac{1}{R} \sum_{r=1}^R\left(\mathbb{E}[\hat{\rho}^r_{-}]-\hat{\rho}^r_{-}\right)\right| \leq C\sqrt{\frac{2\log (2 / \delta)}{R}(\frac{3}{4n}+\frac{1}{2\lambda_1 n})}+\frac{2 \log (2/\delta)}{3 R}.$$
Then, for any $\delta>0$, with probability at least $1-\delta$, we have
\begin{equation*}
\begin{aligned}
\|\mathbb{E}[{\bar{T}}^j]-{\bar{T}}^j\|_{1} 
& = 2\left| \sum_{r=1}^R\frac{1}{R}(\mathbb{E}[\hat{\rho}^r_{-}]-\hat{\rho}^r_{-})\right|
+2 \left| \sum_{r=1}^R\frac{1}{R}(\mathbb{E}[\hat{\rho}^r_{+}]-\hat{\rho}^r_{+})\right|\\
& \leq 4C\sqrt{\frac{2\log (4 / \delta)}{R}(\frac{3}{4n}+\frac{1}{2\lambda_1 n})}+\frac{8 \log (4/\delta)}{3 R}.
\label{final2}
\end{aligned}    
\end{equation*}
\section{Proof of Theorem~\ref{thm:main}}
\label{proof_5}
For applying the Reweight algorithm~\cite{Liu2016TPAMI} in noisy multi-label learning, we have 
\begin{equation}
\begin{aligned}
\bar{R}_{n, w}(\{{T}^{j}\}^q_{j=1}, {f})
&=\frac{1}{n} \sum_{i=1}^{n} \sum_{j=1}^{q} \frac{\mathbbm{\hat{P}}( {Y}^{j}=\bar{y}^{j}_{i} \mid {X}={x}_i)}{\mathbbm{\hat{P}}( \bar{Y}^{j}=\bar{y}^{j}_{i} \mid {X}={x}_i)} \ell\left(f_j\left({x}_{i}\right), \bar{y}^{j}_{i}\right)\\
&=\frac{1}{n}\sum_{i=1}^{n} \sum_{j=1}^{q}\frac{[g'_j({x}_i)]_{\bar{y}^j_i}}{[({{{T}}^j)}^\top{g'}_j({x}_i)]_{\bar{y}^j_i}}\ell\left(f_j\left({x}_{i}\right), \bar{y}^{j}_{i}\right),
\end{aligned}
\end{equation}
where $g'_j({x})=[1-g_j({x}), g_j({x})]^\top$, $f_j({x})=\mb{I}[g_j({x})>0.5]$, and $\mb{I}[.]$ is the indicator function which takes 1 if the identity index is true and 0 otherwise.

Let $S=\{({x}_1,\bar{{y}}_1),\ldots, ({x}_i,\bar{{y}}_i), \ldots, ({x}_n,\bar{{y}}_n)\}$, $S^i=\{({x}_1,\bar{{y}}_1),\ldots,({x}_{i-1},\bar{{y}}_{i-1}), ({x}'_i,\bar{{y}}'_i), ({x}_{i+1},\bar{{y}}_{i+1}), \ldots, ({x}_n,\bar{{y}}_n)\}$, and
\begin{equation}
\Phi(S)=\sup_{{f}}(\bar{R}_{n,w}(\{{{T}^j}\}_{j=1}^q,{f})-\mathbb{E}_S[\bar{R}_{n,w}(\{{{T}^j}\}_{j=1}^q,{f})]).
\end{equation}

\begin{Lemma}\label{lemma:1}
Let  $\hat{{f}}$ be the learned classifier, and $U=\frac{1}{\min_j(1-\max(\mathbbm{P}(\bar{Y}^{j}=0 \mid Y^{j}=1),\mathbbm{P}(\bar{Y}^{j}=1 \mid Y^{j}=0)))}$.
For any $\delta>0$, with probability at least $1-\delta$, we have
\begin{align}
\mathbb{E}[\bar{R}_{n,w}(\{{{T}^j}\}_{j=1}^q,{\hat{f}})]- \bar{R}_{n,w}(\{{{T}^j}\}_{j=1}^q,{\hat{f}}) \leq \mathbb{E}[\Phi(S)]+UM\sqrt{\frac{(\log{1/\delta)}}{2n}}.
\nonumber
\end{align}
\end{Lemma}
The proof of Lemma \ref{lemma:1} is provided in Appendix \ref{sec:app1}. Using the same tricks following~\cite{XiaLW00NS19}, we can upper bound $\mathbb{E}[\Phi(S)]$ by the following lemma.


\begin{Lemma} \label{lemma:2}
\begin{align}
\mathbb{E}[\Phi(S)] \leq 2\mathbb{E}\left[\sup_{{f}}\frac{1}{n}\sum_{i=1}^{n}\sum_{j=1}^q\sigma_i\ell(f_j({x}_i),\bar{y}^j_i)\right],
\nonumber
\end{align}
where $\sigma_1,\ldots,\sigma_n$ are i.i.d. Rademacher random variables.
\end{Lemma}
Proof of Lemma \ref{lemma:2} is provided in Appendix \ref{sec:app2}.

Recall that $f_j({x})=\mathbb{I}\left[g_j({x})>0.5\right], j=1,\ldots,q,$ is the classifier, where $g_j$ is the output of the sigmoid function, \ie, $g_j({x})=1/(1+\exp{(-h_j({x}))}), j=1,\ldots,q$, and ${h}({x})$ is defined by a $d$-layer neural network, \ie, ${h}: {x}\mapsto {W}_d\sigma_{d-1}({W}_{d-1}\sigma_{d-2}(\ldots \sigma_1({W}_1{x})))\in\mathbb{R}^q$,  ${W}_1,\ldots,{W}_d$ are the parameter matrices, and $\sigma_1,\ldots,\sigma_{d-1}$ are activation functions. To further upper bound the Rademacher complexity, we need to consider the Lipschitz continuous property of the loss function w.r.t. to ${h}_j({x})$. To avoid more assumptions, We discuss the widely used {binary cross-entropy loss}, \ie, 
\begin{equation}
\mathcal{L}({{f}}({X}), {Y}) =\sum_{j=1}^{q} \ell(f_j({X}),Y^{j}) = \sum_{j=1}^{q} Y^{j}\log(g_j({X}))+(1-Y^{j})\log(1-g_j({X})).
\end{equation}

Then, we can further upper bound the Rademacher complexity by the following lemma.
\begin{Lemma} \label{lemma:3}
\begin{align}
    \mathbb{E}\left[\sup_{f}\frac{1}{n}\sum_{i=1}^{n}\sum_{j=1}^{q}\sigma_i\ell(f_j({x}_i),\bar{y}^j_i)\right]\leq qL\mathbb{E}\left[ \sup_{{h}\in {H} }\frac{1}{n}\sum_{i=1}^{n}\sigma_i h({x}_i)\right],
    \nonumber
\end{align}
where $H$ is the function class induced by the deep neural network.
\end{Lemma}
The proof of Lemma \ref{lemma:3} is provided in Appendix \ref{sec:app3}. Note that $\mathbb{E}\left[ \sup_{h\in H }\frac{1}{n}\sum_{i=1}^{n}\sigma_ih({x}_i)\right]$ measures the hypothesis complexity of deep neural networks, which can be bounded by the following theorem.
\begin{Theorem}\label{thm:network}
Assume the Frobenius norm of the weight matrices ${W}_1,\ldots,{W}_d$ are at most $M_1,\ldots, M_d$. Let the activation functions be 1-Lipschitz, positive-homogeneous, and applied element-wise (such as the ReLU). Let ${x}$ is upper bounded by B, \ie, for any ${x}
\in \mathcal{X}$, $\|{x}\|\leq B$. Then,
\begin{align}
\mathbb{E}\left[ \sup_{h\in H }\frac{1}{n}\sum_{i=1}^{n}\sigma_ih({x}_i)\right]\leq \frac{B(\sqrt{2d\log2}+1)\Pi_{i=1}^{d}M_i}{\sqrt{n}}.
\nonumber
\end{align}
\end{Theorem}
The proof of Theorem~\ref{thm:network} can be found in \cite{golowich2018size} (Theorem 1 therein). Theorem~\ref{thm:main} follows by combining Lemmas 1, 2, 3, and Theorem \ref{thm:network}.
\subsection{Proof of Lemma~\ref{lemma:1}}
\label{sec:app1}
We employ McDiarmid's concentration inequality \cite{boucheron2013concentration} to prove the lemma. We first check the bounded difference property of $\Phi(S)$, \ie,
\begin{align}
\Phi(S)-\Phi(S^i)  \leq \sup_{{f}}\frac{1}{n}\sum_{j=1}^q\left(\frac{[g'_j({x}_i)]_{\bar{y}^j_i}\ell(f_j({x}_i),\bar{y}^j_i)}{[({{T}}^j)^\top g'_j({x}_i)]_{\bar{y}^j_i}}-\frac{[g'_j({x}'_i)]_{\bar{y}'^{j}_i}\ell(f_j({x}'_i),\bar{y}'^{j}_i)}{[({{T}}^j)^\top g'_j({x}'_i)]_{\bar{y}'^{j}_i}}\right).
\end{align}
Before further upper bounding the above difference, 
we show that the importance weight $\frac{[g'_j({{x}})]_{\bar{y}^j}}{[({{T}}^j)^{\top} g'_j({x})]_{\bar{y}^j}}$ is upper bounded by $U=\frac{1}{\min_j(1-\max(\mathbbm{P}(\bar{Y}^{j}=0 \mid Y^{j}=1),\mathbbm{P}(\bar{Y}^{j}=1 \mid Y^{j}=0)))}$ for any $({x},\bar{{y}}^j)$ and $g'^\prime$. Specifically,
if $\bar{y}^j = 0$ and let $\mathbbm{P}(\bar{Y}^{j}=1 \mid Y^{j}=0)=\rho_-$, $g_{j}({x})=p$, we have
\begin{align}
\scriptsize
&w=\frac{[g'_j({{x}})]_{\bar{y}^j}}{[({{T}}^j)^{\top} g'_j({x})]_{\bar{y}^j}}=\frac{[g'_{j}({x})]_0}{[({{T}}^j)^{\top} g'_j({x})]_{0}}=\frac{(1-p)}{(1-p)(1-\rho_-)+p\rho_{-}}
\end{align}
Since $\frac{\mathrm{d}w}{\mathrm{d}p}=\frac{-\rho_{-}}{[(1-p)(1-\rho_-)+p\rho_{-}]^2}<0$, then $w\leq w|_{p=0}=\frac{1}{1-\rho_-}$, \ie,
\begin{align}
\frac{[g'_j({{x}})]_{\bar{y}^j}}{[({{T}}^j)^{\top} g'_j({x})]_{\bar{y}^j}}\leq \frac{1}{1-\mathbbm{P}(\bar{Y}^{j}=1 \mid Y^{j}=0)}, \text{when } \bar{y}^j = 0.\label{u1}
\end{align}
Similarly, we could get that \begin{align}
\frac{[g'_j({{x}})]_{\bar{y}^j}}{[({{T}}^j)^{\top} g'_j({x})]_{\bar{y}^j}}\leq \frac{1}{1-\mathbbm{P}(\bar{Y}^{j}=0 \mid Y^{j}=1)}, \text{when } \bar{y}^j = 1.\label{u2}
\end{align}
Hence, according to Eq.~(\ref{u1}) and~(\ref{u2}),  for any $({x},\bar{{y}}^j)$ and $g'^\prime$, 
$\frac{[g'_j({{x}})]_{\bar{y}^j}}{[({{T}}^j)^{\top} g'_j({x})]_{\bar{y}^j}}\leq\frac{1}{\min_j(1-\max(\mathbbm{P}(\bar{Y}^{j}=0 \mid Y^{j}=1),\mathbbm{P}(\bar{Y}^{j}=1 \mid Y^{j}=0)))}$.

Then, we can conclude that the weighted loss is upper bounded by $UM$ and that
\begin{align}
\Phi(S)-\Phi(S^i)&\leq \frac{U}{n} \sup_{{f}} \sum_{j=1}^q \ell(f_j({x}_i),\bar{y}^j_i) \nonumber \\
&\leq \frac{U}{n}\sup_{{f}}\mathcal{L}({f}({x}),\bar{{y}})  \leq  \frac{UM}{n}.
\end{align}
Similarly, we could prove that $\Phi(S^i)-\Phi(S)\leq \frac{UM}{n}$.

By employing McDiarmid's concentration inequality, for any $\delta>0$, with probability at least $1-\delta$, we have
\begin{equation}
\Phi(S)-\mathbb{E}[\Phi(S)]\leq UM\sqrt{\frac{\log(1/\delta)}{2n}}.
\end{equation}

\subsection{Proof of Lemma~\ref{lemma:2}}
\label{sec:app2}
Using the same trick to derive Rademacher complexity \cite{bartlett2002rademacher}, we have
\begin{align}
 \mathbb{E}[\Phi(S)]\leq 2\mathbb{E}\left[\sup_{{f}}\frac{1}{n}\sum_{i=1}^{n}\sum_{j=1}^q\sigma_i\frac{[g'_{j}({x}_i)]_{\bar{y}^j_i}}{[({{T}}^j)^\top g'_{j}({x}_i)]_{\bar{y}^j_i}}
\ell(f_j({x}_i),\bar{y}^j_i)\right],
\nonumber
\end{align}
where $\sigma_1,\ldots,\sigma_n$ are i.i.d. Rademacher random variables. 

Then, the proof of Lemma~2 is transformed to prove the following inequality:
\begin{align}
\mathbb{E}\left[\sup_{{f}}\frac{1}{n}\sum_{i=1}^{n}\sum_{j=1}^q\sigma_i\frac{[g'_{j}({x}_i)]_{\bar{y}^j_i}}{[({{T}}^j)^\top g'_{j}({x}_i)]_{\bar{y}^j_i}}
\ell(f_j({x}_i),\bar{y}^j_i)\right]\leq \mathbb{E}\left[\sup_{{f}}\frac{1}{n}\sum_{i=1}^{n}\sum_{j=1}^q\sigma_i\ell(f_j({x}_i),\bar{{y}}^j_i)\right].
\end{align}
Note that
\begin{align}
\mathbb{E}_{{\sigma}}\left[\sup_{{f}}\sum_{i=1}^{n}\sum_{j=1}^q\sigma_i\frac{[g'_{j}({x}_i)]_{\bar{y}^j_i}}{[({{T}}^j)^\top g'_{j}({x}_i)]_{\bar{y}^j_i}}
\ell(f_j({x}_i),\bar{y}^j_i)\right]\leq \mathbb{E}_{\sigma_1,\ldots,\sigma_{n-1}}\left[\mathbb{E}_{{\sigma_n}}\left[\sup_{{f}}\sum_{i=1}^{n}\sum_{j=1}^q\sigma_i\ell(f_j({x}_i),\bar{{y}}^j_i)\right]\right].
\end{align}

Let $s_{n-1}({f})=\sum_{i=1}^{n-1}\sum_{j=1}^q\sigma_i\frac{[g'_{j}({x}_i)]_{\bar{y}^j_i}}{[({{T}}^j)^\top g'_{j}({x}_i)]_{\bar{y}^j_i}}
\ell(f_j({x}_i),\bar{y}^j_i)$.

By definition of the supremum, for any $\epsilon>0$, there exist ${f}^1$ and ${f}^2$ (also $g'^{1}$ and $g'^{2}$) such that
\begin{equation*}
\begin{aligned}
&\sum_{j=1}^q\frac{[{g'^{1}_{j}}({x}_i)]_{\bar{y}^j_i}}{[({{T}}^j)^\top g'^{1}_{j}({x}_i)]_{\bar{y}^j_i}}
\ell(f^1_j({x}_i),\bar{y}^j_i)+s_{n-1}({f}^1)\geq (1-\epsilon)\sup_{ f}\left(\sum_{j=1}^q\frac{[g'_{j}({x}_i)]_{\bar{y}^j_i}}{[({{T}}^j)^\top g'_{j}({x}_i)]_{\bar{y}^j_i}}
\ell(f_j({x}_i),\bar{y}^j_i)+s_{n-1}({f})\right),
\end{aligned}
\end{equation*}
and 
\begin{equation*}
\begin{aligned}
&-\sum_{j=1}^q\frac{[{g'^{2}_{j}}({x}_i)]_{\bar{y}^j_i}}{[({{T}}^j)^\top g'^{2}_{j}({x}_i)]_{\bar{y}^j_i}}
\ell(f^2_j({x}_i),\bar{y}^j_i)+s_{n-1}({f}^2)\geq (1-\epsilon)\sup_{f}\left(-\sum_{j=1}^q\frac{[g'_{j}({x}_i)]_{\bar{y}^j_i}}{[({{T}}^j)^\top g'_{j}({x}_i)]_{\bar{y}^j_i}}
\ell(f_j({x}_i),\bar{y}^j_i)+s_{n-1}({f})\right).
\end{aligned}
\end{equation*}
Thus, for any $\epsilon$, we have
\begin{equation}
\begin{aligned}
&(1-\epsilon)\mathbb{E}_{\sigma_n}\left[\sup_{ {f}}\left(\sum_{j=1}^q\sigma_n\frac{[g'_{j}({x}_i)]_{\bar{y}^j_i}}{[({{T}}^j)^\top g'_{j}({x}_i)]_{\bar{y}^j_i}}
\ell(f_j({x}_i),\bar{y}^j_i)+s_{n-1}({f})\right)\right]\\
&=\frac{(1-\epsilon)}{2}\sup_{ {f}}\left(\sum_{j=1}^q\frac{[{g'}_{j}({x}_i)]_{\bar{y}^j_i}}{[({{T}}^j)^\top g'_{j}({x}_i)]_{\bar{y}^j_i}}
\ell(f_j({x}_i),\bar{y}^j_i)+s_{n-1}({f})\right)\\
&+\frac{(1-\epsilon)}{2}\sup_{ {f}}\left(-\sum_{j=1}^q\frac{[{g'}_{j}({x}_i)]_{\bar{y}^j_i}}{[({{T}}^j)^\top g'_{j}({x}_i)]_{\bar{y}^j_i}}
\ell(f_j({x}_i),\bar{y}^j_i)+s_{n-1}({f})\right)\\
&\leq\frac{1}{2}\left(\sum_{j=1}^q\frac{[{g'^{1}_{j}}({x}_i)]_{\bar{y}^j_i}}{[({{T}}^j)^\top g'^{1}_{j}({x}_i)]_{\bar{y}^j_i}}
\ell(f^1_j({x}_i),\bar{y}^j_i)+s_{n-1}({f}^1)-\sum_{j=1}^q\frac{[{g'^{2}_{j}}({x}_i)]_{\bar{y}^j_i}}{[({{T}}^j)^\top g'^{2}_{j}({x}_i)]_{\bar{y}^j_i}}
\ell(f^2_j({x}_i),\bar{y}^j_i)+s_{n-1}({f}^2)\right)\\
&\leq\frac{1}{2}\left(s_{n-1}({f}^1)+ s_{n-1}({f}^2)+U|\sum_{j=1}^q\ell(f^1_j({x}_i),\bar{y}^j_i)-\sum_{j=1}^q\ell(f^2_j({x}_i),\bar{y}^j_i)|\right),
\end{aligned}
\end{equation}
where the last inequality holds because $\frac{[g'_{j}({x})]_{\bar{y}^j}}{[({{T}}^j)^\top g'_{j}({x})]_{\bar{y}^j}}\leq U$ for any $({x},\bar{y}^j)$ and $g'$.

Let $s=\text{sgn}\left[\sum_{j=1}^q\ell(f^1_j({x}_i),\bar{y}^j_i)-\sum_{j=1}^q\ell(f^2_j({x}_i),\bar{y}^j_i)\right]U$, where $\text{sgn}[z]$ is the sign of a real number $z$. We have
\begin{equation}
\begin{aligned}
&(1-\epsilon)\mathbb{E}_{\sigma_n}\left[\sup_{ {f}}\left(\sum_{j=1}^q\sigma_n\frac{[g'_{j}({x}_i)]_{\bar{y}^j_i}}{[({{T}}^j)^\top g'_{j}({x}_i)]_{\bar{y}^j_i}}
\ell(f_j({x}_i),\bar{y}^j_i)+s_{n-1}({f})\right)\right]\\
&\leq\frac{1}{2}\left(s_{n-1}({f}^1)+ s_{n-1}({f}^2)+s(\sum_{j=1}^q\ell(f^1_j({x}_i),\bar{y}^j_i)-\sum_{j=1}^q\ell(f^2_j({x}_i),\bar{y}^j_i))\right)\\
&=\frac{1}{2}\left(s_{n-1}({f}^1)+s\sum_{j=1}^q\ell(f^1_j({x}_i),\bar{y}^j_i) \right)+\frac{1}{2}\left(s_{n-1}({f}^2)-s\sum_{j=1}^q\ell(f^2_j({x}_i),\bar{y}^j_i)\right)\\
&\leq \frac{1}{2}\sup_{{f}}\left(s_{n-1}({f})+s\sum_{j=1}^q\ell(f_j({x}_i),\bar{y}^j_i) \right)+\frac{1}{2}\sup_{{f}} \left(s_{n-1}({f})-s\sum_{j=1}^q\ell(f_j({x}_i),\bar{y}^j_i)\right)\\
&=\mathbb{E}_{\sigma_n}\left[\sup_{ {f}}\left(\sum_{i=1}^{n}\sum_{j=1}^q\sigma_i\ell(f_j({x}_i),\bar{{y}}^j_i)+s_{n-1}({f})\right)\right].
\end{aligned}
\end{equation}
Since the above inequality holds for any $\epsilon>0$, we have
\begin{equation}
\begin{aligned}
&\mathbb{E}_{\sigma_n}\left[\sup_{ {f}}\left(\sum_{j=1}^q\sigma_n\frac{[g'_{j}({x}_i)]_{\bar{y}^j_i}}{[({{T}}^j)^\top g'_{j}({x}_i)]_{\bar{y}^j_i}}
\ell(f_j({x}_i),\bar{y}^j_i)+s_{n-1}({f})\right)\right]\leq \mathbb{E}_{\sigma_n}\left[\sup_{ {f}}\left(\sum_{i=1}^{n}\sum_{j=1}^q\sigma_i\ell(f_j({x}_i),\bar{{y}}^j_i)+s_{n-1}({f})\right)\right].
\end{aligned}
\end{equation}

Proceeding in the same way for all other $\sigma$, we have
\begin{equation}
\begin{aligned}
&\mathbb{E}_{{\sigma}}\left[\sup_{{f}}\frac{1}{n}\sum_{i=1}^{n}\sum_{j=1}^q\sigma_i\frac{[g'_{j}({x}_i)]_{\bar{y}^j_i}}{[({{T}}^j)^\top g'_{j}({x}_i)]_{\bar{y}^j_i}}
\ell(f_j({x}_i),\bar{y}^j_i)\right]\leq \mathbb{E}_{{\sigma}}\left[\sup_{{f}}\frac{1}{n}\sum_{i=1}^{n}\sum_{j=1}^q\sigma_i\ell(f_j({x}_i),\bar{{y}}^j_i)\right],
\end{aligned}
\end{equation}
and thus
\begin{equation}
\begin{aligned}
&\mathbb{E}\left[\sup_{{f}}\frac{1}{n}\sum_{i=1}^{n}\sum_{j=1}^q\sigma_i\frac{[g'_{j}({x}_i)]_{\bar{y}^j_i}}{[({{T}}^j)^\top g'_{j}({x}_i)]_{\bar{y}^j_i}}
\ell(f_j({x}_i),\bar{y}^j_i)\right]\leq \mathbb{E}\left[\sup_{{f}}\frac{1}{n}\sum_{i=1}^{n}\sum_{j=1}^q\sigma_i\ell(f_j({x}_i),\bar{{y}}^j_i)\right].
\end{aligned}
\end{equation}
\subsection{Proof of Lemma~\ref{lemma:3}}
\label{sec:app3}
Before proving Lemma~\ref{lemma:3}, we show that the base loss function $\ell(f_j({x}),\bar{{y}}^j)$ is 1-Lipschitz-continuous w.r.t. $h_j({x}),j=\{1,\ldots,q\}$.

Recall that
\begin{align}
{\ell}(f_j({x}),\bar{{y}}^j) &= -(1-\bar{{y}}^j)\log(1-g_j({x}))-\bar{{y}}^j\log(g_j({x})) \nonumber \\
&=(\bar{y}^j-1)\log(1+e^{h_j({x})})+\bar{y}^j\log(1+e^{-h_j({x})}).   \nonumber
\end{align}
Take the derivative of $\ell(f_j({x}),\bar{y}_j)$ w.r.t. $h_j({x})$. If $\bar{{y}}^j = 0$, we have 
\begin{equation} \label{derivative1}
\begin{aligned}
&\frac{\mathrm{d}{\ell}(f_j({x}),\bar{{y}}^j)}{\mathrm{d} h_j({x})} =-\frac{e^{h_j({x})}}{1+e^{h_j({x})}}.
\end{aligned}
\end{equation}
If $\bar{{y}}^j = 1$, we have 
\begin{equation} \label{derivative2}
\begin{aligned}
&\frac{\mathrm{d}{\ell}(f_j({x}),\bar{{y}}^j)}{\mathrm{d} h_j({x})} =-\frac{e^{-h_j({x})}}{1+e^{-h_j({x})}}.
\end{aligned}
\end{equation}
According to Eqs.~(\ref{derivative1}) and (\ref{derivative2}), it is easy to conclude that $-1 \leq \frac{\mathrm{d}{\ell}(f_j({x}),\bar{{y}}^j)}{\mathrm{d} h_j({x})} \leq 1$, which indicates that the loss function is 1-Lipschitz with respect to $h_j({x}), \forall j \in \{1,\ldots,q\}$. 

Now we are ready to prove Lemma 3. Let ${H}$ be the hypothesis space of $h_j$, We have
\begin{equation}
\begin{aligned}
&\mathbb{E}\left[\sup_{{f}=\{f_1,\ldots,f_q\}}\frac{1}{n}\sum_{i=1}^{n}\sum_{j=1}^q\sigma_i\ell(f_j({x}_i),\bar{y}^j_i)\right]\\
&\leq \mathbb{E}\left[\sum_{j=1}^q\sup_{f_j}\frac{1}{n}\sum_{i=1}^{n}\sigma_i{\ell}(f_j({x}_i),\bar{y}^j_i)\right] \\
&=  \mathbb{E}\left[\sum_{j=1}^q\sup_{h_j\in H }\frac{1}{n}\sum_{i=1}^{n}\sigma_i{\ell}(f_j({x}_i),\bar{y}^j_i)\right] \\
&=  \sum_{j=1}^q \mathbb{E}\left[\sup_{h_j\in H }\frac{1}{n}\sum_{i=1}^{n}\sigma_i{\ell}(f_j({x}_i),\bar{y}^j_i)\right] \\
&\leq qL\mathbb{E}\left[ \sup_{h_j\in H }\frac{1}{n}\sum_{i=1}^{n}\sigma_ih_j({x}_i)\right] \\
&=qL\mathbb{E}\left[ \sup_{h\in H }\frac{1}{n}\sum_{i=1}^{n}\sigma_ih({x}_i)\right],
\end{aligned}
\end{equation}
where the second equation holds because $f_j$ and $h_j$ are in one-to-one correspondence; the fourth inequality holds because of the Talagrand Contraction Lemma \cite{ledoux2013probability}.
	
	
 
	\section{The Wilcoxon Signed-Ranks Test for Reweight-Ours against Baselines}
	\label{summary}
	Wilcoxon signed-ranks test~\cite{test} is employed to show whether Reweight-Ours has a significant performance
	than other comparing approaches. Note that the performance of these methods used for the Wilcoxon signed-ranks test is from both class-dependent label noise cases and instance-dependent label-noise cases. As shown in Tab.~\ref{win_loss}, Reweight-Ours outperforms all baselines on both OF1 and CF1 metrics at 0.01 significance level.
	
	\begin{table*}[h]
		\caption{Summary of the Wilcoxon signed-ranks test for Reweight-Ours against other comparing approaches at 0.01 significance level. The p-values are shown in the brackets.}
		\centering
		\scriptsize
		\setlength\tabcolsep{3.8pt}
		\begin{tabular}{c|cccccccccc}
			\hline	
			Reweight-Ours & \multirow{2}*{Standard} & \multirow{2}*{GCE} & \multirow{2}*{CDR} & \multirow{2}*{AGCN} & \multirow{2}*{CSRA} & \multirow{2}*{WSIC} &Reweight-T& Reweight-T &Reweight-DualT& Reweight-DualT\\
			against  &  &  &  &  &  &  &max& 97\% & max& 97\% \\	
			\hline	
			mAP  &  \textbf{win} [0.000]  & \textbf{win} [0.004]   & \textbf{win} [0.004]  &   \textbf{win} [0.000]    &  tie [0.772]  &  \textbf{win} [0.000]  & \textbf{win} [0.000]  & \textbf{win} [0.000]  & \textbf{win} [0.000]  & \textbf{win} [0.000]   \\	
			OF1  &  \textbf{win} [0.000]   &  \textbf{win} [0.000]   & \textbf{win} [0.000]   &   \textbf{win} [0.000]     &  \textbf{win} [0.000]   &  \textbf{win} [0.000]   & \textbf{win} [0.000]    & \textbf{win} [0.000]    &  \textbf{win} [0.000]   &  \textbf{win} [0.000]   \\
			CF1  &   \textbf{win} [0.000]   &  \textbf{win} [0.000]   & \textbf{win} [0.000]   &   \textbf{win} [0.000]     &  \textbf{win} [0.000]   &  \textbf{win} [0.000]   & \textbf{win} [0.000]    & \textbf{win} [0.000]    &  \textbf{win} [0.000]   &  \textbf{win} [0.000]  \\
			\hline	
		\end{tabular}
		\label{win_loss}
	\end{table*}

\section{Limitations}
	\label{limit}
	This work still has certain limitations, including:
 \begin{itemize}
     \item This work exploits the memorization effect~\cite{arpit2017closer} in deep learning to perform sample selection, while the memorization effect has not been found in other traditional machine learning methods, and therefore, the proposed estimator can not be applied to such learning methods. 
     \item This work estimates occurrence probabilities using frequency counting. Although this estimation error will converge to zero exponentially fast~\cite{Boucheron2013ConcentrationI}, when the number of one label appearing is too small, \eg, less than 50, the estimation of the transition matrix for this class label is still difficult to be accurate.
     \item Since our work assumes label noise is class-dependent but instance-independent, when this assumption does not hold, the estimation is not guaranteed. The discussions about the relaxation of instance-independent assumption can be found in Appendix~\ref{noise_assumption}, which reveals its applicability in certain typical instance-dependent cases.
 \end{itemize}
\section{Complete Numerical Experimental Results}
\label{complete}
In the main paper, we have provided illustrations comparing the classification performance of Reweight-Ours with several state-of-the-art methods. Here, complete numerical experimental results are presented in Tab.~\ref{cls_VOC2007_c}-~\ref{discuss2_VOC2007} for checks and references.

	\begin{table*}[t]
		\caption{Complete numerical results for classification performance on Pascal-VOC2007 dataset with class-dependent label noise.}
		\centering
		\scriptsize
		\setlength\tabcolsep{5.8pt}
		\begin{tabular}{l|l|cc|cc|cc|cc}
			\hline	
			& Noise rates $(\rho_{-}, \rho_{+})$  &  (0,0.2) & (0,0.6) & (0.2,0) & (0.6,0)  & (0.1,0.1) & (0.2,0.2) &(0.017,0.2)& (0.034,0.4) \\	
			\hline
			\multirow{12}*{\rotatebox{90}{mAP}}& Standard  & 84.25$\pm$1.07  & 77.16$\pm$0.94 & 82.70$\pm$0.54 &  68.65$\pm$1.57 &  83.07$\pm$0.45 &  78.87$\pm$0.52 & 83.92$\pm$0.59  &  80.97$\pm$0.42  \\	
			&GCE  & 83.85$\pm$1.09 & 73.32$\pm$2.22 & 83.03$\pm$0.51 & 67.47$\pm$1.74 & 83.68$\pm$0.66 & 79.39$\pm$0.95 & 84.40$\pm$0.34 & 80.68$\pm$0.52 \\	
			&CDR  & \secbest{84.60$\pm$0.43} & 77.45$\pm$1.23 & 82.76$\pm$0.53 & 68.86$\pm$2.05 & 83.22$\pm$0.57 & 79.02$\pm$0.62 & 84.37$\pm$0.25 & 81.14$\pm$0.28 \\		
			&AGCN  & 83.24$\pm$0.67 & 75.50$\pm$0.56 & 81.09$\pm$0.51 & 66.47$\pm$1.29 & 81.09$\pm$0.48 & 73.79$\pm$0.76 & 82.21$\pm$0.42 & 76.55$\pm$1.11 \\		
			&CSRA  & \textbf{85.11$\pm$0.51} & \textbf{79.47$\pm$1.22} & 82.93$\pm$0.65 & 67.36$\pm$2.25 & 83.69$\pm$0.69 & 78.10$\pm$0.53 & \textbf{84.94$\pm$0.36} & \secbest{81.51$\pm$0.14} \\	
			&WSIC & 84.14$\pm$0.26 & 76.17$\pm$1.31 & 82.30$\pm$0.64 & 66.82$\pm$3.87 & 83.41$\pm$0.31 & 77.93$\pm$1.00 & 84.17$\pm$0.48 & 80.74$\pm$0.44 \\	
			&Reweight-T max   & 84.20$\pm$0.46  & 76.97$\pm$1.20 & 83.04$\pm$0.39 &  71.36$\pm$2.47 &  83.48$\pm$0.15 &  79.10$\pm$0.52 & 84.06$\pm$0.24  &  81.01$\pm$0.99  \\	
			&Reweight-T 97\% & 84.00$\pm$0.68  & \secbest{78.97$\pm$0.69} & 83.07$\pm$0.29 &  73.96$\pm$1.69 &  82.71$\pm$0.30 &  78.80$\pm$0.28 & {84.37$\pm$0.22}  &  {81.42$\pm$0.25}  \\	
			&Reweight-DualT max & {84.46$\pm$0.20}  & 77.65$\pm$1.06 & 83.75$\pm$0.44 &  73.75$\pm$1.61 &  \secbest{83.94$\pm$0.31} &  \secbest{79.48$\pm$1.24} & \secbest{84.60$\pm$0.30}  &  \textbf{81.77$\pm$0.26}  \\	
			&Reweight-DualT 97\%   & 82.36$\pm$0.45  & 77.72$\pm$0.73 & \textbf{84.56$\pm$0.40} &  \textbf{75.76$\pm$2.11} &  79.69$\pm$1.40 &  75.26$\pm$1.70 & 81.84$\pm$0.81  &  77.40$\pm$1.86  \\	
			\cline{2-10}
			& Reweight-Ours  & {84.43$\pm$0.46}  & {78.72$\pm$0.41} & \secbest{84.08$\pm$0.24} &  \secbest{74.46$\pm$0.56} &  \textbf{84.03$\pm$0.29} &  \textbf{80.44$\pm$0.52} & 84.09$\pm$0.62  &  80.97$\pm$1.03  \\
   &Reweight-TrueT  & 84.47$\pm$0.58 & 77.60$\pm$1.23 & 84.20$\pm$0.42 & 75.68$\pm$1.55 & 84.23$\pm$0.46 & 80.29$\pm$0.35 & 84.39$\pm$0.53 & 81.53$\pm$0.56
  \\
			\hline	
			\multirow{12}*{\rotatebox{90}{OF1}}& Standard  & 75.24$\pm$1.40  & 32.02$\pm$5.49 & 78.85$\pm$0.43 &  15.08$\pm$0.25 &  79.24$\pm$0.43 &  {75.85$\pm$0.84} & 75.98$\pm$1.04  &  59.67$\pm$1.65  \\
			&GCE  & 76.17$\pm$1.57 & 36.13$\pm$4.07 & 79.28$\pm$0.44 & 14.85$\pm$0.22 & 79.73$\pm$0.70 & \secbest{76.27$\pm$0.55} & 76.80$\pm$0.68 & 60.26$\pm$2.43  \\	
			&CDR  & 76.05$\pm$0.68 & 34.11$\pm$3.43 & 79.04$\pm$0.46 & 14.99$\pm$0.19 & 79.34$\pm$0.60 & 76.00$\pm$0.47 & 76.56$\pm$0.52 & 59.31$\pm$1.04 \\		
			&AGCN  & 74.92$\pm$1.02 & 30.97$\pm$3.78 & 75.45$\pm$2.06 & 16.85$\pm$0.56 & 78.69$\pm$0.31 & 72.64$\pm$0.51 & 75.16$\pm$0.58 & 56.56$\pm$1.64  \\		
			&CSRA  & 76.94$\pm$1.03 & 33.65$\pm$2.73 & 77.71$\pm$1.23 & 15.94$\pm$0.32 & \textbf{80.36$\pm$0.53} & \textbf{76.92$\pm$0.34} & 77.91$\pm$0.63 & 62.19$\pm$1.97 \\	
			&WSIC & 75.01$\pm$1.18 & 16.48$\pm$6.78 & 79.02$\pm$0.59 & 14.88$\pm$0.21 & 78.55$\pm$1.05 & 72.88$\pm$3.44 & 72.30$\pm$2.82 & 53.26$\pm$9.44 \\	
			&Reweight-T max   & 76.97$\pm$0.45  & 41.54$\pm$2.64 & 79.65$\pm$0.44 &  47.68$\pm$5.65 &  \secbest{80.00$\pm$0.27} &  73.58$\pm$1.67 & 76.94$\pm$0.37  &  66.77$\pm$0.93\\	
			&Reweight-T 97\% & 77.71$\pm$0.65  & \secbest{68.28$\pm$2.03} & 80.16$\pm$0.24 &  \secbest{70.67$\pm$0.70} &  75.28$\pm$0.97 &   65.03$\pm$2.20 & 78.45$\pm$0.63  &  \secbest{74.11$\pm$0.67}  \\	
			&Reweight-DualT max & \secbest{78.38$\pm$0.41}  & \textbf{68.81$\pm$1.41} & 80.02$\pm$1.12 &  65.41$\pm$1.84 &  {79.87$\pm$0.27} &  65.36$\pm$7.31 & \secbest{78.55$\pm$0.36}  &  47.04$\pm$4.15  \\	
			&Reweight-DualT 97\%   & 68.17$\pm$3.53  & 61.81$\pm$3.29 & \secbest{80.74$\pm$0.50} &  \textbf{72.53$\pm$2.65} &  52.99$\pm$5.06 &  36.75$\pm$6.71 & 67.55$\pm$0.64  &  57.41$\pm$0.70  \\	
			\cline{2-10}
			& Reweight-Ours  & \textbf{78.62$\pm$0.58}  & 65.68$\pm$1.67 & \textbf{80.85$\pm$0.25} &  67.43$\pm$4.65 &  79.64$\pm$0.29 &  {75.52$\pm$0.86} & \textbf{79.25$\pm$0.52}  &  \textbf{74.35$\pm$1.65}  \\	
    &Reweight-TrueT  & 78.92$\pm$0.72 & 65.95$\pm$2.04 & 80.72$\pm$0.51 & 73.71$\pm$1.95 & 80.27$\pm$0.14 & 75.25$\pm$1.82 & 79.26$\pm$0.58 & 75.44$\pm$0.45
  \\
			\hline	
			\multirow{12}*{\rotatebox{90}{CF1}}& Standard  & 72.53$\pm$1.11  & 30.64$\pm$3.90 & 76.83$\pm$0.65 &  14.97$\pm$0.24 &  75.86$\pm$1.23 &  70.68$\pm$1.76 & 73.11$\pm$0.54  &  52.07$\pm$2.34  \\
			& GCE  & 73.10$\pm$1.27 & 33.07$\pm$4.65 & 77.25$\pm$0.66 & 14.77$\pm$0.20 & 76.73$\pm$1.57 & 71.24$\pm$1.42 & 73.37$\pm$0.98 & 56.89$\pm$2.84 \\	
			&CDR  & 73.08$\pm$0.47 & 33.06$\pm$1.84 & 76.95$\pm$0.79 & 14.88$\pm$0.17 & 76.09$\pm$1.42 & 70.78$\pm$1.09 & 73.33$\pm$0.85 & 54.16$\pm$3.19 \\		
			&AGCN  & 73.45$\pm$1.04 & 33.41$\pm$1.65 & 72.65$\pm$1.97 & 16.67$\pm$0.55 & 76.20$\pm$0.51 & 69.09$\pm$0.49 & 72.81$\pm$1.02 & 55.09$\pm$3.28 \\			
			&CSRA  & 74.10$\pm$0.56 & 33.44$\pm$3.65 & 75.28$\pm$1.32 & 15.71$\pm$0.23 & 77.52$\pm$0.94 & 73.44$\pm$0.62 & 74.98$\pm$0.48 & 58.60$\pm$2.24 \\	
			&WSIC & 70.13$\pm$2.04 & 13.64$\pm$6.97 & 75.32$\pm$2.00 & 14.77$\pm$0.16 & 74.17$\pm$1.87 & 64.95$\pm$6.47 & 65.41$\pm$4.40 & 43.43$\pm$11.31 \\	
			&Reweight-T max   & 74.05$\pm$0.51  & 39.37$\pm$1.36 & 77.28$\pm$0.47 &  50.35$\pm$5.52 & \secbest{77.20$\pm$0.39} &  \secbest{73.32$\pm$0.47} & 74.08$\pm$0.29  &  62.99$\pm$1.80  \\	
			&Reweight-T 97\% & 76.81$\pm$0.74  & \textbf{71.24$\pm$1.33} & 77.22$\pm$0.56 &  \secbest{67.65$\pm$1.37} &  74.99$\pm$0.37 &  68.07$\pm$0.74 & \textbf{77.62$\pm$0.42}  &   \textbf{74.56$\pm$0.26}  \\	
			&Reweight-DualT max & \secbest{75.27$\pm$0.56}  & 45.31$\pm$1.86 & 76.56$\pm$1.80 &  63.99$\pm$0.42 &  77.27$\pm$0.40 &  69.48$\pm$4.48 & 75.66$\pm$0.35  &  68.32$\pm$1.22  \\	
			&Reweight-DualT 97\%   & 71.93$\pm$1.64  & \secbest{63.85$\pm$3.69} &  \textbf{77.95$\pm$0.79}  &  \textbf{68.44$\pm$2.93} &  62.84$\pm$2.51 &  50.49$\pm$2.43 &  70.99$\pm$0.98  &  60.58$\pm$1.67  \\	
			\cline{2-10}
			& Reweight-Ours  & \textbf{76.86$\pm$0.48}  & 61.29$\pm$1.94 & \secbest{77.89$\pm$0.42} &  66.79$\pm$2.50 &  \textbf{78.04$\pm$0.40} &  \textbf{74.08$\pm$0.79} & \secbest{77.28$\pm$0.48}  &  \secbest{72.18$\pm$0.74}  \\	
   &Reweight-TrueT  & 76.22$\pm$0.63 & 62.91$\pm$1.25 & 77.92$\pm$0.99 & 68.51$\pm$2.29 & 77.60$\pm$0.45 & 68.89$\pm$3.06 & 76.39$\pm$0.77 & 71.54$\pm$0.86
  \\
			\hline		
		\end{tabular}
		\label{cls_VOC2007_c}
	\end{table*}
	\begin{table*}[!h]
		\caption{Complete numerical results for classification performance on Pascal-VOC2012 dataset with class-dependent label noise.}
		\centering
		\scriptsize
		\setlength\tabcolsep{5.8pt}
		\begin{tabular}{l|l|cc|cc|cc|cc}
			\hline	
			& Noise rates $(\rho_{-}, \rho_{+})$  &  (0,0.2) & (0,0.6) & (0.2,0) & (0.6,0)  & (0.1,0.1) & (0.2,0.2) &(0.017,0.2)& (0.034,0.4) \\	
			\hline
			\multirow{12}*{\rotatebox{90}{mAP}}& Standard  & {85.97$\pm$0.09}  & 80.02$\pm$0.62 & 85.70$\pm$0.19 &  76.13$\pm$0.86 &  85.91$\pm$0.10 &  82.54$\pm$0.51 & {86.03$\pm$0.24}  &  82.91$\pm$0.74  \\	
			& GCE  & 86.02$\pm$0.21 & 78.71$\pm$0.72 & 85.96$\pm$0.19 & 75.02$\pm$0.50 & \textbf{86.29$\pm$0.15} & 83.18$\pm$0.33 & 85.84$\pm$0.37 & 82.96$\pm$0.83 \\	
			&CDR  & 86.09$\pm$0.14 & 80.52$\pm$1.61 & 85.61$\pm$0.18 & 76.53$\pm$1.26 & 86.01$\pm$0.19 & 82.79$\pm$0.42 & 85.92$\pm$0.38 & 83.48$\pm$0.62 \\		
			&AGCN  & 85.16$\pm$0.12 & 79.67$\pm$0.43 & 84.91$\pm$0.38 & 77.54$\pm$0.29 & 84.69$\pm$0.30 & 80.15$\pm$0.49 & 84.75$\pm$0.12 & 81.85$\pm$0.35  \\		
			&CSRA  & \textbf{86.88$\pm$0.23} & \textbf{81.75$\pm$0.97} & 85.39$\pm$0.27 & 75.08$\pm$0.84 & 86.14$\pm$0.14 & 80.98$\pm$1.22 & \textbf{86.51$\pm$0.15} & \textbf{83.86$\pm$0.49}  \\	
			&WSIC & \secbest{86.39$\pm$0.38} & 80.75$\pm$0.49 & 85.53$\pm$0.28 & 77.00$\pm$1.03 & 85.67$\pm$0.17 & 82.01$\pm$0.87 & 86.07$\pm$0.22 & 83.19$\pm$0.19  \\	
			&Reweight-T max   & 85.40$\pm$0.43  & 79.06$\pm$1.69 & 85.80$\pm$0.32 &  78.59$\pm$0.69 &  85.98$\pm$0.32 &  82.88$\pm$0.41 & 85.51$\pm$0.51  &  82.38$\pm$1.67  \\	
			&Reweight-T 97\% & 85.97$\pm$0.27  & \secbest{81.04$\pm$0.79} & 85.81$\pm$0.21 &  80.12$\pm$0.75 &  85.66$\pm$0.55 & 82.81$\pm$0.52 & 85.99$\pm$0.51  &  83.28$\pm$1.07  \\	
			&Reweight-DualT max & 85.93$\pm$0.41  & 78.69$\pm$2.62 & \textbf{86.47$\pm$0.26} &  80.00$\pm$0.66 &  {86.23$\pm$0.34} &  \textbf{84.21$\pm$0.18} & \secbest{86.15$\pm$0.44}  &  {83.42$\pm$1.26}  \\	
			&Reweight-DualT 97\%   & 83.93$\pm$0.52  & 79.97$\pm$1.42 & \secbest{86.37$\pm$0.25} &  \textbf{81.84$\pm$0.87} &  82.39$\pm$0.85 &  78.61$\pm$1.01 & 83.33$\pm$1.40  &  80.95$\pm$1.15  \\		
			\cline{2-10}
			& Reweight-Ours  & {86.01$\pm$0.54}  & {80.33$\pm$1.85} & 86.12$\pm$0.10 &  \secbest{80.54$\pm$1.30} &  \secbest{86.23$\pm$0.28} &  \secbest{83.42$\pm$0.51} & 85.92$\pm$0.42  &  \secbest{83.56$\pm$1.31}  \\	
   &Reweight-TrueT  & 86.33$\pm$0.33 & 82.19$\pm$0.58 & 86.28$\pm$0.31 & 81.61$\pm$0.47 & 86.35$\pm$0.40 & 83.72$\pm$0.28 & 85.97$\pm$0.50 & 84.26$\pm$0.21
  \\
			\hline	
			\multirow{12}*{\rotatebox{90}{OF1}}& Standard  & 77.91$\pm$0.20  & 27.84$\pm$4.17 &  80.47$\pm$0.34 &  14.93$\pm$0.28 &  {81.27$\pm$0.22} &  78.51$\pm$0.13 & 77.83$\pm$0.36  &  61.67$\pm$2.10 \\	
			& GCE  & 78.25$\pm$0.28 & 24.67$\pm$3.13 & 80.89$\pm$0.19 & 14.73$\pm$0.28 & 81.47$\pm$0.34 & \textbf{78.78$\pm$0.11} & 78.32$\pm$0.37 & 63.26$\pm$1.40  \\	
			&CDR  & 78.02$\pm$0.23 & 29.45$\pm$6.39 & 80.65$\pm$0.36 & 14.91$\pm$0.26 & 81.34$\pm$0.35 & 78.58$\pm$0.14 & 77.82$\pm$0.34 & 61.94$\pm$1.91 \\		
			&AGCN  & 76.11$\pm$0.76 & 31.06$\pm$4.73 & 80.49$\pm$0.51 & 15.42$\pm$0.35 & 80.28$\pm$0.44 & 76.02$\pm$1.33 & 75.83$\pm$1.18 & 61.59$\pm$4.20  \\			
			&CSRA  & 78.61$\pm$0.56 & 32.36$\pm$6.27 & 79.88$\pm$0.57 & 15.71$\pm$0.50 & \textbf{81.82$\pm$0.35} & 78.12$\pm$1.11 & 79.05$\pm$0.59 & 61.77$\pm$3.46 \\	
			&WSIC & 78.56$\pm$0.78 & 27.90$\pm$6.72 & 80.56$\pm$0.32 & 15.09$\pm$0.24 & 80.98$\pm$0.49 & 77.36$\pm$1.51 & 78.52$\pm$0.65 & 57.56$\pm$4.11  \\	
			&Reweight-T max   & 77.76$\pm$0.72  & 43.96$\pm$6.52 & 81.24$\pm$0.31 &  48.88$\pm$2.45 &  \secbest{81.53$\pm$0.49} &  {78.63$\pm$0.17} & 78.66$\pm$0.22  &  68.04$\pm$2.70  \\	
			&Reweight-T 97\% & 79.31$\pm$0.20  & \textbf{71.79$\pm$2.35} & 81.56$\pm$0.29 &  75.10$\pm$1.81 &  77.81$\pm$1.09 &  71.28$\pm$1.99 & 79.20$\pm$0.35  &  73.90$\pm$1.81  \\
			&Reweight-DualT max & \secbest{80.44$\pm$0.69}  & 63.68$\pm$5.75 & \secbest{82.01$\pm$0.28} &  74.17$\pm$2.34 & 81.04$\pm$0.36 &  \secbest{78.71$\pm$0.83} & \secbest{80.47$\pm$0.35}  &  \secbest{75.56$\pm$1.72}  \\	
			&Reweight-DualT 97\%   & 60.81$\pm$1.36  & 58.42$\pm$2.41 & \textbf{82.20$\pm$0.20} &   \textbf{75.96$\pm$1.47} &  56.42$\pm$1.24 &  30.28$\pm$1.82 & 64.36$\pm$1.42  & 58.49$\pm$1.73  \\	
			\cline{2-10}
			& Reweight-Ours  & \textbf{80.54$\pm$0.61}  & \secbest{69.08$\pm$4.58} & 81.77$\pm$0.43 &  \secbest{75.38$\pm$3.05} &  81.11$\pm$0.56 & 77.78$\pm$0.37  &  \textbf{80.75$\pm$0.43}  &  \textbf{77.31$\pm$1.49}  \\
     &Reweight-TrueT  & 80.65$\pm$0.14 & 74.79$\pm$1.00 & 81.82$\pm$0.29 & 77.94$\pm$0.65 & 81.66$\pm$0.43 & 78.63$\pm$0.28 & 80.83$\pm$0.55 & 78.18$\pm$0.43
  \\
			\hline	
			\multirow{12}*{\rotatebox{90}{CF1}}& Standard  & 75.15$\pm$0.62  & 29.76$\pm$3.77 & 79.06$\pm$0.35 &  14.87$\pm$0.29 &  79.05$\pm$0.26 &  75.61$\pm$0.52 & 76.07$\pm$0.46  &  60.32$\pm$1.90  \\	
			& GCE  & 75.25$\pm$0.84 & 26.25$\pm$4.40 & 79.68$\pm$0.32 & 14.69$\pm$0.27 & 79.13$\pm$0.32 & 75.84$\pm$0.51 & 76.05$\pm$0.84 & 61.99$\pm$0.71 \\	
			&CDR  &75.32$\pm$0.66 & 31.49$\pm$4.03 & 79.20$\pm$0.20 & 14.85$\pm$0.25 & 79.14$\pm$0.45 & 75.56$\pm$0.49 & 76.11$\pm$0.44 & 60.02$\pm$1.91 \\		
			&AGCN  & 74.16$\pm$1.06 & 33.33$\pm$4.55 & 78.81$\pm$0.12 & 15.42$\pm$0.76 & 78.09$\pm$0.99 & 73.28$\pm$1.80 & 73.91$\pm$1.39 & 58.57$\pm$4.67  \\			
			&CSRA  & 75.91$\pm$0.98 & 32.83$\pm$4.88 & 78.86$\pm$0.25 & 15.48$\pm$0.49 & 79.87$\pm$0.26 & 75.53$\pm$1.12 & 76.40$\pm$0.73 & 59.53$\pm$3.12 \\	
			&WSIC & 76.46$\pm$1.39 & 30.19$\pm$5.09 & 79.49$\pm$0.57 & 14.98$\pm$0.23 & 78.50$\pm$0.91 & 74.10$\pm$2.39 & 76.05$\pm$0.90 & 55.11$\pm$4.30  \\	
			&Reweight-T max   & 75.54$\pm$0.59  & 39.66$\pm$5.13 & 79.79$\pm$0.27 &  49.60$\pm$3.31 &  \secbest{79.60$\pm$0.48} &  \textbf{76.45$\pm$0.29} & 76.82$\pm$0.42  &  65.53$\pm$2.66  \\	
			&Reweight-T 97\% & \textbf{78.94$\pm$0.10}  & \textbf{73.66$\pm$1.10} & 79.73$\pm$0.29 &  \textbf{73.38$\pm$1.41} &  77.54$\pm$0.88 &  72.39$\pm$1.57 & \secbest{78.78$\pm$0.35}  &  \textbf{75.54$\pm$1.13}  \\	
			&Reweight-DualT max & 78.37$\pm$0.54  & 58.96$\pm$3.99 & \secbest{80.06$\pm$0.22} &  72.61$\pm$1.76 &  79.29$\pm$0.35 &  \secbest{76.37$\pm$1.44} & 78.69$\pm$0.25  &  65.68$\pm$1.37  \\	
			&Reweight-DualT 97\%   & 67.96$\pm$1.13  & 61.89$\pm$2.49 & \textbf{80.11$\pm$0.33} &  71.87$\pm$3.50 &  65.09$\pm$1.07 &  52.28$\pm$1.03 & 69.26$\pm$0.82  &  63.10$\pm$1.57  \\	
			\cline{2-10}
			& Reweight-Ours  & \secbest{78.81$\pm$0.53}  & \secbest{66.76$\pm$3.38} & 79.82$\pm$0.43 & \secbest{72.84$\pm$2.16} &  \textbf{79.90$\pm$0.39} &  75.90$\pm$0.81 & \textbf{79.37$\pm$0.35}  &  \secbest{75.16$\pm$1.42}  \\	
     &Reweight-TrueT  & 78.95$\pm$0.28 & 71.82$\pm$1.45 & 79.95$\pm$0.53 & 73.87$\pm$1.65 & 79.69$\pm$0.58 & 75.82$\pm$0.78 & 79.12$\pm$0.40 & 75.99$\pm$0.37
  \\
			\hline		
		\end{tabular}
		\label{cls_VOC2012_c}
	\end{table*}
	\begin{table*}[!h]
		\caption{Complete numerical results for classification performance on MS-COCO dataset with class-dependent label noise.}
		\centering
		\scriptsize
		\setlength\tabcolsep{5.8pt}
		\begin{tabular}{l|l|cc|cc|cc|cc}
			\hline	
			& Noise rates $(\rho_{-}, \rho_{+})$  &  (0,0.2) & (0,0.6) & (0.2,0) & (0.6,0)  & (0.1,0.1) & (0.2,0.2) & (0.008,0.2)& (0.015,0.4) \\	
			\hline
			\multirow{12}*{\rotatebox{90}{mAP}}& Standard  & 69.92$\pm$0.06& {63.81$\pm$0.16} &66.77$\pm$0.52&55.45$\pm$0.48&67.77$\pm$0.28&62.50$\pm$0.23&69.76$\pm$0.09&66.82$\pm$0.05  \\	
			& GCE  & 69.90$\pm$0.05 & 62.58$\pm$0.17 & 67.32$\pm$0.11 & 54.01$\pm$0.70 & 68.62$\pm$0.16 & 63.21$\pm$0.35 & 69.99$\pm$0.11 & 66.72$\pm$0.19\\	
			&CDR  & 70.06$\pm$0.05 & 63.85$\pm$0.28 & 67.32$\pm$0.08 & 55.20$\pm$1.62 & 68.01$\pm$0.08 & 62.65$\pm$0.21 & 69.87$\pm$0.09 & 66.85$\pm$0.19\\		
			&AGCN  & \textbf{71.48$\pm$0.14} & \textbf{65.75$\pm$0.32} & \secbest{69.44$\pm$0.10} & 55.71$\pm$0.61 & \textbf{69.42$\pm$0.23} & \secbest{63.96$\pm$0.11} & \textbf{70.90$\pm$0.13} & \secbest{67.86$\pm$0.26} \\			
			&CSRA  & \secbest{71.18$\pm$0.10} & \secbest{65.28$\pm$0.11} & 67.93$\pm$0.18 & 51.49$\pm$0.73 & \secbest{68.83$\pm$0.12} & 61.80$\pm$0.98 & \secbest{70.76$\pm$0.16} & \textbf{68.02$\pm$0.15}\\	
			&WSIC & 68.92$\pm$0.09 & 63.09$\pm$0.28 & 66.22$\pm$0.06 & 53.61$\pm$0.36 & 67.41$\pm$0.15 & 62.33$\pm$0.18 & 68.95$\pm$0.15 & 66.29$\pm$0.21 \\	
			&Reweight-T max   & {69.99$\pm$0.18}  & {63.94$\pm$0.11} & 67.40$\pm$0.13 &  58.27$\pm$0.25 &  67.85$\pm$0.05 &  {63.28$\pm$0.12} & {69.76$\pm$0.07}  &  {66.24$\pm$0.51}  \\	
			&Reweight-T 97\% & 67.98$\pm$0.57&62.52$\pm$0.46&68.00$\pm$0.17&\secbest{59.44$\pm$0.81}&65.69$\pm$0.48&60.03$\pm$0.11&68.13$\pm$0.04&64.40$\pm$0.18  \\	
			&Reweight-DualT max & 67.57$\pm$0.21&60.39$\pm$0.53&68.57$\pm$0.25&58.42$\pm$0.82&{68.01$\pm$0.51}&62.17$\pm$0.32&68.76$\pm$0.08&65.75$\pm$0.15  \\	
			&Reweight-DualT 97\%   &64.97$\pm$0.20 & 58.85$\pm$0.43 & \textbf{69.68$\pm$0.25} & 49.17$\pm$3.02 & 56.36$\pm$0.62 & 49.41$\pm$0.38 & 63.27$\pm$0.36 & 58.21$\pm$0.86 \\	
			\cline{2-10}
			& Reweight-Ours  & {70.57$\pm$0.11} & 63.28$\pm$0.92 & {69.38$\pm$0.36} & \textbf{61.88$\pm$0.66} & {68.70$\pm$0.15} & \textbf{64.46$\pm$0.10} & {70.06$\pm$0.06} & {67.03$\pm$0.08} \\	
   & Reweight-TrueT & 70.78$\pm$0.17 & 65.50$\pm$0.42 & 69.45$\pm$0.07 & 62.40$\pm$0.12 & 69.49$\pm$0.06 & 65.27$\pm$0.36 & 70.46$\pm$0.06 & 68.03$\pm$0.12 \\
			\hline	
			\multirow{12}*{\rotatebox{90}{OF1}}& Standard  & 66.48$\pm$0.50 &19.18$\pm$0.97&69.58$\pm$0.38&7.05$\pm$0.01&{68.64$\pm$0.18}&{64.84$\pm$0.54}&66.07$\pm$0.15&51.70$\pm$0.42  \\	
			& GCE  & 66.67$\pm$0.38 & 19.61$\pm$1.61 & 69.82$\pm$0.25 & 7.03$\pm$0.02 & 69.44$\pm$0.18 & \secbest{64.98$\pm$0.79} & 66.58$\pm$0.32 & 52.04$\pm$0.51\\	
			&CDR  & 66.46$\pm$0.54 & 19.60$\pm$2.86 & 69.72$\pm$0.31 & 7.06$\pm$0.04 & 68.75$\pm$0.09 & 64.68$\pm$0.54 & 66.03$\pm$0.39 & 52.49$\pm$1.18 \\		
			&AGCN  & 67.14$\pm$0.52 & 16.02$\pm$0.76 & \textbf{70.63$\pm$0.12} & 7.04$\pm$0.02 & \secbest{69.66$\pm$0.25} & \textbf{66.07$\pm$0.55} & 66.61$\pm$0.48 & 52.38$\pm$1.05 \\			
			&CSRA  & 67.98$\pm$0.40 & 24.08$\pm$2.60 & 70.14$\pm$0.17 & 7.06$\pm$0.02 & \textbf{69.70$\pm$0.31} & 64.89$\pm$1.22 & \secbest{67.37$\pm$0.28} & 51.81$\pm$0.53  \\	
			&WSIC & 66.67$\pm$0.15 & 23.02$\pm$4.70 & 69.02$\pm$0.06 & 7.02$\pm$0.00 & 67.78$\pm$0.58 & 62.31$\pm$0.69 & 66.38$\pm$0.25 & 52.07$\pm$1.94 \\	
			&Reweight-T max   & \secbest{67.13$\pm$0.41}  & 39.47$\pm$1.47 & 69.84$\pm$0.11 &  59.26$\pm$1.76 &  64.03$\pm$2.00 &  57.68$\pm$4.00 &{66.45$\pm$0.30}  &  53.63$\pm$1.28  \\	
			&Reweight-T 97\% & 57.66$\pm$0.56&\secbest{54.06$\pm$1.19}&69.72$\pm$0.39&\secbest{64.79$\pm$0.85}&43.81$\pm$0.54&33.65$\pm$2.93&55.44$\pm$0.56&51.78$\pm$1.37  \\	
			&Reweight-DualT max & 65.22$\pm$0.28&55.01$\pm$0.86&70.16$\pm$0.26&56.90$\pm$4.55&59.62$\pm$1.59&49.79$\pm$0.18&65.64$\pm$0.71& \secbest{61.73$\pm$2.11} \\	
			&Reweight-DualT 97\%   & 29.39$\pm$0.36 & 29.30$\pm$0.74 & {70.24$\pm$0.26} & 48.87$\pm$6.87 & 25.83$\pm$0.16 & 8.51$\pm$0.26 & 39.29$\pm$0.45 & 35.78$\pm$1.10\\	
			\cline{2-10}
			& Reweight-Ours  & \textbf{70.10$\pm$0.10} & \textbf{61.74$\pm$0.64} & \secbest{70.52$\pm$0.26} & \textbf{65.78$\pm$0.56} & {64.45$\pm$0.47} & {58.60$\pm$3.30} & \textbf{69.40$\pm$0.31} & \textbf{65.93$\pm$0.42}\\	
      & Reweight-TrueT & 70.61$\pm$0.14 & 65.96$\pm$0.45 & 70.45$\pm$0.06 & 66.46$\pm$0.15 & 70.26$\pm$0.09 & 67.27$\pm$0.13 & 70.21$\pm$0.30 & 68.55$\pm$0.18 \\
			\hline	
			\multirow{12}*{\rotatebox{90}{CF1}}& Standard  & 60.27$\pm$0.52&22.73$\pm$0.15&64.66$\pm$0.82&7.07$\pm$0.02&62.38$\pm$0.27&56.78$\pm$1.31&60.04$\pm$0.17&45.35$\pm$0.60 \\	
			& GCE  & 60.76$\pm$0.08 & 21.06$\pm$1.12 & 65.27$\pm$0.22 & 7.04$\pm$0.03 & 63.60$\pm$0.25 & 56.72$\pm$1.56 & 60.66$\pm$0.19 & 44.28$\pm$1.12 \\	
			&CDR  & 60.26$\pm$0.55 & 22.42$\pm$1.78 & 65.27$\pm$0.10 & 7.07$\pm$0.04 & 62.63$\pm$0.08 & 56.41$\pm$1.36 & 59.85$\pm$0.38 & 45.41$\pm$0.85\\		
			&AGCN  & 61.79$\pm$0.98 & 19.30$\pm$1.20 & 66.29$\pm$0.11 & 7.05$\pm$0.04 & 64.09$\pm$0.39 & 58.97$\pm$1.15 & 60.35$\pm$0.54 & 43.89$\pm$1.46 \\			
			&CSRA  & 62.46$\pm$0.53 & 24.17$\pm$2.76 & 65.80$\pm$0.02 & 7.06$\pm$0.02 & 63.90$\pm$0.48 & 55.97$\pm$2.29 & 61.30$\pm$0.42 & 44.25$\pm$2.33 \\	
			&WSIC & 61.12$\pm$0.08 & 27.16$\pm$1.54 & 63.71$\pm$0.34 & 7.03$\pm$0.01 & 61.12$\pm$1.30 & 52.47$\pm$0.99 & 60.51$\pm$0.25 & 45.52$\pm$1.32 \\	
			&Reweight-T max   & 61.59$\pm$0.51  & 32.78$\pm$0.70 & 64.82$\pm$0.24 &  56.68$\pm$0.80 &  63.35$\pm$0.16 &  59.93$\pm$0.50 & 60.92$\pm$0.08  &  48.40$\pm$0.33  \\	
			&Reweight-T 97\% & 55.67$\pm$0.45&\secbest{52.79$\pm$1.04} &63.97$\pm$0.74&\secbest{56.70$\pm$2.11}&48.47$\pm$0.70&40.20$\pm$0.82&55.33$\pm$0.12&52.19$\pm$0.29 \\	
			&Reweight-DualT max & \secbest{64.79$\pm$0.23} &52.16$\pm$2.08&63.51$\pm$0.35&54.62$\pm$1.12& \secbest{66.29$\pm$0.34} &\textbf{61.33$\pm$0.09}&\secbest{65.18$\pm$0.22}&\secbest{60.88$\pm$0.59}  \\	
			&Reweight-DualT 97\%   &  32.15$\pm$0.43 & 30.32$\pm$0.91 & \secbest{65.23$\pm$0.28} & 33.76$\pm$8.88 & 29.99$\pm$0.10 & 12.77$\pm$0.21 & 41.29$\pm$0.83 & 37.08$\pm$1.29\\	
			\cline{2-10}
			& Reweight-Ours  & \textbf{67.18$\pm$0.17} & \textbf{57.46$\pm$0.52} & \textbf{65.42$\pm$0.49} & \textbf{58.63$\pm$1.30} & \textbf{66.65$\pm$0.15} & \secbest{61.13$\pm$1.01} & \textbf{66.42$\pm$0.10} & \textbf{62.94$\pm$0.28} \\	
   & Reweight-TrueT & 65.67$\pm$0.15 & 59.74$\pm$0.37 & 65.43$\pm$0.20 & 58.78$\pm$0.05 & 65.10$\pm$0.02 & 60.25$\pm$0.50 & 65.02$\pm$0.55 & 63.09$\pm$0.62 \\
			\hline		
		\end{tabular}
		\label{cls_MSCOCO_c}
	\end{table*}

	\begin{table*}[h]
		\caption{Complete numerical results for classification performance on Pascal-VOC2007 dataset with instance-dependent label noise. }
		\centering
		\scriptsize
		\setlength\tabcolsep{3.8pt}
		\begin{tabular}{l|l|ccc|ccc}
			\hline	
			& Noise type  & Pair-wise 10\% &	Pair-wise 20\% &	Pair-wise 30\% &	PMD-Type-I	& PMD-Type-II	&  PMD-Type-III  \\		
			\hline	
			\multirow{11}*{\rotatebox{90}{mAP}}& Standard &	84.76$\pm$0.29 & 82.64$\pm$0.28 & 79.94$\pm$0.35 & 77.99$\pm$0.63 & \secbest{83.11$\pm$0.18} & \secbest{82.69$\pm$0.38}
		 \\		
	 & GCE &	\secbest{84.92$\pm$0.22} & \best{83.16$\pm$0.27} & 79.81$\pm$0.93 & 77.30$\pm$0.67 & 83.02$\pm$0.33 & 82.59$\pm$0.14
	   \\
	 & CDR &84.69$\pm$0.23 & 82.69$\pm$0.28 & 79.94$\pm$0.42 & 77.91$\pm$0.91 & 83.01$\pm$0.23 & 82.61$\pm$0.34
	   \\
	 & AGCN &	83.09$\pm$0.51 & 80.27$\pm$0.30 & 74.80$\pm$1.07 & 75.90$\pm$0.52 & 81.48$\pm$0.20 & 81.69$\pm$0.38 \\
   	& CSRA &	\best{85.47$\pm$0.59} & \secbest{82.90$\pm$0.79} & \best{80.34$\pm$1.26} & \best{79.63$\pm$0.15} & \best{83.56$\pm$0.56} & \best{83.62$\pm$0.47} \\
   & WSIC & 84.42$\pm$0.30 & 82.26$\pm$0.30 & 79.16$\pm$0.46 & 77.81$\pm$0.56 & 82.65$\pm$0.34 & 82.51$\pm$0.24 \\	
			& Reweight-T max & 84.47$\pm$0.31 & 82.02$\pm$0.46 & 79.57$\pm$0.69 & 78.02$\pm$0.56 & 82.14$\pm$0.38 & 81.68$\pm$0.70
 \\		
			& Reweight-T 97\% & 83.98$\pm$0.44 & 80.82$\pm$0.86 & 79.49$\pm$0.48 & 78.01$\pm$0.71 & 82.30$\pm$0.58 & 81.36$\pm$0.38
 \\	
			& Reweight-DualT max & 84.60$\pm$0.34 & 82.37$\pm$0.38 & \secbest{80.20$\pm$0.60} & 78.20$\pm$0.75 & 82.29$\pm$0.81 & 81.83$\pm$0.66
 \\	
			& Reweight-DualT 97\% & 83.48$\pm$0.56 & 81.60$\pm$0.51 & 73.85$\pm$2.04 & 75.03$\pm$1.89 & 79.91$\pm$0.45 & 80.35$\pm$0.35
 \\
			\cline{2-8}	
			& Reweight-Ours &  84.45$\pm$0.38 & 82.29$\pm$0.44 & 79.11$\pm$1.07 & \secbest{78.72$\pm$0.61} & 82.35$\pm$0.65 & 81.77$\pm$0.60
 \\	
			\hline	
			\multirow{11}*{\rotatebox{90}{OF1}}& Standard &	80.04$\pm$0.27 & 78.13$\pm$0.34 & 75.29$\pm$0.47 & 58.08$\pm$1.93 & 76.97$\pm$0.76 & 76.12$\pm$0.24  \\		
		& GCE &	 79.96$\pm$0.29 & 78.62$\pm$0.31 & 74.90$\pm$0.98 & 56.90$\pm$3.62 & 77.10$\pm$0.77 & 76.24$\pm$0.45   \\
		& CDR &	 79.86$\pm$0.41 & 78.10$\pm$0.37 & 75.28$\pm$0.60 & 59.04$\pm$1.39 & 76.93$\pm$0.85 & 76.33$\pm$0.41 
		 \\
		& AGCN & 78.85$\pm$0.74 & 76.30$\pm$0.37 & 71.64$\pm$0.86 & 56.02$\pm$3.92 & 76.49$\pm$0.39 & 76.06$\pm$0.75  \\
		& CSRA & \textbf{81.21$\pm$0.74} & \secbest{79.31$\pm$0.56} & \secbest{76.17$\pm$0.89} & 59.08$\pm$1.98 & \textbf{78.45$\pm$0.51} & \best{77.67$\pm$0.29}  \\
   & WSIC & 78.03$\pm$1.77 & 77.25$\pm$0.25 & 73.19$\pm$2.50 & 55.89$\pm$1.54 & 72.04$\pm$1.16 & 74.86$\pm$1.56 \\		
			& Reweight-T max & 80.04$\pm$0.23 & 78.54$\pm$0.35 & 75.68$\pm$0.65 & 63.56$\pm$1.33 & 77.27$\pm$0.35 & 76.17$\pm$0.51 \\	
			& Reweight-T 97\% & 79.21$\pm$0.72 & 76.78$\pm$0.73 & 74.35$\pm$0.41 & \secbest{72.19$\pm$0.49} & 76.81$\pm$0.48 & 76.27$\pm$0.21 \\	
			& Reweight-DualT max &80.26$\pm$0.26 & 79.11$\pm$0.45 & \best{76.68$\pm$0.44} & 69.53$\pm$1.07 & 77.49$\pm$0.32 & 76.29$\pm$0.38
  \\	
			& Reweight-DualT 97\% &70.59$\pm$0.19 & 59.64$\pm$0.37 & 57.40$\pm$0.18 & 60.17$\pm$1.18 & 67.20$\pm$0.32 & 70.03$\pm$2.10 \\
			\cline{2-8}	
			& Reweight-Ours & \secbest{80.45$\pm$0.13} & \best{79.42$\pm$0.45} & 74.34$\pm$0.35 & \best{73.66$\pm$0.73} & \secbest{77.62$\pm$0.40} & \secbest{77.18$\pm$0.40}
 \\	
			\hline	
			\multirow{11}*{\rotatebox{90}{CF1}} & Standard &77.21$\pm$0.38 & 73.96$\pm$0.35 & 69.46$\pm$0.69 & 50.84$\pm$3.44 & 73.70$\pm$0.95 & 72.52$\pm$0.50  \\		
		& GCE &	 77.37$\pm$0.13 & 74.48$\pm$0.21 & 69.56$\pm$0.78 & 51.05$\pm$3.00 & 73.68$\pm$1.10 & 72.71$\pm$0.89 
		 \\
		& CDR &	 77.08$\pm$0.36 & 73.76$\pm$0.49 & 69.49$\pm$0.68 & 52.67$\pm$2.00 & 73.53$\pm$1.05 & 72.87$\pm$0.61  \\
		& AGCN & 76.71$\pm$0.67 & 72.50$\pm$0.67 & 66.35$\pm$1.42 & 53.73$\pm$2.90 & 73.15$\pm$0.42 & 73.12$\pm$0.91 \\
		& CSRA & \best{78.98$\pm$0.88} & 75.93$\pm$0.55 & 71.39$\pm$1.16 & 56.91$\pm$1.17 & 75.78$\pm$0.66 & \best{75.10$\pm$0.18}  \\
   & WSIC & 73.57$\pm$3.03 & 72.08$\pm$0.42 & 65.45$\pm$4.66 & 48.03$\pm$2.78 & 63.49$\pm$2.59 & 70.06$\pm$2.83 \\		
			& Reweight-T max & 77.62$\pm$0.21 & 74.75$\pm$0.65 & 71.00$\pm$0.78 & 60.28$\pm$1.39 & 74.39$\pm$0.52 & 73.43$\pm$0.40\\	
			& Reweight-T 97\% & 77.53$\pm$0.51 & 74.14$\pm$0.87 & 71.10$\pm$0.61 & \secbest{72.27$\pm$0.57} & \secbest{75.81$\pm$0.64} & 74.76$\pm$0.31
 \\	
			& Reweight-DualT max &  77.65$\pm$0.32 & 75.66$\pm$0.48 & \best{72.50$\pm$0.77} & 65.85$\pm$1.54 & 74.85$\pm$0.75 & 73.45$\pm$0.46 \\	
			& Reweight-DualT 97\% & 73.90$\pm$0.14 & 65.79$\pm$0.10 &60.23$\pm$0.18  & 62.28$\pm$1.36 & 71.22$\pm$1.00 & 71.50$\pm$1.02  \\
			\cline{2-8}	
			& Reweight-Ours & \secbest{78.32$\pm$0.21} & \best{76.57$\pm$0.51} & \secbest{71.42$\pm$0.62} & \best{72.41$\pm$0.62} & \best{76.23$\pm$0.63} & \secbest{75.05$\pm$0.38}
			\\	
			\hline		
		\end{tabular}
		\label{cls_VOC2007_ins_c}
	\end{table*}
	\begin{table*}[!h]
		\caption{Complete numerical results for classification performance on Pascal-VOC2012 dataset with instance-dependent label noise.}
		\centering
		\scriptsize
		\setlength\tabcolsep{3.8pt}
		\begin{tabular}{l|l|ccc|ccc}
			\hline	
			& Noise type  & Pair-wise 10\% &	Pair-wise 20\% &	Pair-wise 30\% &	PMD-Type-I	& PMD-Type-II	&  PMD-Type-III  \\		
			\hline	
			\multirow{11}*{\rotatebox{90}{mAP}}& Standard &	86.19$\pm$0.49 & 84.21$\pm$0.27 & 81.46$\pm$0.33 & 80.11$\pm$0.62 & 84.42$\pm$0.36 & 84.53$\pm$0.24
 \\
   & GCE & \best{86.70$\pm$0.33} & 84.67$\pm$0.23 & 82.14$\pm$0.37 & 79.98$\pm$0.58 & 84.59$\pm$0.29 & 84.44$\pm$0.29
 \\	
   & CDR & 86.38$\pm$0.30 & 84.41$\pm$0.40 & 81.56$\pm$0.38 & 80.36$\pm$0.60 & 84.52$\pm$0.30 & 84.53$\pm$0.20 \\	
   & AGCN & 85.28$\pm$0.25 & 82.75$\pm$0.29 & 78.88$\pm$1.18 & 77.25$\pm$1.57 & 84.10$\pm$0.46 & 83.96$\pm$0.22 \\	
   & CSRA & \secbest{86.45$\pm$0.69} & 84.60$\pm$0.19 & 80.95$\pm$0.78 & 80.30$\pm$0.69 & \best{85.45$\pm$0.22} & \best{85.10$\pm$0.28}
 \\	
   & WSIC & 86.41$\pm$0.21 & 84.47$\pm$0.56 & \best{82.41$\pm$0.64} & 80.20$\pm$0.59 & \secbest{84.70$\pm$0.15} & \secbest{84.64$\pm$0.37}\\	
			& Reweight-T max & 85.79$\pm$0.10 & 83.98$\pm$0.30 & 81.09$\pm$0.67 & 79.84$\pm$0.80 & 84.34$\pm$0.34 & 83.41$\pm$0.42 \\		
			& Reweight-T 97\% & 85.86$\pm$0.20 & 83.92$\pm$0.54 & 81.33$\pm$0.28 & \best{81.22$\pm$0.43} & 84.34$\pm$0.26 & 83.22$\pm$0.36 \\	
			& Reweight-DualT max & 86.36$\pm$0.18 & \best{84.79$\pm$0.33} & 82.11$\pm$0.50 & 80.75$\pm$0.41 & 84.49$\pm$0.35 & 83.52$\pm$0.38 \\	
			& Reweight-DualT 97\% & 83.90$\pm$0.80 & 79.15$\pm$1.74 & 75.36$\pm$0.91 & 76.93$\pm$1.11 & 82.22$\pm$0.42 & 81.64$\pm$0.51
 \\
			\cline{2-8}	
			& Reweight-Ours & 86.30$\pm$0.20 & \secbest{84.75$\pm$0.32} & \secbest{82.30$\pm$0.52} & \secbest{80.87$\pm$0.70} & 84.50$\pm$0.36 & 83.65$\pm$0.45
 \\	
			\hline	
			\multirow{11}*{\rotatebox{90}{OF1}}& Standard &	81.27$\pm$0.23 & 79.85$\pm$0.27 & 77.62$\pm$0.57 & 60.01$\pm$2.13 & 78.02$\pm$0.53 & 77.05$\pm$0.39
 \\
   & GCE & 81.49$\pm$0.34 & 79.81$\pm$0.20 & 77.96$\pm$0.39 & 58.80$\pm$1.92 & 77.90$\pm$0.43 & 76.88$\pm$0.32 \\	
   & CDR & 81.36$\pm$0.32 & 79.90$\pm$0.19 & 77.68$\pm$0.56 & 59.06$\pm$2.49 & 77.98$\pm$0.38 & 76.92$\pm$0.57\\	
   & AGCN & 80.79$\pm$0.44 & 78.79$\pm$0.61 & 75.74$\pm$0.90 & 56.91$\pm$4.16 & 77.76$\pm$0.63 & 76.76$\pm$1.06
 \\	
   & CSRA & \secbest{82.01$\pm$0.90} & 80.77$\pm$0.40 & 77.66$\pm$0.58 & 58.06$\pm$4.14 & \best{78.62$\pm$0.41} & \secbest{77.51$\pm$0.36}
 \\	
   & WSIC & 81.36$\pm$0.35 & 79.95$\pm$0.63 & 77.78$\pm$0.21 & 52.39$\pm$5.61 & 77.49$\pm$0.83 & 76.48$\pm$0.89 \\		
			& Reweight-T max & 81.47$\pm$0.33 & 80.39$\pm$0.38 & \secbest{78.01$\pm$0.35} & 66.39$\pm$0.76 & \secbest{78.18$\pm$0.53} & 76.93$\pm$0.46 \\	
			& Reweight-T 97\% & 79.18$\pm$0.74 & 77.33$\pm$0.86 & 74.90$\pm$0.37 & 72.42$\pm$1.02 & 77.34$\pm$0.43 & 76.30$\pm$0.68\\	
			& Reweight-DualT max &81.95$\pm$0.20 & \best{81.45$\pm$0.15} & \best{79.08$\pm$0.46} & \secbest{73.13$\pm$0.75} & 78.16$\pm$0.58 & 77.07$\pm$0.40
  \\	
			& Reweight-DualT 97\% & 65.02$\pm$1.97 & 60.00$\pm$1.68 & 56.34$\pm$0.79 & 56.86$\pm$1.49 & 64.86$\pm$1.78 & 65.03$\pm$0.91
 \\
			\cline{2-8}	
			& Reweight-Ours & \best{82.03$\pm$0.28} & \secbest{81.27$\pm$0.13} & 77.08$\pm$0.57 & \best{73.42$\pm$0.78} & 77.86$\pm$0.56 & \best{77.66$\pm$0.40}
 \\	
			\hline	
			\multirow{11}*{\rotatebox{90}{CF1}} & Standard &	78.89$\pm$0.60 & 76.29$\pm$0.45 & 72.50$\pm$0.74 & 56.57$\pm$4.64 & 76.00$\pm$0.80 & 75.51$\pm$0.52
\\
   & GCE & 79.15$\pm$0.72 & 76.08$\pm$0.55 & 72.92$\pm$0.31 & 55.47$\pm$5.42 & 75.95$\pm$0.85 & 74.85$\pm$0.56
 \\	
   & CDR & 79.15$\pm$0.39 & 76.36$\pm$0.36 & 72.44$\pm$0.60 & 56.40$\pm$5.18 & 76.02$\pm$0.64 & 75.66$\pm$0.69
 \\	
   & AGCN & 79.31$\pm$0.40 & 75.22$\pm$0.51 & 70.26$\pm$1.08 & 53.36$\pm$3.38 & 76.28$\pm$0.30 & 75.37$\pm$1.16 \\	
   & CSRA & \secbest{80.25$\pm$0.79} & \secbest{77.37$\pm$0.63} & 72.59$\pm$0.68 & 57.37$\pm$3.93 & 77.06$\pm$0.27 & 76.06$\pm$0.46 \\	
   & WSIC & 78.97$\pm$0.64 & 76.41$\pm$0.73 & 72.43$\pm$0.53 & 49.99$\pm$6.30 & 75.82$\pm$0.46 & 74.45$\pm$1.13 \\
			& Reweight-T max & 79.64$\pm$0.27 & 77.03$\pm$0.70 & 73.06$\pm$0.89 & 63.28$\pm$1.59 & 76.68$\pm$0.34 & 75.45$\pm$0.75 \\	
			& Reweight-T 97\% & 78.89$\pm$0.34 & 75.60$\pm$0.94 & 71.68$\pm$0.38 & \secbest{73.77$\pm$0.46} & \secbest{77.43$\pm$0.45} & \secbest{76.82$\pm$0.51}\\	
			& Reweight-DualT max & 78.49$\pm$0.39 & 76.64$\pm$0.18 & \best{75.57$\pm$0.88} & 65.35$\pm$2.36 & 74.90$\pm$0.65 & 73.86$\pm$0.67 \\	
			& Reweight-DualT 97\% & 71.38$\pm$1.03 & 63.15$\pm$1.99 & 58.79$\pm$0.53 & 61.46$\pm$0.91 & 71.80$\pm$0.62 & 71.86$\pm$0.56 \\
			\cline{2-8}	
			& Reweight-Ours & \best{80.48$\pm$0.32} & \best{79.01$\pm$0.32} & \secbest{73.58$\pm$0.68} & \best{74.12$\pm$0.62} & \best{77.57$\pm$0.49} & \best{77.02$\pm$0.60}
			\\	
			\hline		
		\end{tabular}
		\label{cls_VOC2012_ins_c}
	\end{table*}
	\begin{table*}[!h]
		\caption{Complete numerical results for classification performance on MS-COCO dataset with instance-dependent label noise.}
		\centering
		\scriptsize
		\setlength\tabcolsep{3.8pt}
		\begin{tabular}{l|l|ccc|ccc}
			\hline	
			& Noise type  & Pair-wise 10\% &	Pair-wise 20\% &	Pair-wise 30\% &	PMD-Type-I	& PMD-Type-II	&  PMD-Type-III  \\		
			\hline	
			\multirow{11}*{\rotatebox{90}{mAP}}& Standard &	70.50$\pm$0.11 & 67.78$\pm$0.07 & 65.45$\pm$0.12 & 61.02$\pm$0.19 & 64.90$\pm$0.16 & 64.66$\pm$0.20
		\\		
		& GCE &	 70.99$\pm$0.16 & 68.40$\pm$0.17 & 66.24$\pm$0.13 & 60.77$\pm$0.27 & 65.06$\pm$0.18 & 64.67$\pm$0.10
		 \\
		& CDR &	70.52$\pm$0.04 & 67.99$\pm$0.16 & 65.40$\pm$0.22 & 61.00$\pm$0.12 & 65.02$\pm$0.27 & 64.82$\pm$0.11 \\
		& AGCN & \best{71.97$\pm$0.11} & \best{69.54$\pm$0.17} & \best{67.01$\pm$0.14} & \secbest{62.12$\pm$0.11} & \secbest{65.46$\pm$0.22} & \secbest{65.19$\pm$0.10}
		 \\
		& CSRA & \secbest{71.62$\pm$0.29} & \secbest{68.76$\pm$0.10} & \secbest{66.03$\pm$0.09} & 61.82$\pm$0.21 & \textbf{65.49$\pm$0.20} & \textbf{64.99$\pm$0.06} \\	
   & WSIC & 69.22$\pm$0.08 & 66.85$\pm$0.09 & 64.46$\pm$0.15 & 61.14$\pm$0.17 & 64.75$\pm$0.01 & 64.38$\pm$0.12 \\	
			& Reweight-T max & 70.42$\pm$0.12 & 67.85$\pm$0.11 & 65.58$\pm$0.17 & 61.29$\pm$0.21 & 64.95$\pm$0.02 & 64.54$\pm$0.28
 \\		
			& Reweight-T 97\% & 68.31$\pm$0.05 & 64.93$\pm$0.22 & 61.83$\pm$0.53 & 59.28$\pm$0.16 & 62.97$\pm$0.28 & 62.50$\pm$0.10
 \\	
			& Reweight-DualT max & 68.24$\pm$0.24 & 65.40$\pm$0.57 & 62.88$\pm$0.77 & 60.17$\pm$0.33 & 63.02$\pm$0.09 & 62.05$\pm$0.26 \\	
			& Reweight-DualT 97\% & 64.13$\pm$0.11 & 59.83$\pm$0.11 & 56.70$\pm$0.77 & 53.37$\pm$0.51 & 58.28$\pm$0.15 & 57.95$\pm$0.31
\\
			\cline{2-8}	
			& Reweight-Ours & 70.80$\pm$0.13 & 68.34$\pm$0.19 & \secbest{66.59$\pm$0.06} & \best{62.25$\pm$0.36} & 64.97$\pm$0.06 & 64.51$\pm$0.16 \\	
			\hline	
			\multirow{11}*{\rotatebox{90}{OF1}}& Standard &	70.35$\pm$0.17 & 67.86$\pm$0.04 & 65.53$\pm$0.26 & 48.45$\pm$2.43 & 61.39$\pm$0.44 & 59.30$\pm$0.40  \\		
		& GCE &	70.91$\pm$0.14 & 68.30$\pm$0.05 & \secbest{66.22$\pm$0.17} & 46.62$\pm$0.85 & 61.70$\pm$0.55 & 59.14$\pm$0.31  \\
		& CDR &	70.43$\pm$0.27 & 68.09$\pm$0.18 & 65.66$\pm$0.52 & 48.25$\pm$1.76 & 61.46$\pm$0.73 & 59.35$\pm$0.28   \\
		& AGCN &\secbest{71.07$\pm$0.23} & \secbest{68.75$\pm$0.27} & 66.04$\pm$0.17 & 42.82$\pm$1.86 & 61.74$\pm$0.45 & 59.63$\pm$0.78 \\
		& CSRA &	70.73$\pm$0.33 & 68.32$\pm$0.25 & 65.98$\pm$0.14 & 47.15$\pm$1.43 & 61.97$\pm$0.61 & 59.10$\pm$0.59  \\
   & WSIC & 69.40$\pm$0.16 & 67.66$\pm$0.15 & 65.47$\pm$0.43 & 50.92$\pm$0.98 & \secbest{62.72$\pm$0.42} & \secbest{61.31$\pm$0.30}
 \\		
			& Reweight-T max & 70.34$\pm$0.53 & 68.54$\pm$0.34 & 66.16$\pm$0.44 & 56.60$\pm$0.29 & 62.16$\pm$0.12 & 59.80$\pm$0.62
 \\	
			& Reweight-T 97\% & 56.52$\pm$0.32 & 54.00$\pm$0.68 & 54.13$\pm$0.95 & 51.43$\pm$1.34 & 52.35$\pm$0.04 & 51.93$\pm$0.74
 \\	
			& Reweight-DualT max &\secbest{67.18$\pm$0.63} & 64.80$\pm$0.80 & 62.88$\pm$1.34 & 61.06$\pm$1.06 & 58.86$\pm$0.60 & 59.63$\pm$1.04
  \\	
			& Reweight-DualT 97\% & 35.60$\pm$0.29 & 33.68$\pm$0.42 & 34.19$\pm$0.62 & 31.45$\pm$0.88 & 33.84$\pm$1.24 & 33.70$\pm$0.90
\\
			\cline{2-8}	
			& Reweight-Ours & \best{71.15$\pm$0.13} & \best{69.68$\pm$0.21} & \best{66.54$\pm$0.46} & \best{63.75$\pm$0.12} & \best{64.94$\pm$0.19} & \best{63.60$\pm$0.40} \\	
			\hline	
			\multirow{11}*{\rotatebox{90}{CF1}} & Standard &64.81$\pm$0.26 & 61.41$\pm$0.27 & 57.70$\pm$0.58 & 42.38$\pm$1.28 & 54.53$\pm$0.33 & 52.34$\pm$0.16  \\		
		& GCE &	 65.65$\pm$0.19 & 61.86$\pm$0.26 & 58.51$\pm$0.31 & 40.51$\pm$0.39 & 55.08$\pm$0.80 & 52.37$\pm$0.32  \\
		& CDR &	 64.91$\pm$0.38 & 61.58$\pm$0.36 & 57.65$\pm$0.49 & 42.83$\pm$1.22 & 54.72$\pm$0.77 & 52.30$\pm$0.36  \\
		& AGCN &	\secbest{66.92$\pm$0.48} & 63.41$\pm$0.50 & 58.40$\pm$0.37 & 39.44$\pm$0.31 & 55.00$\pm$0.13 & 53.12$\pm$0.29  \\
		& CSRA &66.38$\pm$0.40 & 62.87$\pm$0.53 & 58.41$\pm$0.43 & 41.52$\pm$1.18 & 55.39$\pm$0.98 & 52.63$\pm$0.70 \\
   & WSIC & 64.08$\pm$0.24 & 61.78$\pm$0.40 & 58.39$\pm$0.73 & 46.75$\pm$0.51 & 56.80$\pm$0.28 & 55.51$\pm$0.26
 \\		
			& Reweight-T max & 65.09$\pm$0.63 & 62.46$\pm$0.52 & 58.65$\pm$0.49 & 47.58$\pm$0.37 & 55.93$\pm$0.30 & 53.30$\pm$0.33
 \\	
			& Reweight-T 97\% & 54.88$\pm$0.52 & 51.90$\pm$0.56 & 53.18$\pm$1.04 & 48.72$\pm$0.86 & 51.04$\pm$0.13 & 51.18$\pm$0.13 \\	
			& Reweight-DualT max &  66.37$\pm$0.36 & \secbest{64.28$\pm$0.63} & \secbest{60.67$\pm$0.86} & \secbest{59.75$\pm$0.21} & \secbest{58.78$\pm$0.40} & \secbest{58.05$\pm$0.60} \\	
			& Reweight-DualT 97\% & 38.34$\pm$0.39 & 36.92$\pm$0.20 & 43.24$\pm$0.44 & 31.85$\pm$0.55 & 36.89$\pm$0.48 & 36.73$\pm$0.23  \\
			\cline{2-8}	
			& Reweight-Ours & \best{68.24$\pm$0.39} & \best{66.68$\pm$0.35} & \best{63.50$\pm$0.21} & \best{60.93$\pm$0.48} & \best{61.16$\pm$0.15} & \best{60.44$\pm$0.36}
			\\	
			\hline		
		\end{tabular}
		\label{cls_MSCOCO_ins_c}
	\end{table*}

\begin{table*}[h]
		\caption{Complete numerical results for classification performance with different loss correction ways on Pascal-VOC2007 dataset with class-dependent label noise. The best performances are in \textbf{bold}.}
		\centering
		\scriptsize
		\setlength\tabcolsep{5.8pt}
		\begin{tabular}{l|l|cc|cc|cc|cc}
			\hline	
			& Noise rates $(\rho_{-}, \rho_{+})$  &  (0,0.2) & (0,0.6) & (0.2,0) & (0.6,0)  & (0.1,0.1) & (0.2,0.2) &(0.017,0.2)& (0.034,0.4) \\		
			\hline	
			\multirow{3}*{\rotatebox{90}{mAP}}& Reweight-Ours & {84.43$\pm$0.46}  & {78.72$\pm$0.41} & {84.08$\pm$0.24} &  \textbf{74.46$\pm$0.56} &  {84.03$\pm$0.29} &  {80.44$\pm$0.52} & 84.09$\pm$0.62  &  {80.97$\pm$1.03}   \\		
			& Backward-Ours &83.41$\pm$0.18 & 67.22$\pm$2.97 & 81.13$\pm$0.45 & 63.82$\pm$1.36 & 79.00$\pm$0.88 & 70.44$\pm$1.80 & 81.27$\pm$0.37 & 69.22$\pm$1.96\\	
			& Forward-Ours & {84.96$\pm$0.28} & \textbf{80.23$\pm$0.25} & 83.41$\pm$0.72 & 71.45$\pm$3.65 & 83.19$\pm$0.24 & 76.77$\pm$2.14 & {84.31$\pm$0.51} & 80.46$\pm$1.02  \\	
			& Revision-Ours & \textbf{84.99$\pm$0.33} & 79.26$\pm$0.51 & \textbf{84.37$\pm$0.24}  & 74.44$\pm$1.16 & \textbf{84.49$\pm$0.36} & \textbf{80.61$\pm$0.73} & \textbf{84.98$\pm$0.11} & \textbf{82.14$\pm$0.10}\\
			\hline	
			\multirow{3}*{\rotatebox{90}{OF1}}& Reweight-Ours  & {78.62$\pm$0.58}  & 65.68$\pm$1.67 & {80.85$\pm$0.25} &  {67.43$\pm$4.65} &  79.64$\pm$0.29 &  {75.52$\pm$0.86} & {79.25$\pm$0.52}  &  {74.35$\pm$1.65} \\	
			& Backward-Ours & 74.98$\pm$0.56 & 16.35$\pm$5.06 & 78.55$\pm$0.33 & 14.57$\pm$0.29 & 74.35$\pm$1.12 & 63.60$\pm$4.39 & 71.68$\pm$0.78 & 36.38$\pm$4.13 \\	
			& Forward-Ours  & \textbf{80.09$\pm$0.45} & \textbf{70.38$\pm$0.29} & 80.44$\pm$0.50 & 59.35$\pm$10.36 & {79.77$\pm$0.35} & 73.92$\pm$2.33 & \textbf{79.92$\pm$0.62} & {75.18$\pm$1.24} \\
			& Revision-Ours &79.07$\pm$0.94 & 63.97$\pm$3.14 & \textbf{81.20$\pm$0.32} & \textbf{68.74$\pm$4.18} & \textbf{80.22$\pm$0.40} & \textbf{75.68$\pm$1.16} & 79.83$\pm$0.33 & \textbf{75.23$\pm$0.56}\\	
			\hline	
			\multirow{3}*{\rotatebox{90}{CF1}}& Reweight-Ours & {76.86$\pm$0.48}  & 61.29$\pm$1.94 & {77.89$\pm$0.42} &  {66.79$\pm$2.50} &  {78.04$\pm$0.40} &  \textbf{74.08$\pm$0.79} & {77.28$\pm$0.48}  & \textbf{72.18$\pm$0.74} \\	
			& Backward-Ours& 70.64$\pm$0.59 & 15.01$\pm$4.66 & 75.48$\pm$0.84 & 14.55$\pm$0.26 & 66.26$\pm$2.15 & 46.19$\pm$8.35 & 65.56$\pm$1.59 & 23.78$\pm$2.95\\	
			& Forward-Ours  & \textbf{77.85$\pm$0.56} & \textbf{64.64$\pm$1.69} & 77.14$\pm$0.35 & 62.26$\pm$3.06 & 77.72$\pm$0.25 & 70.45$\pm$4.38 & {77.74$\pm$0.80} & 71.76$\pm$1.45  \\
			& Revision-Ours &  76.82$\pm$0.90 & 58.56$\pm$3.72 & \textbf{78.26$\pm$0.51} & \textbf{66.89$\pm$1.46} & \textbf{78.55$\pm$0.58} & 73.94$\pm$1.06 & \textbf{77.86$\pm$0.33} & 72.12$\pm$1.16\\	
			\hline		
		\end{tabular}
		\label{discuss1_VOC2007}
	\end{table*}
	\begin{table*}[h]
		\caption{Complete numerical results for classification performance with different base learning algorithms on Pascal-VOC2007 dataset with class-dependent label noise. The best performances are in \textbf{bold}.}
		\centering
		\scriptsize
		\setlength\tabcolsep{5.8pt}
		\begin{tabular}{l|l|cc|cc|cc|cc}
			\hline	
			& Noise rates $(\rho_{-}, \rho_{+})$  &  (0,0.2) & (0,0.6) & (0.2,0) & (0.6,0)  & (0.1,0.1) & (0.2,0.2) &(0.017,0.2)& (0.034,0.4) \\		
			\hline	
						\multirow{6}*{\rotatebox{90}{mAP}}& Standard  & 84.25$\pm$1.07  & 77.16$\pm$0.94 & 82.70$\pm$0.54 &  68.65$\pm$1.57 &  83.07$\pm$0.45 &  78.87$\pm$0.52 & 83.92$\pm$0.59  &  80.97$\pm$0.42  \\	
			& +R-Ours  & {84.43$\pm$0.46}  & {78.72$\pm$0.41} & {84.08$\pm$0.24} &  {74.46$\pm$0.56} &  {84.03$\pm$0.29} &  {80.44$\pm$0.52} & 84.09$\pm$0.62  &  80.97$\pm$1.03  \\	\cline{2-10}
			&AGCN  & 83.24$\pm$0.67 & 75.50$\pm$0.56 & 81.09$\pm$0.51 & 66.47$\pm$1.29 & 81.09$\pm$0.48 & 73.79$\pm$0.76 & 82.21$\pm$0.42 & 76.55$\pm$1.11 \\	
			& +R-Ours& 85.07$\pm$0.56 & \textbf{80.35$\pm$0.69} & 83.41$\pm$0.68 & 66.07$\pm$4.32 & 83.79$\pm$0.41 & 78.18$\pm$2.53 & 84.38$\pm$0.30 & 80.76$\pm$0.91 \\	\cline{2-10}
			&CSRA  & {85.11$\pm$0.51} & {79.47$\pm$1.22} & 82.93$\pm$0.65 & 67.36$\pm$2.25 & 83.69$\pm$0.69 & 78.10$\pm$0.53 & {84.94$\pm$0.36} & {81.51$\pm$0.14} \\
			& +R-Ours & \textbf{85.83$\pm$0.53} & 77.64$\pm$3.21 & \textbf{84.66$\pm$0.60} & \textbf{75.41$\pm$1.77} & \textbf{84.68$\pm$0.44} & \textbf{81.01$\pm$0.57} & \textbf{85.74$\pm$0.29} & \textbf{82.64$\pm$0.35} \\	
			\hline	
			\multirow{6}*{\rotatebox{90}{OF1}}& Standard  & 75.24$\pm$1.40  & 32.02$\pm$5.49 & 78.85$\pm$0.43 &  15.08$\pm$0.25 &  79.24$\pm$0.43 &  {75.85$\pm$0.84} & 75.98$\pm$1.04  &  59.67$\pm$1.65  \\
			& +R-Ours & {78.62$\pm$0.58}  & 65.68$\pm$1.67 & {80.85$\pm$0.25} &  \textbf{67.43$\pm$4.65} &  79.64$\pm$0.29 &  {75.52$\pm$0.86} & {79.25$\pm$0.52}  &  {74.35$\pm$1.65} \\\cline{2-10}
			&AGCN  & 74.92$\pm$1.02 & 30.97$\pm$3.78 & 75.45$\pm$2.06 & 16.85$\pm$0.56 & 78.69$\pm$0.31 & 72.64$\pm$0.51 & 75.16$\pm$0.58 & 56.56$\pm$1.64  \\
			& +R-Ours & 80.28$\pm$0.41 & \textbf{71.36$\pm$2.52} & 79.18$\pm$1.38 & 46.77$\pm$3.34 & 79.06$\pm$0.92 & 73.14$\pm$2.63 & 79.49$\pm$0.76 & 75.78$\pm$2.11\\	\cline{2-10}
			&CSRA  & 76.94$\pm$1.03 & 33.65$\pm$2.73 & 77.71$\pm$1.23 & 15.94$\pm$0.32 & {80.36$\pm$0.53} & {76.92$\pm$0.34} & 77.91$\pm$0.63 & 62.19$\pm$1.97 \\
			& +R-Ours & \textbf{80.61$\pm$0.80} & 61.48$\pm$3.61 & \textbf{81.22$\pm$0.45} & 57.48$\pm$9.33 & \textbf{80.43$\pm$0.44} & \textbf{76.04$\pm$0.80} & \textbf{81.22$\pm$0.34} & \textbf{77.21$\pm$0.45}\\
			\hline	
			\multirow{6}*{\rotatebox{90}{CF1}}
			& Standard  & 72.53$\pm$1.11  & 30.64$\pm$3.90 & 76.83$\pm$0.65 &  14.97$\pm$0.24 &  75.86$\pm$1.23 &  70.68$\pm$1.76 & 73.11$\pm$0.54  &  52.07$\pm$2.34  \\
			& +R-Ours & {76.86$\pm$0.48}  & 61.29$\pm$1.94 & {77.89$\pm$0.42} &  \textbf{66.79$\pm$2.50} &  {78.04$\pm$0.40} &  {74.08$\pm$0.79} & {77.28$\pm$0.48}  & {72.18$\pm$0.74} \\	\cline{2-10}
			&AGCN  & 73.45$\pm$1.04 & 33.41$\pm$1.65 & 72.65$\pm$1.97 & 16.67$\pm$0.55 & 76.20$\pm$0.51 & 69.09$\pm$0.49 & 72.81$\pm$1.02 & 55.09$\pm$3.28 \\
& +R-Ours &\textbf{78.74$\pm$0.95} & \textbf{68.58$\pm$3.35} & 77.62$\pm$0.65 & 51.45$\pm$1.57 & 78.02$\pm$0.61 & 71.84$\pm$3.18 & 78.09$\pm$0.24 & 74.32$\pm$1.45 \\	\cline{2-10}			
			&CSRA  & 74.10$\pm$0.56 & 33.44$\pm$3.65 & 75.28$\pm$1.32 & 15.71$\pm$0.23 & 77.52$\pm$0.94 & 73.44$\pm$0.62 & 74.98$\pm$0.48 & 58.60$\pm$2.24 \\
			& +R-Ours & 78.65$\pm$0.75 & 59.28$\pm$3.71 & \textbf{78.32$\pm$0.73} & 62.59$\pm$6.25 & \textbf{78.98$\pm$0.45} & \textbf{75.28$\pm$0.37} & \textbf{79.52$\pm$0.39} & \textbf{74.38$\pm$0.98} \\	
			\hline		
		\end{tabular}
		\label{discuss2_VOC2007}
	\end{table*}
 
    \end{appendices}
\end{document}